\documentclass{article}
\usepackage{fullpage}
\usepackage{amsmath}
\usepackage{amsthm,amsfonts,amssymb}
\usepackage{xspace}
\usepackage{bbm}
\usepackage{graphicx}
\usepackage{bbm}
\usepackage{hyperref}   
\usepackage[noend]{algorithmic}
\usepackage[ruled,vlined]{algorithm2e}
\usepackage{comment}
\usepackage{natbib}
\usepackage{float}

	\hypersetup{
			colorlinks=true,
			linkcolor=[rgb]{0,0,.5},
			urlcolor=[rgb]{0,0,.5},
			citecolor=[rgb]{0,0,.5},
			pdfstartview=FitH}

\usepackage{scrextend}

\usepackage{accents}

\usepackage{graphicx} 
\usepackage{booktabs} 

\usepackage{caption}

\usepackage{xcolor}

\usepackage{pgfplots}

\pgfplotsset{compat=1.10}

\usepackage{tikz}
\usepackage{authblk}

\usepackage{enumitem}
\usepackage[caption=false]{subfig}
\usepackage{mathtools}
\usetikzlibrary{quotes,positioning}

\newtheorem{theorem}{Theorem}[section]
\newtheorem{lemma}{Lemma}[section]
\newtheorem{corollary}{Corollary}[section]
\newtheorem{definition}{Definition}[section]
\newtheorem{remark}{Remark}[section]

\newtheorem{obs}{Observation}[section]

\usepackage{xcolor}

\newcommand{\ind}{\mathbbm{1}}

\newcommand{\E}{\mathop{\mathbb{E}}}
\newcommand{\argmin}{\mathop{\mathrm{argmin}}}
\newcommand{\argmax}{\mathop{\mathrm{argmax}}}
\newcommand{\cX}{\mathcal{X}}
\newcommand{\cY}{\mathcal{Y}}

\newcommand{\cP}{\mathcal{P}}
\newcommand{\cPmean}{\cP_{\text{mean}}}

\newcommand{\cPinterval}{\cP_{\text{interval}}}
\newcommand{\cG}{\mathcal{G}}

\newcommand{\cD}{\mathcal{D}}
\newcommand{\bmu}{\overline{\mu}}
\newcommand{\mk}{m^k}
\newcommand{\bmk}{\overline{m}^k}
\newcommand{\bmkp}{\overline{m}^{k\prime}}
\newcommand{\hbmu}{h^{\bmu}}
\newcommand{\hbmk}{h^{\bmk}}

\newcommand{\tth}{^\text{th}}

\newcommand{\bell}{\overline{\ell}}

\newcommand{\bH}{\overline{H}}

\newcommand{\bu}{\overline{u}}

\newcommand{\Binv}{B^{-1}}
\newcommand{\Bm}{B_n}
\newcommand{\Bminv}{B_n^{-1}}
\newcommand{\Bmp}{B_{n'}}
\newcommand{\Bmmp}{B_{n, n'}}

\newcommand{\centerbucket}{\hat\mu_i}
\newcommand{\centerbucketmu}{\hat\mu_{\bmu}}
\newcommand{\centerbucketmut}{\hat\mu_{\bmu_t}}

\newcommand{\tL}{\tilde{L}}
\newcommand{\cdotk}{ }
\newcommand{\tV}{\tilde{V}}
\newcommand{\tpi}{\tilde{\pi}}
\newcommand{\tZ}{\tilde{Z}}
\newcommand{\ty}{\tilde{y}}

\newcommand{\cPmm}{\mathcal{P}_{\text{(mean,moment)}}}

\usepackage{thmtools} 
\usepackage{thm-restate}

\usepackage{amsmath}
\usepackage{pgfplots}

\title{Online Multivalid Learning:\\ Means, Moments, and  Prediction Intervals}

\author[1]{Varun Gupta}
\author[1]{Christopher Jung}
\author[1]{Georgy Noarov}
\author[2]{Mallesh M. Pai}
\author[1]{Aaron Roth}
\affil[1]{University of Pennsylvania Department of Computer and Information Science}
\affil[2]{Rice University Department of Economics}

\begin{document}
\maketitle

\begin{abstract}
We present a general, efficient technique for providing contextual predictions that are ``multivalid'' in various senses, against an online sequence of adversarially chosen examples $(x,y)$. This means that the resulting estimates correctly predict various statistics of the labels $y$ not just \emph{marginally} --- as averaged over the sequence of examples --- but also conditionally on $x \in G$ for any $G$ belonging to an arbitrary intersecting collection of groups $\cG$. 

We provide three instantiations of this framework. The first is mean prediction, which corresponds to an online algorithm satisfying the notion of multicalibration from \cite{multicalibration}. The second is variance and higher moment prediction, which corresponds to an online algorithm satisfying the notion of mean-conditioned moment multicalibration from \cite{momentmulti}. Finally, we define a new notion of prediction interval multivalidity, and give an algorithm for finding prediction intervals which satisfy it. Because our algorithms handle adversarially chosen examples, they can equally well be used to predict statistics of the residuals of arbitrary point prediction methods, giving rise to very general techniques for quantifying the uncertainty of predictions of black box algorithms, even in an online adversarial setting. When instantiated for prediction intervals, this solves a similar problem as conformal prediction, but in an adversarial environment and with multivalidity guarantees stronger than simple marginal coverage guarantees. 
\end{abstract}

\thispagestyle{empty} \setcounter{page}{0}
\clearpage
\tableofcontents
\thispagestyle{empty} \setcounter{page}{0}
\clearpage

\section{Introduction}
Consider the problem of making predictions about the prognoses of patients with an infectious disease at the early stages of a pandemic. To be able to guide the allocation of medical interventions, we may want to predict, from each patient's observable features $x$, things such as the expected severity of the disease $y$ in two days' time. And since we will be using these predictions to allocate scarce resources, we will want to be able to quantify the uncertainty of our predictions: perhaps by providing estimates of the variance of outcomes, or perhaps by providing prediction intervals at a desired level of confidence.  

This is an \emph{online} problem because we must start making predictions before we have much data, and the predictions are needed immediately upon the arrival of a patient. It is also a problem in which the environment is rapidly changing: the distribution of patients changes as the disease spreads through different populations, and the conditional distribution on outcomes given features changes as we learn how to better treat the disease. 

How can we approach this problem? The  \emph{conformal prediction} literature \citep{conformal} aims to equip arbitrary regression and classification procedures for making point predictions with prediction intervals that contain the true label with (say) 95\% probability. But for the application in our example, conformal prediction has two well-known shortcomings:
\paragraph{Marginal Guarantees:} 
Conformal prediction only gives \emph{marginal} prediction intervals: in other words, it provides guarantees that (e.g.) 95\% of the prediction intervals produced over a sequence of predictions cover their labels. But these guarantees are averages over what are typically large, heterogeneous populations, and therefore provide little guidance for making decisions about individuals. For example, it would be entirely consistent with the guarantee of a 95\% marginal prediction interval $[\ell_t,u_t]$ if for individuals from some demographic group $G$ making up less than $5\%$ of the population, their labels $y_t$ fall outside of $[\ell_t,u_t]$ 100\% of the time.%
\footnote{Even more insidious reversals, albeit not in the context of conformal prediction, have been observed on real world data---see the Wikipedia entry for Simpson's paradox (\url{https://en.wikipedia.org/wiki/Simpson\%27s_paradox}) for several examples.}
One could run many parallel algorithms for different demographic groups $G_i$, but then there would be no clear way to interpret the many different predictions one would receive for an individual belonging to several demographic groups at once ($x \in G_i$ for multiple groups $G_i$); for example, prediction intervals corresponding to different demographic groups could be disjoint. To see that marginal guarantees on their own are extremely weak, consider a batch (distributional) setting in which labelled points are drawn from a fixed distribution $\cD$: $(x,y) \sim \cD$. Then we could provide valid 95\% marginal prediction intervals by entirely ignoring the features and giving a fixed prediction interval of $[\ell,u]$ for every point, where $[\ell,u]$ is such that $\Pr_{(x,y) \sim \cD}[y \not\in [\ell,u]] = 0.05$. 

\paragraph{Distributional Assumptions:} 
The conformal prediction literature almost exclusively assumes that the data is drawn from an \emph{exchangeable} distribution (for example, i.i.d. data satisfies this property), and does not offer any guarantees when the data can quickly change in unanticipated or adversarial ways.

\paragraph{}
In this paper we give techniques for dealing with both of these problems (and similar issues that arise for the problem of predicting label means and higher moments) by drawing on ideas from the literature on \emph{calibration} \cite{dawid1982well,FV98}. Calibration is similar to conformal prediction in that it aims to give point estimates in nonparametric settings that satisfy marginal rather than conditional guarantees (i.e. that agree with the true distribution as averaged over the data rather than conditioned on the features of a particular data point). But calibration is concerned with predicting label expectations, rather than giving prediction intervals. Informally speaking, calibrated predictions satisfy that when averaging over all rounds over which the prediction was (approximately) $p$, the realized labels average to (approximately) $p$, for all $p$. Note that in a distributional setting, if a learner truly was predicting the conditional label expectations conditional on features $p_x = \E_{(x,y) \sim \cD}[y|x]$, then the forecasts would be calibrated --- but just as with marginal prediction intervals, calibration on its own is a very weak condition in a distributional setting. For example, a learner could achieve calibration simply by making a single, constant prediction of $p = \E_{(x,y)\sim \cD}[y]$ for every point, and so calibrated predictions need not convey much information.  Thus, just like the conformal prediction literature, the calibration literature is primarily focused on the online prediction setting. But from early on, the calibration literature has focused on the \emph{adversarial} setting in which no distributional assumptions need to be made at all \cite{FV98,FL99,sandroni2003calibration}. 

Calibration also suffers from the weaknesses that come with marginal guarantees: namely that calibrated predictions may have little to do with the conditional label expectations for members of structured sub-populations. \cite{multicalibration} proposed an elegant solution to this problem in the batch setting, when predicting expectations, which they termed ``multicalibration''. Informally speaking, a guarantee of multicalibration is parameterized by a large collection of potentially intersecting subsets of the feature space $\cG$ (corresponding e.g. to demographic groups or other categories relevant for the prediction task at hand). Multicalibration asks for predictions that are not just calibrated over the full distribution $\cP$, but are also simultaneously calibrated over all of the induced distributions that are obtained by conditioning on membership in a set $G \in \cG$. Moreover, \citet{multicalibration} showed how to obtain multicalibrated estimators in the batch, distributional setting with sample complexity that depends only logarithmically on $|\cG|$. \citet{momentmulti} showed how to extend the notion of (multi)calibration from expectations to  variances and other higher moments --- and derived algorithms for obtaining such estimates in the batch setting. 

\subsection{Our Results and Techniques}
In this paper, we give a general method for obtaining different kinds of ``multivalid'' predictions in an online, adversarial setting. This includes mean estimates that satisfy the notion of mean multicalibration from \cite{multicalibration}, moment estimates that satisfy the notion of mean-conditioned moment multicalibration from \citet{momentmulti}, and prediction intervals which satisfy a new notion of multivalidity, defined in this paper. The latter asks for tight marginal prediction intervals, which are simultaneously valid over each demographic group $G \in \cG$. We give a formal definition in Section~\ref{sec:prelims} (and review the definitions of mean and moment multicalibration), but informally, multivalidity for prediction intervals asks, given a target coverage probability $1-\delta$, that for each group $G \in \cG$ there be roughly a $1-\delta$-fraction of points $(x_t,y_t)$  with $x_t \in G$ whose label is contained within the predicted interval ($y_t \in [\bell_t,\bu_t)$). In fact, we ask for the stronger calibration-like guarantee, that these marginal coverage guarantees hold even conditional on the prediction interval, which (among other things) rules out the trivial solution to marginal coverage that predicts the full interval with probability $1-\delta$ and an empty interval with probability $\delta$. Because our algorithms handle adversarially selected examples, they can equally well be used to augment arbitrary point prediction procedures which give predictions $f_t(x_t) = \hat y_t$, independently of how they are trained: We can simply feed our algorithms for multivalid predictions with the \emph{residuals} $\hat y_t - y_t$. For example, we can get variance estimates or prediction intervals for the residuals to endow the \emph{predictions} of $f_t$ with uncertainty estimates. Endowing point predictors with prediction intervals in this way provides an alternative to conformal prediction that gives stronger-than-marginal (multivalid) guarantees, under much weaker assumptions (adversarially chosen examples). In general, for each of our techniques, if we instantiate them with the trivial group structure (i.e. one group, containing all points), then we recover standard (or slightly stronger) marginal guarantees: i.e. simple calibrated predictions and simple marginal prediction intervals.%
\footnote{In fact, even with the trivial group structure, our guarantees (with appropriately set parameters) remain stronger than marginal coverage. This is because our prediction intervals remain valid even conditioning on the prediction that we made. For example, a prediction interval $[\ell,u)$ is valid not just as averaged over all rounds $t$, but also as averaged over all rounds $t$ for which we made that specific prediction: $t : [\bell_t,\bu_t) = [\ell,u)$.} But as we  enrich our collection of sets $\cG$, our guarantees become correspondingly stronger.

\paragraph{The General Strategy} 
We derive our online algorithms using a general strategy that dates back to \citet{FL99}, who used it to give online algorithms for the problem of simple calibration in a setting without features (see also the argument by Sergiu Hart, communicated in \citet{FV98} and more recently elaborated on in \cite{Hart20}). In our context, the general strategy proceeds as follows:
\begin{enumerate}
    \item Define a surrogate loss function, such that if the surrogate loss is small at the end of $T$ rounds, then the learner's predictions satisfy our chosen notion of multivalidity over the  empirical distribution of the history of the interaction.
    \item Argue that if at each round $t$, the adversary's chosen distribution over labelled examples were known to the learner, then there would be some prediction that the learner could make that would guarantee that the expected increase in the surrogate loss function at that round would be small. This step is often straightforward, because once we fix a known data distribution $\cD$, ``true distributional quantities'' like conditional label expectations, conditional label variances, conditional label quantiles, etc, generally satisfy our corresponding multivalidity desideratum by design. \label{step2} 
    \item Appeal to the minimax theorem to conclude that there must therefore exist a randomized prediction strategy for the learner that guarantees that the expected increase in the surrogate loss function is small for \emph{any} choice of the adversary. \label{step3}
\end{enumerate}

On its own, carrying out this strategy for a particular notion of multivalidity proves the \emph{existence} of an algorithm that can obtain the appropriate notion of multivalidity against an adversary; but turning it into an actual (and efficient) algorithm requires the ability to \emph{compute} at each round the equilibrium strategy whose existence is shown in Step~\ref{step3} above. 

We instantiate this general strategy in Section~\ref{sec:onlinemeanmulti} for the case of mean multicalibration, which also serves as a template for our derivation and analysis of algorithms for moment multicalibration in Section~\ref{sec:onlinemoment} and prediction interval multivalidity in Section~\ref{sec:onlinemultivalid}. The framework of our analysis is the same in each case, but the details differ: to carry out Step~\ref{step2}, we must bound the value of a different game, and to carry out Step~\ref{step3}, we must solve for the equilibrium of a different game. In each case, we obtain efficient online algorithms for obtaining  high probability $\alpha$-approximate multivalidity bounds (of different flavors), with $\alpha$ scaling roughly as $\alpha \approx \sqrt{\log |\cG|/T}$, over interactions of length $T$ --- but see Sections \ref{sec:meanexistential}, \ref{sec:momentexistential}, and \ref{sec:intervalexistential} for exact theorem statements. In all cases, our algorithms have per-round runtime that is linear in $|\cG|$, and polynomial in the other parameters of the problem. In fact, both our run-time and our convergence bounds can be improved if each individual appears in only a bounded number of groups. Our algorithms can at each step $t$ be implemented in time linear in the number of groups $G \in \cG$  that \emph{contain the current example $x_t$}. This is linear in $|\cG|$ in the worst case, but can be substantially smaller. Similarly, we show in Appendix \ref{sec:moregroups} that if each individual appears in at most $d$ groups, then the $\log|\cG|$ term in our convergence bounds can be replaced with $\log (d)$, which gives informative bounds even if $\cG$ is infinitely large.  Without assumptions of this sort, running time that is polynomial in $|\cG|$ (rather than logarithmic in $|\cG|$, as our convergence bounds are) is necessary in the worst case, even for mean multicalibration in the offline setting, as shown by \citet{multicalibration}. 

Adapting the original approach of \cite{FL99} runs into several obstacles, stemming from the fact that the \emph{action space} of both the learner and the adversary and the \emph{number of constraints} defining our calibration desideratum are both much larger in our setting. Consider the case of mean prediction --- in which the goal is to obtain calibrated predictions.  In the featureless setting studied by \cite{FL99}, the action space for the learner corresponds to a discretization of the real unit interval $[0,1]$, and the action space of the adversary is binary. In our setting, in which data points are endowed with features from a large feature space $\cX$, the learner's action space corresponds to the set of all \emph{functions} mapping $\cX$ to $[0,1]$, and the adversary's action space corresponds to the set of all labelled examples $\cX\times [0,1]$. Similarly, for simple calibration, the number of constraints is equal to the chosen discretization granularity of the unit interval $[0,1]$, whereas in our case, the number of constraints also grows linearly with $|\cG|$, the number of groups over which we want to be able to promise guarantees.

\paragraph{Convergence Rates and Sample Complexity}
The surrogate loss function used by \cite{FL99} bounds the $\ell_2$ calibration error --- i.e. the average squared violation of all of the constraints used to define calibration. Because all of the notions of multivalidity that we consider consist of a set of constraints of size scaling linearly with $|\cG|$, if we were to attempt to bound the $\ell_2$ violation of our multivalidity constraints, we would necessarily obtain convergence bounds that scale polynomially with $|\cG|$. Instead we use a different surrogate loss function  --- a sign-symmetrized version of an exponential soft-max  --- that can be used to bound the $\ell_\infty$ violation of our multivalidity constraints, and allows us to obtain bounds that scale only logarithmically with $|\cG|$. For moment multicalibration, we face the further complication of needing to define a potential function bounding a linear surrogate for what is ultimately a nonlinear measure of distributional fidelity. An outline of the specific new ideas needed here can be found in Section~\ref{sec:moment-outline}. For interval multivalidity, we face the further complication that tight prediction intervals need not exist even in the distributional setting, for worst-case distributions. An outline of the new ideas we need to overcome this can be found in Section~\ref{sec:interval-outline}. Finally, we note that $\ell_\infty$ violation is consistent  with how the existing literature on batch multicalibration \citep{multicalibration} has quantified approximation guarantees.  In fact, by using standard online-to-offline reductions, we are able to derive new, optimal sample complexity bounds for mean and moment multicalibration for the \emph{batch distributional} setting in Appendix~\ref{sec:batch} that improve on the sample complexity bounds given in \citet{multicalibration,momentmulti}. This is because when applied to the batch setting, our online algorithms take only a single pass through the data, and avoid issues related to adaptive data re-use that complicated previous algorithms in the batch setting. 

\paragraph{Computation of Equilibrium Strategies}
To compute equilibria of the large action space games we define,  we do not attempt to directly compute or represent the function that we use at each round $t$ to map features to labels. Instead, we represent this function implicitly by ``lazily'' solving a smaller equilibrium computation problem only after we have observed the adversary's choice of feature vector $x$ (but before we have observed the label $y$) to compute a distribution over predictions. We show in each of our three settings that this computation is tractable. In the case of mean multicalibration, we are able to analytically derive a simple algorithm for sampling from this equilibrium strategy, presented in Section~\ref{sec:meanalg}. For mean-conditioned $k\tth$ moment multicalibration we show that the equilibrium can be found using a linear program with polynomially many variables and $2^k+1$ constraints. For the most interesting cases, $k$ is a small constant (e.g. for variance, $k = 2$, and so the linear program has only 5 constraints). Even when $k$ is large, we show that this linear program has a separation oracle that runs in time $O(k)$, and so it can be solved efficiently via the Ellipsoid algorithm. We show in Appendix~\ref{sec:momentellipsoid} that there always exists an equilibrium for the learner with support over at most $k+1$ many predictions, limiting the extent to which it needs to deploy randomization.  Finally, for prediction interval multivalidity, we show in Section~\ref{sec:intervalalg} that we can express the equilibrium computation problem as a linear program. Although the linear program is naively defined by infinitely many constraints, we show that it can ultimately be represented with only finitely many constraints, and that it has an efficient separation oracle, so can be solved in polynomial time using the Ellipsoid algorithm. 

\paragraph{Advantages of Conformal Prediction}
We have thus far emphasized the advantages that our techniques have over conformal prediction ---  but we also want to highlight the strengths of conformal prediction relative to our work, and directions for future improvement. 
Conformal prediction aims to obtain  marginal coverage with respect to some (unknown) underlying distribution. As a result of the distributional assumption, it is able to obtain coverage (over the randomness of the distribution) at a rate of coverage $1 - \delta + O(1/T)$ \citep{lei2018distribution}. In contrast, in our setting, there is no underlying distribution. We therefore give guarantees on  \emph{empirical coverage} --- i.e the fraction of labels that our predicted intervals have covered in the realized sequence of examples. As a result, our coverage bounds necessarily have error terms that tend to $0$ at a rate of $O(1/\sqrt{T})$, over sequences of length $T$. We note that conformal prediction methods also obtain \emph{empirical} coverage on the order of $1-\delta \pm O(1/\sqrt{T})$, as our methods do \citep{lei2018distribution}. Conformal prediction methods  naturally give one sided coverage error on the distribution (i.e. the coverage probability is always $\geq 1-\delta$), whereas as we present our bounds, our empirical coverage has two sided error. We note that there is a simple but inelegant way to use our techniques to obtain one sided coverage: run our algorithms with coverage parameter $1-\delta' = 1-\delta/2$, and predict trivial coverage intervals until our error bounds are $\leq \delta/2$\footnote{Restarting periodically with $\delta'$ closer to $\delta$ if we want to asymptotically converge to  exact coverage}.  Techniques from the conformal prediction literature also can be applied to very general label domains $\cY$, and can be used to produce very general kinds of prediction \emph{sets}. In our paper, we restrict attention to real-valued labels $\cY = [0,1]$ and prediction \emph{intervals}. We do not believe that there are any fundamental obstacles to generalizing our techniques to other label domains and prediction sets, and this is an interesting direction for future work. Finally, the conformal prediction literature has developed a number of very simple, practical techniques. In this paper, we give polynomial time algorithms, of varying complexity. Our algorithm for mean multicalibration in Section~\ref{sec:onlinemeanmulti} is very simple to implement, but our algorithm for multivalid interval prediction in Section~\ref{sec:onlinemultivalid} requires solving a linear program with a separation oracle. Another important direction for future work is reducing the complexity of our techniques, and doing empirical evaluations.

\subsection{Additional Related Work}
Work on calibrated mean prediction dates back to \cite{dawid1982well}. \cite{FV98} were the first to show that in the online setting without features, it is possible to obtain asymptotic calibration even against an adversary. Once this initial result was proven, a number of proofs of it were given using different techniques, including Blackwell's approachability theorem \citep{Fos99} and a non-constructive minimax argument (originally communicated verbally by Sergiu Hart, appearing first in \cite{FV98}, and more recently formalized in \cite{Hart20}). This argument was ``non-constructive'' because it was a minimax argument over the entire algorithm design space. \cite{FL99} gave a more tractable per-round minimax argument, which we adapt to our work --- although they were satisfied with an existential argument, and do not derive a concrete algorithm. The algorithm we give for online multicalibration is similar to the algorithm given by \cite{foster2019forecast} for the simple calibration problem in the special case of a featureless setting and the trivial group structure.   \cite{lehrer2001any,sandroni2003calibration} (and in a slightly different context, \cite{fudenberg1999conditional}) generalized this literature and showed that it was possible to extend these ideas in order to satisfy dramatically more demanding notions of calibration (e.g. calibration on all computable subsequences of rounds).
This line of work primarily gives limit results via non-constructive arguments without establishing rates. There are two notable exceptions. \cite{foster2011complexity} give a non-constructive argument establishing that it is possible to obtain mean calibration loss $\tilde O(\sqrt{\frac{\log K}{T}})$ with respect to a set of $K$ ``checking rules'' which define subsequences over which the algorithm must be calibrated. These results are derived in a setting without features $x$, but we believe their techniques could be used to establish the same convergence bounds that we do, for mean multicalibration: $ \alpha = \tilde O(\sqrt{\frac{\log |\cG|}{T}})$.  \cite{foster2006calibration}  give an efficient algorithm based on ridge-regression which can be used to achieve what we call \emph{mean consistency}\footnote{This is also what is known as \emph{multi-accuracy} in \citep{multicalibration,multiaccuracy}.} on a collection of sets $\cG$ with error rates converging as  $\alpha = \tilde O(\sqrt{\frac{|\cG|}{T}})$. Their algorithm is deterministic, which in particular means it cannot be used to achieve the standard notion of calibration, which can only be achieved by randomized algorithms in adversarial environments \citep{oakes1985self}. It can be used to achieve what is called ``weak calibration'' by \cite{kakade2004deterministic} and ``smooth calibration'' by \cite{foster2018smooth}  --- a relaxation that can be obtained by deterministic algorithms. In comparison, our algorithm for mean multicalibration achieves the standard notion of calibration with the optimal sample complexity dependence on $\log |\cG|$, while simultaneously being explicitly defined and computationally efficient.

There has also been a recent resurgence of interest in calibration in the computer science community, in part motivated by fairness concerns \citep{KMR16,Cho17,pleiss2017fairness}. It is from this literature that the original proposal for multicalibration arose \citep{multicalibration}, as well as the related notion of multiaccuracy \citep{multicalibration,multiaccuracy}.   \cite{multiUC} prove uniform convergence bounds for multicalibrated predictors, \cite{dwork2019learning} draw connections between multicalibrated predictors and notions of fair rankings, and \cite{dwork2020outcome} define a notion of outcome indistinguishability related to distribution testing, and show close connections to multicalibration. \cite{momentmulti} extend the notion of mean calibration to variances and higher moments, and give efficient algorithms for learning moment multicalibrated predictors. \cite{momentmulti} also show that their moment predictors can be used to derive \emph{conservative} multivalid prediction intervals, using Chebyshev's inequality and generalizations to higher moments. In general, however, these moment-based inequalities give intervals that may cover their label much more frequently than the target $1-\delta$ coverage probability, and cannot achieve the kinds of tight multicoverage guarantees that we obtain in this work. All of this work operates in the batch, distributional setting. Recently, \citet{qiao2020stronger} proved lower bounds for simple mean calibration in the online setting, showing that no algorithm can obtain rates better than $O(T^{-0.472})$ against an adversary. At first blush, our upper bounds appear to contradict these lower bounds --- but they do not, because we study convergence in the $\ell_{\infty}$ sense, whereas they study it in the $\ell_1$ sense. 

Conformal prediction is motivated similarly to calibration, but aims to produce marginal prediction intervals rather than mean estimates --- see \cite{conformal} for an overview. The problems that we highlight --- namely, that marginal guarantees are weak, and that this literature relies on strong distributional assumptions --- have been noted before. For example, \cite{barber2019limits} prove that even in the distributional setting, \emph{conditional} prediction intervals are impossible to provide, and aim instead for a goal that is similar to ours: providing marginal prediction intervals that are valid as averaged over a large number of subgroups $\cG$. They take a  conservative approach, by using a holdout set to estimate empirical prediction intervals separately for each group, and then taking the union of all of these prediction intervals over the demographic groups of a new individual. The result is that their prediction intervals --- unlike ours --- do not become tight, even in the limit.  \cite{chernozhukov2018exact} consider the problem of conformal prediction for time series data, for which the exchangeability assumption may not hold. They show that if the data comes from a rapidly mixing process (so that, in particular, points that are well separated in the sequence are approximately independent) then it is still possible to obtain approximate marginal coverage guarantees.  \cite{tibshirani2019conformal} consider the problem of conformal prediction under \emph{covariate shift}, in which the marginal distribution on features $\cX$ differs between the training and test distributions, but the conditional distribution on labels $\cY | \cX$ remains the same. They show how to adapt techniques from conformal prediction when the likelihood ratio between the training and test distribution is known.  In the distributional setting, \cite{gupta2020distribution} have proven close relationships between calibration, confidence intervals, and prediction intervals.

Finally, the notion of multicalibration is related to subgroup fairness notions \citep{gerrymandering,subgroup,cbawareness} that ask for statistical ``fairness'' constraints of various sorts (beyond calibration) to hold across all subgroups defined by some rich class $\cG$. See \cite{fairsurvey} for a survey.

\section{Preliminaries}
\label{sec:prelims}
\subsection{Notation}
We write $\cX$ to denote a feature domain and $\cY = [0,1]$ to denote a label domain. We write $\cG \subseteq 2^\cX$ to denote a collection of subsets of $\cX$. Given any $x \in \cX$, we write $\cG(x)$ for the set of groups that contain $x$, i.e.  $\cG(x) = \{G \in \cG: x \in G\}$. Given an integer $T$ we write $[T]$ to denote the set of integers $[T] = \{1,\ldots,T\}$.  In general, we denote our random variables with tildes (e.g. $\tilde X$, $\tilde Y$) to distinguish them from their realizations (denoted e.g. $X$, $Y$). Given a finite set $A$, we write $\Delta A$ for the probability simplex over the elements in $A$. 

\subsection{Online Prediction}
Online (contextual) prediction proceeds in rounds that we index by $t \in [T]$, for a given finite horizon $T$. In each round, an interaction between a \emph{learner} and an \emph{adversary} proceeds as follows. In each round $t$:
\begin{enumerate}
    \item The \emph{adversary} chooses a joint distribution over feature vectors $x_t \in \cX$ and labels $y_t \in \cY$. The learner receives $x_t$ (a realized feature vector), but no information about $y_t$ is revealed. 
    \item The \emph{learner} chooses a distribution over predictions $p_t \in \cP$. (We will consider several different kinds of predictions in this paper, and so are agnostic to the domain of the prediction for now --- we use $\cP$ as a generic domain).
    \item The learner observes $y_t$ (a realized label). 
\end{enumerate}

For an index $s \in [T]$, we denote by $\pi_s$ the \emph{transcript} of the interaction in rounds $t =1$ through $s$: $\pi_s = ((x_t,p_t,y_t))_{t=1}^{s}$.  We write $\Pi^*$ as the domain of all transcripts.

Formally, the adversary is modelled as a probabilistic mapping $\mathrm{Adv}:\Pi^*\rightarrow \Delta(\cX \times \cY)$ from transcripts to distributions over labelled data points, and the learner is modeled as a mapping $\mathrm{Learn}:\Pi^*\rightarrow (\cX \rightarrow \Delta \cP)$ from transcripts to a probabilistic mapping from feature vectors to distributions over predictions. An adversary may be either unconstrained (free to play any point in $\Delta(\cX\times \cY)$) or constrained to choose from some specified subset of $\Delta(\cX\times \cY)$. Fixing both a learner and an adversary induces a probability distribution over transcripts. Our goal is to derive particular learning algorithms, and to prove that various kinds of bounds hold either in expectation, or with high probability over the randomness of the transcript, in the worst case over transcript distributions, where we quantify over all possible adversaries.

Given a transcript $\pi_T$, a group $G \in \cG$ and a set of rounds $S \subseteq [T]$, we write 
\[
    G_S = \{ t \in S: x_t \in G\}.
\]
In words, this is the set of rounds in $S$ in which the realized feature vectors in the transcript belonged to $G$.  When it is clear from context, we sometimes overload notation, and for a group $G \in \cG$, and a period $s \leq T$, write $G_s$ to denote the set of data points (indexed by their rounds) in a transcript $\pi_s$ that are members of the group $G$:
\[
    G_s = \{t \in [s]: x_t \in G\}.
\]

\subsubsection{Types of Predictions, and Notions of Validity}
We consider three types of predictions in this paper: Mean predictions, pairs of mean and higher moment predictions (e.g. variance), and prediction intervals. 

\paragraph{Mean Predictions} 
For mean predictions, the prediction domain will be the unit interval: $\cPmean = [0,1]$. 
The learner will select $p_t \equiv \bmu_t \in \cPmean$ in each round $t$, with the goal of predicting the conditional label expectation $\E[y_t | x_t]$.  For any subset of days $S \subseteq [T]$, we write 
\[
    \mu(S) = \frac{1}{|S|}\sum_{t \in S} y_t, \quad \bmu(S) = \frac{1}{|S|}\sum_{t \in S} \bmu_t
\]
to denote the true label population mean conditional on $t \in S$ and the average of our mean estimates over days $t \in S$, respectively. We will ask for our predictions to satisfy large numbers of \emph{mean consistency} constraints: that the conditional label averages be (approximately) equal to conditional prediction averages over different sets $S$. 

\begin{definition}[Mean Consistency]
Given a transcript $\pi_T$, we say that the mean predictions $\{\bmu_t\}_{t=1}^T$ are $\alpha$-mean consistent on $S \subseteq [T]$ , if 
\[
    |\mu(S) - \bmu(S)| \le \alpha \frac{T}{|S|}.
\]
\end{definition}

\begin{remark}
Note the scaling with both $T$ and $|S|$. If $S = [T]$, then this condition simply asks for the true label mean and the average prediction to be within $\alpha$ of one another, as averaged over the entire transcript. For smaller sets, the allowable error grows with the inverse of $\frac{|S|}{T}$ --- i.e. the measure of $S$ within the uniform distribution over the transcript. Even in a distributional setting, estimates inevitably degrade with the size of the set we are conditioning on, and our formulation corresponds exactly to how mean consistency is defined in \cite{momentmulti}. Our definitions are also consistent with how the literature on online calibration quantifies calibration error with respect to subsequences. \citet{multicalibration} handle this issue slightly differently, by asking for uniform bounds, but in the end proving bounds only for sets $S$ that have sufficient mass $\gamma$ in the underlying probability distribution. In the batch setting, our formulation can recover bounds that are strictly stronger than those of \citet{multicalibration} after a reparametrization $\alpha \leftarrow \gamma \alpha$.
\end{remark}

Next, we define multicalibration in our setting. Informally, a sequence of mean predictions is \emph{calibrated} if the average realized label $y_t$ on all days for which $\bmu_t$ is (roughly) $p$ is (roughly) $p$. The need to consider days in which the prediction was \emph{roughly} $p$ arises from the fact that a learning algorithm will not necessarily ever make the same prediction twice. More generally, by bucketing predictions at a fixed granularity, we can guarantee that the average number of predictions within each bucket grows linearly with $T$.  

To collect mean predictions $\bmu_t$ that are approximately equal to $p$ for each $p$, we group real-valued predictions into $n$ buckets of width $\tfrac{1}{n}$. Here $n$ is a parameter controlling the coarseness of our calibration guarantee. For any coarseness parameter $n$ and bucket index $i \in [n-1]$, we write $\Bm(i) = \left[\frac{i-1}{n}, \frac{i}{n}\right)$ and $\Bm(n ) = \left[\frac{n-1}{n}, 1 \right]$ so that these buckets partition the unit interval. Conversely, given a $\bmu \in [0,1]$, define $\Bminv (\bmu) \in [n]$ in the obvious way i.e. $\Bminv (\bmu) = i$ where $i$ is such that $\bmu \in \Bm(i)$.  When clear from the context, we elide the subscript $n$ and write $B(i)$ and $\Binv(\bmu)$.

For any $S \subseteq [T]$ and $i \in [n]$, we write
\[
    S(i) = \left\{t \in S: \bmu_t \in \Bm(i) \right\}.
\]
In words, $S(i)$ corresponds to the subset of rounds in $S$ where the mean prediction falls in the $i\tth$ bucket. 

(Simple) calibration asks for the sequence of predictions to be $\alpha$-mean-consistent on all sets $[T](i)$ for $i \in [n]$ --- i.e. for the subset of rounds in which the prediction fell into the $i\tth$ bucket, for all $i$. Multicalibration asks for the predictions to be  calibrated not just on the overall sequence, but also simultaneously on all the subsequences corresponding to each group $G \in \cG$. In  our notation, it asks for mean consistency on each set $G(i)$, for every group $G \in \cG$ and $i \in [n]$. 
\begin{definition}[Mean-Multicalibration]
\label{def:meanmulti}
Given a transcript $\pi_T$, we say that the mean predictions $\{\bmu_{t}\}_{t=1}^T$ are \emph{($\alpha,n$)-mean multicalibrated} with respect to $\cG$ if we have that for every $G \in \cG$ and $i \in [n]$, the mean-predictions are $\alpha$-mean consistent on $G_T(i)$: 
\[
    |\mu(G_T(i)) - \bmu(G_T(i))| \leq \alpha \frac{T}{|G_T(i)|}.
\]
\end{definition}

\begin{remark}
Note that we define mean multicalibration (and our other notions of multivalidity, shortly) to have two parameters: $n$, which controls the coarseness of the guarantee, and $\alpha$, which controls the error of the guarantee. These parameters can be set independently --- in the sense that we will be able to achieve $(\alpha,n)$ mean multicalibration for any pair $(\alpha,n)$ --- but they should be interpreted together. For example, to avoid the trivial solution in which the learner simply selects uniformly at random at each iteration (thereby guaranteeing that $|G_T(i)| \leq T/n$ for all $G, i$), we should set $\alpha \ll \tfrac1n$.
\end{remark}

\paragraph{(Mean, Moment) Predictions}
In this case, the prediction domain is the product of the unit interval with itself: $\cPmm = [0,1] \times [0,1]$. In each round $t$, the learner selects $p_t = (\bmu_t, \bmk_t)$ with the goal of matching $\E[y_t | x_t]$ and $\E[(y_t - \E[y_t|x_t])^k | x_t]$ respectively --- the conditional label expectation, and its conditional $k\tth$ central moment. For simplicity, we assume throughout that $k$ is even, so the $k\tth$ moment has nonnegative range, but there is no obstacle other than notation to handling odd moments as well. 

We group continuous predictions $(\bmu,\bmk)$ into a finite set of discrete buckets---again, defined with respect to a pair of discretization parameters $n$ and $n'$.  Recall our bucketing notation for mean prediction:
for any $i \in [n-1]$, we wrote $\Bm(i) = \left[\frac{i-1}{n}, \frac{i}{n}\right)$ and $\Bm(n) = \left[\frac{n-1}{n}, 1 \right]$. Here we generalize this notation to pairs, and write for any $i\in [n]$ and $j \in [n']$:
\[
    \Bmmp(i,j) = \left\{(a,b) \in [0,1]\times[0,1] : a \in  \Bm(i), b \in  \Bmp(j) \right\}.
\] 
If $n = n'$, we will write $\Bm(i,j)$. Once again, when $n$ and $n'$ are clear from the context, we may elide the subscript $(n,n')$ entirely. 

Analogously to our notation for mean prediction, for any $S \subseteq [T]$ we write
\[
\mk(S) = \frac{1}{|S|} \sum_{t \in S} (y_t- \mu(S))^k, \quad \bmk(S) = \frac{1}{|S|} \sum_{t \in S} \bmk_t
\]
for the empirical $k^{\text{th}}$ central moment of the label distribution on the subsequence $S$, and for the average of the moment prediction on $S$, respectively. Just as with mean consistency, moment consistency asks that these two quantities be approximately equal on a set $S$. 

\begin{definition}[Moment Consistency]
Given a transcript $\pi_T$, we say that moment predictions $\{\bmk_t\}_{t=1}^T$ are $\alpha$-moment consistent on set $S \subseteq [T]$ if 
\[ 
|\mk(S) - \bmk(S)| \le \alpha \frac{T}{|S|}.
\]
\end{definition}

It is not sensible to ask for moment consistency on arbitrary sets $S$, because higher central moments are not linear, and so even true conditional label moments would not satisfy moment consistency conditions on arbitrary sets $S$. True conditional label moments \emph{do} satisfy moment consistency on sets of points $x$ that share the same label mean, however, and so this is what we will ask of our predictions as well (See \cite{momentmulti} for an extensive discussion of this condition and its applications). To that end, for any $S \subseteq [T]$ and $i \in [n], j \in [n']$, we write
\[
    S(i,j) = \left\{t \in S: (\bmu_t,\bmk_t) \in \Bmmp (i,j) \right\}.
\]
In words, $S(i,j)$ corresponds to the subset of rounds in $S$ in which our predicted mean and moment fall into the bucket $\Bmmp(i,j)$.  

\begin{definition}[Mean-Conditioned Moment Multicalibration] \label{def:mean-conditioned-moment-multi}
Given a transcript $\pi_T$, we say that the (mean, moment) predictions $\{(\bmu_t, \bmk_t)\}_{t=1}^T$ are \emph{$(\alpha,\beta, n, n')$-mean-conditioned moment multicalibrated with respect to $\cG$}, if for every $i \in [n], j \in [n']$ and $G \in \cG$, we have that the mean predictions are $\alpha$-mean consistent on $G_T(i,j)$ and the moment predictions are $\beta$-moment consistent on $G_T(i,j)$:
\begin{align*}
&    |\mu(G_T(i,j)) - \bmu(G_T(i,j))| \le \alpha \frac{T}{|G_T(i,j)|},\\
&    |\mk(G_T(i,j)) - \bmk(G_T(i,j))| \le \beta \frac{T}{|G_T(i,j)|}.
\end{align*}
\end{definition}

\paragraph{Interval Predictions}
 In this case, the prediction domain is the set of ordered pairs of endpoints in the unit interval: $\cPinterval = \{(\ell, u): \ell \leq u, \; u, \ell \in [0,1]\}$. 
Given a pair $(\ell,u) \in \cPinterval$, we say that it \emph{covers} a label $y \in [0,1]$ if $y$ falls between $\ell$ and $u$, which we write as $\mathrm{Cover}((\ell,u),y) = 1$. To avoid issues of ``double counting'', we define coverage in the same manner as we defined our bucketing, using intervals that are closed on the left but open on the right, with the exception of $u = 1$:
$$\mathrm{Cover}((\ell,u), y) = \begin{cases} \ind(y \in [\ell,u)) &\mbox{if } u < 1, \\
\ind(y \in [\ell,u]) & \mbox{if } u = 1. \end{cases}$$

In each round $t$, we will predict an interval $p_t = (\bell_t, \bu_t)$ with the goal of achieving $\E[\mathrm{Cover}((\bell_t,\bu_t),y) | x_t] = 1-\delta$ for some target coverage probability $1-\delta \in [0,1]$. 
We again bucket our coverage intervals using a discretization parameter $n$, using the same notation as for moment predictions. 

For any $S \subseteq [T]$ and $i \leq j \in [n]$, we write
\[
    S(i,j) = \left\{t \in S: (\bell_t,\bu_t) \in \Bm(i,j) \right\}.
\]
In words, $S(i,j)$ corresponds to the subset of rounds in $S$ in which our predicted interval's endpoints are in buckets $i$ and $j$, respectively. We can now define multivalidity analogously to how we defined multicalibration.

For any $S \subseteq [T]$, we write
\[  
    \bH(S) = \frac{1}{|S|} \sum_{t \in S} \mathrm{Cover}((\bell_t,\bu_t),y_t).
\]

\begin{definition}
We say that interval predictions $\{(\bell_t, \bu_t)\}_{t=1}^T$ are \emph{$\alpha$-consistent} on set $S$ with respect to failure probability $\delta \in (0,1)$, if the following holds:
\[
    |\bH(S) - (1-\delta)| \le \alpha \frac{T}{|S|}.
\]
\end{definition}

\begin{definition}
\label{def:multivalid}
Given a transcript $\pi_T$, we say that the interval predictions are \emph{$(\alpha,n)$-multivalid with respect to $\delta$ and $\cG$}, if for every $i\le j \in [n]$ and $G \in \cG$, we have that the interval predictions are  $\alpha$-consistent on $G_T(i,j)$ with respect to coverage probability $1-\delta$:
\[
    |\bH(G_T(i,j)) - (1-\delta)| \le \alpha \frac{T}{|G_T(i,j)|}.
\]
\end{definition}

\subsection{Zero-sum Games}
Our analysis will hinge on properties of zero-sum games, and in particular on the minimax theorem. 
\begin{definition}\label{def:zerosumgame}
A zero-sum game is defined by:
\begin{enumerate}
    \item A \emph{minimization} player with a convex and compact strategy space $\mathcal{Q}_1 \subseteq \mathbb{R}^{d_1}$ for some $d_1 \in (0, \infty)$.
    \item A \emph{maximization} player with a convex and compact strategy space $\mathcal{Q}_2 \subseteq \mathbb{R}^{d_2}$ for some $d_2 \in (0, \infty)$.
    \item An \emph{objective function} $u:\mathcal{Q}_1\times \mathcal{Q}_2\rightarrow \mathbb{R}$, concave in its first argument and convex in its second argument. 
\end{enumerate}
\end{definition}

Zero-sum games are often defined by endowing each player with a finite set of \emph{pure} strategies $X_1, X_2$. The convex compact strategy sets $\mathcal{Q}_1$ and $\mathcal{Q}_2$ are then formed by allowing players to randomize over their pure strategies and taking $\mathcal{Q}_1 = \Delta X_1$, $\mathcal{Q}_2 = \Delta X_2$ to be the probability simplices over the pure strategies of each player. An objective function $u:X_1 \times X_2 \rightarrow \mathbb{R}$ can be linearly extended to $\mathcal{Q}_1$ and $\mathcal{Q}_2$ in the natural way (i.e. by taking expectations over the randomized strategies of each player) -- i.e. for any $Q_1 \in \mathcal{Q}_1$ and $Q_2 \in \mathcal{Q}_2$, we write $u(Q_1, Q_2) = \E_{x_1 \sim Q_1, x_2 \sim Q_2}[u(x_1, x_2)]$.

In a zero-sum game, the minimization player chooses some action $Q_1 \in \mathcal{Q}_1$ and the maximization player chooses some action $Q_2 \in \mathcal{Q}_2$, resulting in objective value $u(Q_1,Q_2)$. The goal of the minimization player is to minimize the objective value, and the goal of the maximization player is to maximize it. The key property of zero-sum games, first proved by von Neumann for the case of games with finite sets of pure strategies and generalized to general zero-sum games of the form considered in Definition~\ref{def:zerosumgame} by Sion, is that the order of play does not affect the objective value that each player can guarantee. This is captured in the minimax theorem, which says that whether the minimization player \emph{first} gets to observe the strategy of the maximization player, and \emph{then} best respond, or whether she must first announce her strategy and allow the maximization player to best respond, she is able to guarantee herself the same value. 
\begin{theorem}[Sion's Minimax Theorem]
\label{thm:minimax}
For any zero-sum game $(\mathcal{Q}_1,\mathcal{Q}_2,u)$:
$$\min_{Q_1 \in \mathcal{Q}_1}\max_{Q_2 \in \mathcal{Q}_2} u(Q_1,Q_2) = \max_{Q_2 \in \mathcal{Q}_2}\min_{Q_1 \in \mathcal{Q}_1} u(Q_1,Q_2).$$
\end{theorem}
The minimax theorem justifies the following definitions:
\begin{definition}[Value, Equilibrium, and Best Response]
The \emph{value} of a zero-sum game $(\mathcal{Q}_1,\mathcal{Q}_2,u)$ is the unique $v \in \mathbb{R}$ such that 
$$\min_{Q_1 \in \mathcal{Q}_1}\max_{Q_2 \in \mathcal{Q}_2} u(Q_1,Q_2) = \max_{Q_2 \in \mathcal{Q}_2}\min_{Q_1 \in \mathcal{Q}_1} u(Q_1,Q_2) = v.$$
We say that a strategy for the minimization player $Q^*_1 \in \mathcal{Q}_1$ is a (minimax) equilibrium strategy if it guarantees that the objective value is at most the value of the game, for any strategy $Q_2 \in \mathcal{Q}_2$ of the maximization player:
$$\max_{Q_2 \in \mathcal{Q}_2} u(Q^*_1,Q_2) = v.$$ 
We say that $Q_2$ is a best response for the maximization player in response to $Q^*_1$ if it realizes the above maximum. 
\end{definition}

In our analysis, we will identify the Learner with the minimization player and the Adversary with the maximization player, and so will denote their strategy spaces as $\mathcal{Q}^L$ and $\mathcal{Q}^A$ respectively. 


\section{Online Mean Multicalibration}
\label{sec:onlinemeanmulti}
In this section, we show how to obtain mean multicalibrated estimators  in an online adversarial setting. Our derivation also serves as a warm up example of our general technique, which we also instantiate (in somewhat more involved settings) in Sections \ref{sec:onlinemoment} and \ref{sec:onlinemultivalid} to derive online algorithms for mean-conditioned moment multicalibrated estimators and for multivalid prediction intervals respectively. 

\subsection{An Outline of Our Approach}\label{sec:mean-outline}
At a high level, the derivation of our algorithm and its proof of correctness proceeds as follows:
\begin{enumerate}
    \item For each group $G \in \cG$,  $i \in [n]$, and transcript $\pi_s$ up to period $s$, we define an empirical quantity $V_s^{G,i}$  (Definition~\ref{def:meanerror}) which represents the calibration error that our algorithm has incurred with respect to group $G$ over those of the rounds $1$ through $s$ when the $i\tth$ bucket was predicted. These quantities are defined so that if for each $G$ and $i$, $|V_T^{G,i}|$ is small, then our algorithm is approximately multicalibrated with respect to $\cG$ across $T$ rounds. 
    
    The premise of our algorithm will be to greedily make decisions at each round $s$ so as to minimize the \emph{maximum possible increase} of these quantities ($\max_{G,i}|V_{s+1}^{G,i}| -\max_{G,i}|V_{s}^{G,i}|$), in the worst case over the choices of the adversary. If we could bound this quantity at every round, then by telescoping, we would have a bound on $\max_{G,i}|V_T^{G,i}|$ at the end of the interaction, and therefore a guarantee of mean multicalibration. 

\item The increase in the maximum value of $|V_{s+1}^{G,i}|$ is inconvenient to work with, and so we instead define a smooth potential function $L_s$ (Definition~\ref{def:meansurrogate}) corresponding to a soft-max function which upper bounds $\max_{G,i} |V_{s}^{G,i}|$. Our design goal instead becomes to upper bound the increase in our potential function from round to round: $\Delta_{s+1} = L_{s+1}-L_s$. We view this as defining a zero-sum game, in which the learner's goal is to minimize this increase, and the adversary's goal is to maximize it. 
\item We show that for each fixed distribution that the adversary could employ at each round $s+1$, there is a prediction the learner could employ (if only she knew the adversary's distribution) that would guarantee that the increase in potential $\Delta_{s+1}$ is small. Intuitively, this is because if we knew the true joint distribution over feature label pairs, then we could predict the true conditional expectations, $\bmu_{s+1} = \E[y_{s+1} | x_{s+1}]$, which would be perfectly calibrated on all groups. Of course, the learner does not have the luxury of knowing the adversary's distribution before choosing her own. But this thought experiment establishes the value of the game, and so we can conclude via the minimax theorem that there must be some fixed \emph{distribution} over prediction rules that the learner can play that will guarantee $\Delta_{s+1}$ being small against \emph{all} actions of the adversary. \label{step:minimax}

\item Step~\ref{step:minimax} suffices to argue for the \emph{existence} of an algorithm obtaining multicalibration guarantees (Algorithm~\ref{alg:meanexistential}). However, to actually derive an implementable algorithm we need to find a way to compute the equilibrium strategy at each round, whose existence was argued in Step~\ref{step:minimax}. A priori, this seems daunting because the learner's strategy space consists of all randomized mappings between $\cX$ and $\cY$, and the adversary's strategy space consists of all joint distributions on $\cX\times \cY$. However, we derive a simple algorithm in Section~\ref{sec:meanalg} that implements the optimal equilibrium strategy needed to realize Step~\ref{step:minimax}. Informally, we are able to do so by representing the mapping between $\cX$ and $\cY$ only implicitly, and delaying all computation until $x_t$ has been chosen. We then show that the equilibrium strategy for the learner has a simple structure and randomizes over only at most 2 predictions. Our final algorithm (Algorithm~\ref{alg:meantemplate}) simply computes the relevant portion of the equilibrium strategy at each round and then samples from it. 
\item To apply the minimax theorem, and to derive a concrete algorithm, we need to restrict our algorithm to making predictions in $[0,1]$ that are discretized at units of $1/rn$ for some $r > 1$. This parameter $r$ appears in our final bounds, but neither the runtime of our algorithm nor our convergence rate has any dependence on $r$, and so it can be imagined to be arbitrarily small. Taking it to be $r = 1/\sqrt{T}$ causes it to become a low order term in our final bounds. 
\end{enumerate}

Finally, in Appendix~\ref{sec:batch}, we give a standard online-to-offline conversion to show how to use our Algorithm~\ref{alg:meantemplate} to solve offline (batch) multicalibration problems. This gives optimal sample complexity bounds for the offline problem, yielding an improvement over those proven in \cite{multicalibration,momentmulti}. The crux of the improvement is that unlike the algorithms given in \cite{multicalibration,momentmulti}, our algorithm takes only a single pass over the data, and so avoids complications that arise from data re-use. However, unlike previous batch algorithms which make deterministic predictions, the batch algorithm that we obtain through this reduction makes randomized predictions.

\subsection{An Existential Derivation of the Algorithm and Multicalibration Bounds}
\label{sec:meanexistential}

We begin by defining notation $V_s^{G,i}$ for the (unnormalized) \emph{portion} of the mean calibration error corresponding to each group $G \in \cG$ and bucket $i \in [n]$: 
\begin{definition}
\label{def:meanerror}
Given a transcript $\pi_s = ((x_t,\bmu_t,y_t))_{t=1}^s$,  we define the mean calibration error for a group $G \in \cG$ and bucket $i \in [n]$ at time $s$ to be:
\begin{align}
\label{meanSurrogateLoss}
V_s^{G,i}(\pi_s) = \left|G_s(i)\right| \left(\mu\left(G_s(i)\right) -  \bmu\left(G_s(i)\right)\right) = \sum_{t=1}^s \ind[\bmu_t \in B(i), x_t \in G] \left( y_t - \bmu_t\right)
\end{align}
When the transcript is clear from context we will sometimes simply write $V^{G,i}_s$. 
\end{definition}

Observe that our definition of mean multicalibration (Definition~\ref{def:meanmulti}) corresponds to asking that $|V_s^{G,i}|$ be small for all $i ,G$.

\begin{obs} 
\label{obs:meancalibration}
Fix a transcript $\pi_T$. 
If for all $G \in \cG$, $i \in [n]$, we have that:
$$\left|V_T^{G,i}\right| \leq \alpha T,$$
then the corresponding sequence of predictions is $(\alpha,n)$-mean multicalibrated with respect to $\cG$. 
\end{obs}

We next define a surrogate loss function that we can use to bound our calibration error. 
\begin{definition}[Surrogate loss function]
\label{def:meansurrogate}
Fixing a transcript $\pi_s \in \Pi^*$ and a parameter $\eta \in [0,\frac12]$, define a surrogate calibration loss function at day $s$ as:
$$L_s(\pi_s) = \sum_{\substack{G \in \cG,\\i \in [n]}}\left(\exp(\eta V_s^{G,i}) + \exp(-\eta V_s^{G,i}) \right).$$
When the transcript $\pi_s$ is clear from context, we will sometimes simply write $L_s$. 
\end{definition}
We will leave $\eta$ unspecified for now, and choose it later to optimize our bounds.  Observe that this ``soft-max style" function allows us to tightly upper bound our calibration loss:
\begin{obs} For any transcript $\pi_T$, and any $\eta \in [0,\frac12]$, we have that:
$$\max_{G \in \cG, i \in [n]}\left| V_T^{G,i}\right|\leq \frac{1}{\eta}\ln(L_T) \leq \max_{G \in \cG, i \in [n]}\left| V_T^{G,i}\right| + \frac{\ln\left(2 |\cG|n\right)}{\eta}.$$
\end{obs}

 Part of our analysis will depend on viewing the transcript as a random variable: in this case, in keeping with our convention for random variables, we refer to it as $\tpi$. The associated random variables tracking calibration and surrogate loss are denoted $\tV$ and  $\tL$ respectively.

Our goal is to find a strategy for the learner that guarantees that our surrogate loss $L_T$ remains small. Towards this end,  we define $\Delta_{s+1}(\pi_s, x_{s+1}, \bmu_{s+1})$ to be the expected increase in the surrogate loss function in the event that the adversary plays feature vector $x_{s+1}$ \emph{and} the learner plays prediction $\bmu_{s+1}$.  Here the expectation is over the only remaining source of randomness after the conditioning --- the distribution over labels $y_{s+1}$ (which we observe is determined, once we fix $\pi_s$ and $x_{s+1}$).

\begin{definition}[Conditional Change in Surrogate Loss]
\[\Delta_{s+1}(\pi_s,x_{s+1}, \bmu_{s+1}) = \E_{\ty_{s+1}}\left[\tL_{s+1}-L_s \middle\vert x_{s+1}, \bmu_{s+1}, \pi_s\right].\]
\end{definition}

We begin with a simple bound on $\Delta_{s+1}(\pi_s,x_{s+1}, \bmu_{s+1})$:
\begin{lemma}
\label{lem:boundincrease}
For any transcript $\pi_s \in \Pi^*$, any $x_{s+1} \in \cX$, and any $\bmu_{s+1} \in \cPmean$ such that $\bmu_{s+1} \in B(i)$ for some $i \in [n]$:
    \[
    \Delta_{s+1}(\pi_s,x_{s+1}, \bmu_{s+1}) \leq \eta \left(\E_{\tilde{y}_{s+1}}[\tilde{y}_{s+1}]-\bmu_{s+1}\right) C_s^i(x_{s+1}) + 2\eta^2\cdotk  L_{s},
    \]
where for each $i \in [n]$:
\begin{align}\label{eqn:csi}
C_s^{i}(x_{s+1})  \equiv \sum_{\cG(x_{s+1})}\exp(\eta V_s^{G,i}) -  \exp(-\eta V_s^{G,i}).
\end{align}

\end{lemma}
\begin{proof}
Fix any transcript $\pi_s \in \Pi^*$ (which defines $L_s$), feature vector $x_{s+1} \in \cX$, and $\bmu_{s+1}$ such that $\bmu_{s+1} \in B(i)$ for some $i \in [n]$. By direct calculation, we obtain:
\begin{align*}
&\,\,\;\Delta_{s+1}(\pi_s,x_{s+1}, \bmu_{s+1})\\ 
=& \E_{\ty_{s+1}}\left[\sum_{G \in \cG(x_{s+1})}\exp(\eta V_s^{G,i})\left(\exp(\eta (\ty_{s+1}-\bmu_{s+1}))-1\right) + \exp(-\eta V_s^{G,i})\left(\exp(-\eta (\ty_{s+1}-\bmu_{s+1}))-1\right)\right], \\
\leq& \E_{\ty_{s+1}}\left[\sum_{G \in \cG( x_{s+1})}\exp(\eta V_s^{G,i})\left(\eta (\ty_{s+1}-\bmu_{s+1})+2\eta^2\right) + \exp(-\eta V_s^{G,i})\left(-\eta (\ty_{s+1}-\bmu_{s+1})+2\eta^2\right)\right], \\
=& \eta \left(\E_{\ty_{s+1}}[\ty_{s+1}]-\bmu_{s+1}\right) \sum_{G \in \cG(x_{s+1})} \left(\exp(\eta V_s^{G,i}) \!-\!  \exp(-\eta V_s^{G,i})\right) + 2\eta^2 \!\! \sum_{G \in \cG(x_{s+1})}\left(\exp(\eta V_s^{G,i}) + \exp(-\eta V_s^{G,i}) \right), \\
\leq& \eta \left(\E_{\ty_{s+1}}[\ty_{s+1}]-\bmu_{s+1}\right)\cdotk \left(\sum_{G \in \cG(x_{s+1})}\exp(\eta V_s^{G,i}) -  \exp(-\eta V_s^{G,i})\right) + 2\eta^2\cdotk  L_{s},\\
=&\eta \left(\E_{\ty_{s+1}}[\ty_{s+1}]-\bmu_{s+1}\right) C_s^i(x_{s+1})  + 2\eta^2\cdotk  L_{s}.
\end{align*}
Here, the first inequality follows from the fact that for $0 < |x| < \frac{1}{2}$, $\exp(x) \leq 1+x+2x^2$. 
\end{proof}

Using this bound, we define a zero-sum game between the learner and the adversary and use the minimax theorem to conclude that the learner always has a strategy that guarantees that the per-round increase in surrogate loss can be bounded. To satisfy the convexity and compactness requirements of the minimax theorem, it will be convenient for us to imagine that the learner's pure strategy space is a finite, discrete subset of $\cPmean = [0,1]$. To this end, we define the following discretization for any $r \in \mathbb{N}$ (here $n$ is the discretization parameter we use to define the coarseness of our bucketing): 
\[
    \cP^{rn} = \left\{0, \frac{1}{rn}, \frac{2}{rn}, \dots, 1\right\}.
\]

We use this discretization also in our algorithm in Section~\ref{sec:meanalg} --- but we remark at the outset that the need to discretize is only for technical reasons, and our algorithm will have no dependence --- neither in runtime nor in its convergence rate --- on the value of $r$ that we choose, so we can imagine the discretization to be arbitrarily fine.

To simplify notation, for each $\bmu \in \cP^{rn}$, define $C^{\bmu}_s \equiv C^i_s$ where $i \in [n]$ s.t. $\bmu \in \Bm(i)$.

\begin{lemma}
\label{lem:exists}
For any transcript $\pi_s \in \Pi^*$, any $x_{s+1} \in \cX$, and any $r \in \mathbb{N}$ there exists a distribution over predictions for the learner $Q^L_{s+1} \in \Delta \cP^{rn}$, such that regardless of the adversary's choice of distribution of $y_{s+1}$ over $\Delta \cY$,  we have that:
\[
    \E_{\bmu \sim Q^L_{s+1}}\left[\Delta_{s+1}(\pi_s,x_{s+1}, \bmu)\right] \le L_s\cdotk \left(\frac{\eta}{rn}+2\eta^2\right).
\]
\end{lemma}
\begin{proof}

 We define a zero-sum game played between the learner (the minimization player) and the adversary (the maximization player). The learner's pure strategy space is the set of discrete predictions $X_1 = \cP^{rn}$. The adversary's pure strategy space is (a priori) the set of all distributions over labels in $[0,1]$. However, we will observe in a moment that the objective function of our game depends only on the \emph{expected value} of the label, and so without loss of generality, we will be able to take the adversary's full strategy space to be the set of all pure strategies, i.e., $\mathcal{Q}^A = [0,1]$ (which is closed and convex), because it already spans the set of realizable expectations. As usual, we take the learner's full strategy space to be the set of distributions over pure strategies: $\mathcal{Q}^L = \Delta \cP^{rn}$.

Fix the transcript $\pi_s$ and the feature vector $x_{s+1}$. We define the objective of this game to be the upper bound we proved on $ \Delta_{s+1}(\pi_s,x_{s+1}, \bmu)$ in Lemma~\ref{lem:boundincrease}. For each $\bmu \in \cP^{rn}$ and each $y \in [0,1]$, we let:
\[
    u(\bmu, y) = \eta \left(y-\bmu\right) C^{\bmu}_s(x_{s+1})   + 2\eta^2\cdotk  L_{s}.
\]
Note that for any distribution over labels $y$ of the adversary, the expected objective value depends on his strategy only through $\E[\ty]$ because the above objective function is linear in $y$: that is, $\E_{\ty}[u(\bmu,\ty)] = u(\bmu, \E[\ty])$.
Thus we are justified in our reduced-form representation of the adversary's full strategy as choosing $\E[\ty]$ in the interval $[0,1]$.

We now establish the value of this game. Observe that for any strategy of the adversary (which fixes $\E[\ty])$, the learner can respond by playing $\bmu^* = \argmin_{\bmu \in \cP^{rn}} |\E[\ty] - \bmu|$, and that because of our discretization, $\min |\E[\ty] - \bmu^*| \leq \frac{1}{rn}$. Therefore, the value of the game is at most: 
\begin{eqnarray*}
\max_{y \in [0,1]} \min_{\bmu^* \in \cP^{rn}} u(\bmu^*,y) &\leq& \max_{\bmu \in \cP^{rn}} \frac{\eta}{rn} \left|C^{\bmu}_s(x_{s+1}) \right| + 2\eta^2\cdotk  L_{s}, \\
&\leq& L_s\cdotk \left(\frac{\eta}{rn} + 2\eta^2\right).
\end{eqnarray*}
Here the latter inequality follows since $C^{\bmu}_s(x_{s+1})  \leq L_s$ for all $\bmu \in \cP^{rn}$, by observation.  We can now apply the minimax theorem (Theorem~\ref{thm:minimax}) to conclude that there exists a fixed distribution $Q^L_{s+1} \in \mathcal{Q}^L$ for the learner that guarantees that simultaneously for \emph{every} label $y \in [0,1]$ that might be chosen by the adversary:
$$\E_{\bmu \sim Q^L_{s+1}}\left[u(\bmu,y)\right] \leq L_s\cdotk \left(\frac{\eta}{rn}+2\eta^2\right),$$
as desired.
\end{proof}

\begin{corollary}\label{cor:exists-mean}
For every $r \in \mathbb{N}$, $s \in [T]$, $\pi_s \in \Pi^*$, and $x_{s+1} \in \cX$ (which fixes $L_s$ and $Q^L_{s+1}$), and any distribution over $\cY$:

\[
    \E_{\bmu_{s+1} \sim Q^L_{s+1}}[\tL_{s+1}|\pi_s] = L_s + \E_{\bmu_{s+1}\sim Q^L_{s+1}}[\Delta_{s+1}(\pi_{s},x_{s+1},\bmu_{s+1})] \leq L_s\cdotk\left(1 + \frac{\eta}{rn} + 2\eta^2\right).
\]
\end{corollary}

Lemma~\ref{lem:exists} defines (existentially) an algorithm that the learner can use to make predictions---Algorithm~\ref{alg:meanexistential}. We will now show that Algorithm~\ref{alg:meanexistential} (if we could compute the distributions $Q^L_t$) results in multicalibrated predictions. In Section~\ref{sec:meanalg} we show a simple and efficient method for sampling from $Q^L_t$.

\begin{algorithm}[H]
\SetAlgoLined
\begin{algorithmic}
\FOR{$t=1, \dots, T$}
	\STATE Observe $x_t$. Given $\pi_{t-1}$ and $x_t$, let $Q^L_t \in \mathcal{Q}^L_t$ be the distribution over predictions whose existence \\ 
	is established in Lemma~\ref{lem:exists}.
	\STATE Sample $\bmu \sim Q^L_t$ and predict $\bmu_t = \bmu$
\ENDFOR
\end{algorithmic}
\caption{A Generic Multicalibrator}
\label{alg:meanexistential}
\end{algorithm}

We now prove two convergence bounds for Algorithm~\ref{alg:meanexistential}. The first will bound its multicalibration error \emph{in expectation}, and the other will provide a high probability bound. To show these bounds, we first state a helper theorem that will be useful not just in this section, but also in deriving the final convergence bounds for the algorithms presented in Sections \ref{sec:onlinemoment} and \ref{sec:onlinemultivalid}. The proof is in Appendix \ref{app3}.

\newcommand{\tX}{\tilde{X}}
\begin{restatable}{theorem}{helperthm}
\label{thm:general-bounds}
Consider a nonnegative random process $\tX_t$ adapted to the filtration $\mathcal{F}_t= \sigma(\pi_t)$, where $\tX_0$ is constant a.s.
Suppose we have that for any period $t,$ and any $\pi_{t-1}$, $\E[\tX_t|\pi_{t-1}] \leq X_{t-1} (1 + \eta c + 2\eta^2)$ for some $\eta \in [0,\frac12], c \in [0,1]$.
Then we have that: 
\begin{equation}
    \E_{\tpi_T}[\tX_T] \leq X_0 \exp \left(T\eta c+ 2T\eta^2 \right). \label{eqn:expectedbound}
\end{equation}
Further, define a process $\tZ_t$ adapted to the same filtration by $\tZ_t = Z_{t-1} + \ln \tX_t - \E [\ln (\tX_t)|\pi_{t-1}]$. Suppose that $|Z_t - Z_{t-1}| \leq 2\eta$, where $Z_0=0$ a.s. Then, with probability $1-\lambda$,  
\begin{equation}
    \ln(X_T(\pi_T)) \le \ln(X_0) + T \left(\eta c + 2\eta^2\right) + \eta \sqrt{8 T \ln\left(\frac{1}{\lambda}\right) }.  \label{eqn:hpbound}
\end{equation}
\end{restatable}

We are now ready to bound our multicalibration error. As a straightforward consequence of Corollary~\ref{cor:exists-mean} and the first part of Theorem~\ref{thm:general-bounds}, we have the following Corollary. 
\begin{corollary}
\label{cor:surrogateloss-mean}
Against any adversary, Algorithm~\ref{alg:meanexistential} instantiated with discretization parameter $r$ results in surrogate loss satisfying:
\[
    \E_{\tpi_T}[\tL_T] \le 2|\cG|n\cdotk \exp\left(\frac{T\eta}{rn}+2T\eta^2\right).
\]
\end{corollary}
\begin{proof}
Note that the first part of Theorem~\ref{thm:general-bounds} applies to the process $L$ with $L_0 = 2 |\cG|n$ and $c = \frac{1}{rn}$. The bound follows by plugging these values into \eqref{eqn:expectedbound}. 
\end{proof}

Next, we can convert this into a bound on Algorithm~\ref{alg:meanexistential}'s expected calibration error:

\begin{theorem}
\label{thm:meanmulti}
When Algorithm~\ref{alg:meanexistential} is run using $n$ buckets for calibration, discretization $r \in \mathbb{N}$, and $\eta = \sqrt{\frac{\ln(2|\cG|n)}{2T}} \in (0, 1/2)$, then against any adversary, its sequence of mean predictions is $(\alpha,n)$-multicalibrated with respect to $\cG$, where:
\begin{equation*}
    \E[\alpha] \leq \frac{1}{rn} + 2\cdotk \sqrt{\frac{2\ln(2|\cG|n)}{T}}.
\end{equation*}
For $r = \frac{\sqrt{T}}{\epsilon n\sqrt{2\ln(2|\cG|n})}$ this gives:
$$\E[\alpha] \leq \left( 2+\epsilon\right) \cdotk \sqrt{\frac{2}{T} \ln\left(2|\cG|n \right)}.$$
Here the expectation is taken over the randomness of the transcript $\pi_T$.
\end{theorem}
\begin{proof}
From Observation~\ref{obs:meancalibration}, it suffices to show that $$\frac{1}{T} \E_{\tpi_T}\left[\max_{G \in \cG, i \in [n]}|\tV_T^{G,i}|\right] \leq \frac{1}{rn} + 2\cdotk \sqrt{\frac{2\ln(2|\cG|n)}{T}}.$$ 

We begin by computing a bound on the (exponential of) the expectation of this quantity:
\begin{eqnarray*}
\exp\left(\eta \cdotk \E_{\tpi_T}\left[\max_{G,i}|\tV_T^{G,i}|\right]\right)&\leq&\E_{\tpi_T}\left[\exp\left(\eta\cdotk \max_{G,i}|\tV_T^{G,i}|\right)\right], \\
&=&\E_{\tpi_T}\left[\max_{G,i}\exp\left(\eta\cdotk|\tV_T^{G,i}|\right)\right], \\
&\leq& \E_{\tpi_T}\left[\max_{G,i}\left(\exp\left(\eta\cdotk \tV_T^{G,i}\right)+\exp\left(-\eta\cdotk \tV_T^{G,i}\right)\right)\right], \\
&\leq& \E_{\tpi_T}\left[\sum_{G,i}\left(\exp\left(\eta\cdotk \tV_T^{G,i}\right)+\exp\left(-\eta\cdotk \tV_T^{G,i}\right)\right)\right] ,\\
&=& \E_{\tpi_T}[\tL_T], \\
&\leq& 2|\cG|n\cdotk \exp\left(\frac{T\eta}{rn}+2T\eta^2\right).
\end{eqnarray*}
Here the first step is by Jensen's inequality and the last one follows from Corollary~\ref{cor:surrogateloss-mean}. Taking the logarithm of both sides and dividing by $\eta\cdotk T$, we have
$$\frac{1}{T}\E_{\tpi_T}\left[\max_{G,i}|\tV_T^{G,i}|\right] \leq \frac{\ln(2|\cG|n)}{\eta\cdotk T} + \frac{1}{rn} + 2\eta.$$
Choosing $\eta = \sqrt{\frac{\ln(2|\cG|n)}{2T}}$, we thus obtain the desired inequality
$$ \frac{1}{T}\E_{\tpi_T}\left[\max_{G,i}|\tV_T^{G,i}|\right]\leq \frac{1}{rn} + 2\cdotk \sqrt{\frac{2\ln(2|\cG|n)}{T}}. \qedhere $$\end{proof}

    
Now, given $\tilde{L}$, let us define its associated martingale process $\tilde{Z}$ as in the second part of Theorem~\ref{thm:general-bounds}. The next lemma shows that the increments of $\tilde{Z}$ are uniformly bounded over all rounds $t$. The proof is in Appendix \ref{app3}.
\begin{restatable}{lemma}{martingalebounded}
\label{lem:martingale-bounded-mean}
At any round $t \in [T]$ and for any realized transcript $\pi_t$, $|Z_t - Z_{t-1}|\le 2\eta.$
\end{restatable}

We can now use the second part of Theorem~\ref{thm:general-bounds} to prove a high probability bound on the multicalibration error of Algorithm~\ref{alg:meanexistential}.

\begin{theorem}\label{thm:hpcalibration}
When Algorithm~\ref{alg:meanexistential} is run using $n$ calibration buckets, discretization $r \in \mathbb{N}$ and $\eta = \sqrt{\frac{\ln(2|\cG|n)}{2T}} \in (0, 1/2)$, then against any adversary, its sequence of mean predictions is $\alpha$-multicalibrated, with respect to $\cG$ with probability $1-\lambda$ over the randomness of the transcript $\pi_T$, for
\[ 
    \alpha \leq \frac{1}{rn} +  4 \cdotk \sqrt{\frac{2}{T} \ln\left(\frac{2|\cG|n}{\lambda} \right)}.
\]
Choosing $r = \frac{\sqrt{T}}{\epsilon n\sqrt{2\ln(2|\cG|n/\lambda})}$, this gives:
$$\alpha \leq \left( 4 +\epsilon\right) \cdotk \sqrt{\frac{2}{T} \ln\left(\frac{2|\cG|n}{\lambda} \right)}.$$
\end{theorem}
\begin{proof}
By Lemma~\ref{lem:martingale-bounded-mean}, the second part of Theorem~\ref{thm:general-bounds} applies; plugging in  $L_0 = 2|\cG|n$ and $c = \frac{1}{rn}$, we have:
 \begin{align*}
& \ln(L_T(\pi_T)) \le \ln(2|\cG|n) + T \left(\frac{\eta}{rn} + 2\eta^2\right) + \eta \sqrt{8 T \ln\left(\frac{1}{\lambda}\right) }.
\end{align*}
Now, note that 
\begin{align*}
\exp\left(\eta\cdotk \max_{G,i}|V_T^{G,i}|\right) &=\max_{G,i}\exp\left(\eta\cdotk|V_T^{G,i}|\right), \\
&\leq \max_{G,i}\left(\exp\left(\eta\cdotk V_T^{G,i}\right)+\exp\left(-\eta\cdotk V_T^{G,i}\right)\right), \\
&\leq \sum_{G,i}\left(\exp\left(\eta\cdotk V_T^{G,i}\right)+\exp\left(-\eta\cdotk V_T^{G,i}\right)\right), \\
&= L_T(\pi_T). 
\end{align*}
Taking log on both sides and dividing both sides by $\eta T$, we get
\begin{align*}
    \frac{1}{T} \max_{G,i} |V^{G,i}_{T}| \le \frac{1}{\eta T} \ln(L_T(\pi_T))
    \le \frac{\ln(2|\cG|n)}{\eta\cdotk T} + \frac{1}{rn} + 2\eta + \sqrt{\frac{8  \ln\left(\frac{1}{\lambda}\right)}{T}}.
\end{align*}
Choosing $\eta = \sqrt{\frac{\ln(2|\cG|n)}{2T}}$, we thus obtain the desired inequality
\[ 
\frac{1}{T} \max_{G,i}|V_T^{G,i}| \le \frac{1}{rn} + 2 \sqrt{\frac{2\ln(2|\cG|n)}{T}} + \sqrt{\frac{8  \ln\left(\frac{1}{\lambda}\right)}{T}} \le \frac{1}{rn} +  4 \cdotk \sqrt{\frac{2}{T} \ln\left(\frac{2|\cG|n}{\lambda} \right)}. \qedhere
\]
\end{proof}

\begin{remark}
In both Theorems \ref{thm:meanmulti} and \ref{thm:hpcalibration}, the dependence on $\log(|\cG|)$ can be replaced with a dependence on $\log(d)$ under the assumption that $|\cG(x_t)| \leq d$ for all $t$ --- i.e. that each observed data point is contained in only boundedly many groups. This gives us non-trivial guarantees even when $\cG$ is infinitely large. See Appendix \ref{sec:moregroups} for details. 
\end{remark}

\subsection{Deriving an Efficient Algorithm via Equilibrium Computation}
\label{sec:meanalg}

\begin{algorithm}[H]
\SetAlgoLined
\begin{algorithmic}
\FOR{$t=1, \dots, T$}
	\STATE Observe $x_t$ and compute for each $i \in [n]$	$C^i_{t-1}(x_{t}) $ as defined in \eqref{eqn:csi}.
    \IF {$C^i_{t-1}(x_{t}) > 0$ for all $i \in [n]$}
        \STATE Predict $\bmu_t= 1$.
    \ELSIF {$C^i_{t-1}(x_{t}) < 0$ for all $i \in [n]$}
        \STATE Predict $\bmu_t = 0$.  
    \ELSE 
        \STATE Find $i^* \in [n-1]$ such that $C^{i^*}_{t-1}(x_{t}) \cdot C^{i^*+1}_{t-1}(x_{t}) \leq 0$
        \STATE Define $0 \leq q_t \leq 1$ such that $q_t C^{i^*}_{t-1}(x_t) + (1-q_t) C^{i^*+1}_{t-1}(x_t) = 0$. In other words, define it as follows (using the convention that 0/0 = 1):
        $$q_t = \frac{|C^{i^*+1}_{t-1}(x_{t})|}{|C^{i^*+1}_{t-1}(x_{t})| + |C^{i^*}_{t-1}(x_{t})|}.$$
       
        \STATE Predict $\bmu_t = \frac{i^*}{n}- \frac{1}{rn}$ with probability $q_t$ and $\bmu_t = \frac{i^*}{n}$ with probability $1-q_t$.
\ENDIF
\ENDFOR
\end{algorithmic}
\caption{Von Neumann's Mean Multicalibrator($\eta,n,r$)}
\label{alg:meantemplate}
\end{algorithm}

\bigskip

In Section~\ref{sec:meanexistential}, we derived Algorithm~\ref{alg:meanexistential} and proved that it results in mean multicalibrated predictions. However, Algorithm~\ref{alg:meanexistential} was not defined explicitly: it relies on the distributions $Q^L_t$, whose existence we showed in Lemma~\ref{lem:exists} but which we did not explicitly construct. In this section, we derive a scheme for sampling from these distributions $Q^L_t$, which leads to Algorithm~\ref{alg:meantemplate} --- an explicit, efficient implementation of Algorithm~\ref{alg:meanexistential}.

\begin{theorem}
Algorithm~\ref{alg:meantemplate} implements Algorithm~\ref{alg:meanexistential}. In particular it obtains the multicalibration guarantees proven in Theorems \ref{thm:meanmulti} and \ref{thm:hpcalibration}. 
\end{theorem}
\begin{proof}
Recall that Algorithm~\ref{alg:meanexistential} samples at every round $s+1$ from a distribution $Q^L_{s+1}$ that is a minimax equilibrium strategy of a game between the learner and the adversary, with objective function
\begin{align*}
&u(\bmu, y) = \eta \left(y-\bmu\right)\cdotk C^{\bmu}_s(x_{s+1}) + 2\eta^2 L_s.\\
\intertext{The equilibrium structure of the game is preserved under positive affine transformations, so instead we consider}
&u(\bmu, y) =  \left(y-\bmu\right)\cdotk C^{\bmu}_s(x_{s+1}).
\end{align*}
We wish to find a distribution $Q^L_{s+1} \in \mathcal{Q}^L$ that guarantees --- against any strategy of the adversary --- an objective value that is at most the bound on the value of the game we proved in Lemma~\ref{lem:exists}. For the transformed game, this bound is:
$$\max_{y \in [0,1]} \E_{\bmu \sim Q_{s+1}}[u(\bmu,y)] \leq \frac{1}{rn}  L_s.$$

We can start by characterizing the best response of the adversary. 
\begin{obs} \label{obs:br}
    For any $Q^L \in \mathcal{Q}^L$:
    $$\max_{y \in [0,1]}\E_{\bmu \sim Q^L}[u(\bmu,y)]   = \left( \E_{\bmu \sim Q^L}[C^{\bmu}_s(x_{s+1})]\right)^+ - \E_{\bmu \sim Q^L}\left[\bmu C^{\bmu}_s(x_{s+1})\right],$$
    where $\left( x \right)^+ = \max(x,0)$.
\end{obs}
\begin{proof}
Note that:
\begin{align*}
u(\mu, y)   &=  \left(y-\bmu\right)C^{\bmu}_s(x_{s+1}) \\
&= y C^{\bmu}_s(x_{s+1}) - \bmu C^{\bmu}_s(x_{s+1}).
\end{align*}
Observe that only the first term depends on $y$. Therefore, if the learner plays according to $Q^L$, then the adversary will choose $y$ so as to maximize the linear expression $y \E_{\bmu \sim Q^L}[C^{\bmu}_s(x_{s+1})].$ This is  always maximized either at $y =0$ or $y=1$. It is maximized at $y = 1$ when $ \E_{\bmu \sim Q^L}[C^{\bmu}_s(x_{s+1})] > 0$, and at $y = 0$ otherwise.
\end{proof}

Finally, we can reduce the analysis to three disjoint cases:
\begin{enumerate}
    \item $C_s^i(x_{s+1})>0$ for all $i \in [n]$: Then for any distribution $Q^L$, by Observation~\ref{obs:br} we have:
    $$\max_{y \in [0,1]}\E_{\bmu \sim Q^L}[u(\bmu,y)]  = \E_{\bmu \sim Q^L}[C^{\bmu}_s(x_{s+1})] -  \E_{\bmu \sim Q^L}\left[\bmu C^{\bmu}_s(x_{s+1})\right].$$
In this case, letting $Q^L$ be a point mass on  $\bmu=1$  achieves a value of $ 0 < \frac{1}{rn} L_s$. 
    \item $C_s^i(x_{s+1})<0$ for all $i \in [n]$: Then for any distribution $Q^L$, by Observation~\ref{obs:br} we have:
    $$\max_{y \in [0,1]}\E_{\bmu \sim Q^L}[u(\bmu,y)]  = - \E_{\bmu \sim Q^L}\left[\bmu C^{\bmu}_s(x_{s+1})\right] $$
    In this case, letting $Q^L$ be a point mass on  $\bmu=0$  achieves a value of $ 0 < \frac{1}{rn} L_s$. 

    \item In the remaining case, there must exist some index $i^* \in [n-1]$ such that either  $C^{i^*}_s(x_{s+1})$ and $C^{i^*+1}_s(x_{s+1})$ have opposite signs, or such that at least one of them takes value exactly zero.  Randomizing as in the algorithm results in: 
    \begin{align*}
& \max_{y \in [0,1]} \E_{\bmu \sim Q^L_{s+1}}[u(\bmu,y)]  \\
=& \left( \E_{\bmu \sim Q^L_{s+1}}\left[C_s^{\bmu}(x_{s+1})\right]\right)^+ -  E_{\bmu \sim Q^L_{s+1}}\left[\bmu C_s^{\bmu}(x_{s+1})\right] \\
=&  \left( q_{s+1} \cdotk C_s^{i^*}(x_{s+1}) + (1-q_{s+1}) C_s^{i^*+1}(x_{s+1})\right)^+ - \left( q_{s+1}\left(\tfrac{i^*}{n} -\tfrac{1}{rn}\right) C_s^{i^*}(x_{s+1}) + (1-q_{s+1}) \tfrac{i^*}{n}C_s^{i^*+1}(x_{s+1}) \right) \\
=&  \frac{1}{rn} C_s^{i^*}(x_{s+1})\\
\leq& \frac{1}{rn} L_s.
\end{align*} 
\end{enumerate}
Algorithm~\ref{alg:meantemplate} plays according to this distribution $Q^L_{s+1}$ at every round, which completes the proof. 
\end{proof}

\paragraph{Running Time} 
Our algorithm is elementary, and given values for $C_{t-1}^i(x_t)$, it runs in time per iteration which is linear in the number of buckets $n$. For large collections of groups $\cG$, the bulk of the computational cost is due to the first step of Algorithm~\ref{alg:meantemplate}, in which we compute the quantities $C_{t-1}^i(x_t)$ as in Equation~\ref{eqn:csi}:
$$C_{t-1}^{i}(x_t) \equiv \sum_{\cG(x_{t})}\exp(\eta V_{t-1}^{G,i}) -  \exp(-\eta V_{t-1}^{G,i})$$
These quantities are a sum over every group $G \in \cG$ such that $x_{t} \in G$. In the worst case, we can compute this by enumerating over all such groups, and we obtain runtime that is linear in $|\cG|$. However, for any class $\cG$ such that we can efficiently enumerate the set of groups containing $x_{t}$ (i.e. $\cG(x_{t})$), our per-round runtime is only linear in $|\cG(x_{t})|$, which may be substantially smaller than $|\cG|$. For example, this property holds for collections $\cG$ of groups induced by conjunctions or disjunctions of binary features. Finally, we observe that our runtime is entirely independent of the choice of the discretization parameter $r$. 

\section{Online Moment Multicalibration}\label{sec:onlinemoment}
\subsection{An Outline of Our Approach}
\label{sec:moment-outline}

In this section, we derive an online algorithm for supplying mean and $k\tth$-moment predictions that are mean-conditioned moment multicalibrated with respect to some collection of groups $\cG$, as defined in Definition~\ref{def:mean-conditioned-moment-multi}. We follow the same basic strategy that we developed in Section~\ref{sec:onlinemeanmulti} for making multicalibrated mean predictions. In particular, 
the first few steps of our approach exactly mirror the approach in Section~\ref{sec:onlinemeanmulti}: Analogously to Steps 1 and 2 of Section~\ref{sec:mean-outline} we define calibration losses and a convenient soft-max style surrogate loss function and bound the increase to that surrogate loss function at each round. However, we make a couple of important deviations.
\begin{enumerate}

\item The first complication that arises is that moment consistency is not a linearly separable constraint across rounds (because moments are nonlinear). However, we are able to define linearly separable ``pseudo-moment'' consistency losses $M$ and prove in Lemma~\ref{lem:nonlinear} that if both our pseudo-moment consistency losses $M$ and our mean consistency losses $V$ are small then our predictions are mean-conditioned moment multicalibrated. 

\item The next complication arises when we attempt to define a zero-sum game using our bound on the per-round increase of the surrogate loss. The bound on the loss that we obtain for mean-conditioned moment multicalibration is nonlinear in both the learner's (mean) prediction and the adversary's choice of label $y$. We cannot directly apply a minimax theorem because the necessary concavity and convexity conditions are not satisfied. Our argument instead requires a change of variables: we show that in the game we define, the adversary's payoff, fixing the strategy of the learner, is linear in the first $k$ (uncentered) moments of the distribution over the labels chosen by the adversary.  We also expand the strategy space of the adversary to allow him to pick $k$ arbitrary real numbers, representing the first $k$ centered moments of his label distribution, unencumbered by the requirement that these chosen values actually correspond to the moments of any real label distribution. Enlarging the adversary's strategy space in this way can only \emph{increase} the value of the game, and so the upper bounds we prove on the value of this simplified game continue to hold for the original game. Moreover, a minimax theorem applies to this transformed game, and therefore guarantees the \emph{existence} of a prediction strategy for the learner that is approximately mean-conditioned moment multicalibrated. 

\item In order to implement this strategy with an explicit efficient prediction algorithm, we need to  solve a game in   which the learner has $r^2 n n'$ pure strategies. Doing this naively would inherit a running time dependence on $r$, a discretization parameter that we want to take to be very small. However, we prove a ``structure theorem'' about the enlarged game described above: that without loss of generality, the learner need only randomize over a support of at most $4nn'$ pure strategies. With this structure theorem in hand, we show that the equilibrium computation problem can be cast as a  linear program with $4n n'$ variables and $2^k+1$ constraints. If $k$ is a small constant (e.g. $k = 2$ for variance multicalibration), then this linear program can be explicitly described and solved. But even when $k$ is too large to enumerate all $2^k$ constraints, we show that there is a separation oracle that runs in time $O(k)$, allowing us to efficiently solve this linear program  using the Ellipsoid algorithm. In Appendix~\ref{sec:momentellipsoid}, we show that there exist solutions to the learner's problem that have small support---in which the learner mixes over at most $k+1$ strategies. 
\end{enumerate}

\subsection{An Existential Derivation of the Algorithm and Moment Multicalibration Bounds}
\label{sec:momentexistential}
We will calibrate our mean predictions $\{\bmu_{t}\}_{t=1}^T$ over $n$ buckets, and $k\tth$ moment predictions $\{\bmk\}_{t=1}^T$ over $n' < n$ buckets. 
As before, we introduce  notation to denote the \emph{portion} of the mean calibration error corresponding to each pair of buckets $(i,j)$ and group $G$, and consider a similar quantity that serves as a proxy for the portion of the moment calibration error corresponding to each group $G \in \cG$ and buckets $i \in [n]$, $j\in [n']$. We will need an extra piece of notation: for any $i \in [n]$, define $\centerbucket \equiv \frac{2i-1}{2n}$. For any $i \in [n]$ and $\bmu \in \Bm(i)$,  we abuse notation and write $\centerbucketmu = \centerbucket$.  

\begin{definition}\label{def:meanmomentlosses}
Given a transcript $\pi_s = ((x_t, (\bmu_t, \bmk_t), y_t))_{t=1}^s$,  for each group $G \in \cG$ and buckets $i \in [n], j \in [n']$ at time $s$, we write 
\begin{align*}
    &V^{G, i, j}_s(\pi_s) = \sum_{t=1}^s \ind[\bmu_t \in \Bm(i), \bmk_t \in \Bm(j), x_t\in G] \cdotk \left(y_t - \bmu_t \right), \\
    &M^{G, i, j}_s(\pi_s) = \sum_{t=1}^s \ind[\bmu_t \in \Bm(i), \bmk_t \in \Bm(j), x_t\in G] \cdotk \left(\left(y_t - \centerbucket\right)^k - \bmk_t \right).
\end{align*}
When the transcript $\pi_s$ is clear from context we will simply write $V^{G, i, j}_s, M^{G, i,j}_s$. 
\end{definition}
In words, $V^{G,i,j}_s$ calculates the difference between the true mean and the mean of our predictions over the subset of periods up to $s$ in which  the realized feature vector was in group $G$ and the learner predicted a mean $\bmu \in \Bm(i)$ and a moment $\bmk \in \Bmp(j)$. $M^{G,i,j}_s$ defines a similar quantity for moments --- but not exactly. Instead of calculating the empirical moment around the empirical mean (i.e. $(y_t - \mu(G_s(i,j)))^k$), we center around $\centerbucket$, i.e.\ the middle of the bucket $\Bm(i)$. We do this to make $M^{G,i,j}_s$ linearly separable across rounds.

We show, using an argument similar\footnote{$(y_t - \centerbucket)^k$ roughly corresponds to what is referred to as a pseudo-moment in \cite{momentmulti}.} to  \cite{momentmulti}, that if our mean predictions are sufficiently calibrated --- which ensures $\centerbucket \approx \mu(G_T(i,j))$ --- then we can still bound the mean-conditioned moment multicalibration error through our proxy quantity $M^{G,i,j}_s$.%

\begin{lemma}
\label{lem:nonlinear}
    For a given $i \in [n], j \in [n']$ and $G \in \cG$, if $\frac{1}{T}|V^{G,i,j}_T| \le \alpha, \frac{1}{T}|M^{G,i,j}_T| \le \beta$, 
    then we have
            \begin{align}
                &\left\vert \mu(G_T(i,j)) - \bmu(G_T(i,j)) \right\vert \le \frac{\alpha T}{|G_T(i,j)|},\tag{Mean Consistency}\\
                &\left\vert m^k(G_T(i,j)) -  \bmk(G_T(i,j)) \right\vert \le \frac{(\beta + k\alpha + \frac{k}{2n})T}{|G_T(i,j)|}. \tag{Moment Consistency}            \end{align}
\end{lemma}
\begin{proof}
It is easy to see mean-consistency:
    \begin{align*}
        &\frac{|G_T(i,j)|}{T}\left\vert \bmu(G_T(i,j))- \mu(G_T(i,j)) \right\vert =  \frac{1}{T} \left\vert \sum_{t \in G_T(i,j)} \left(  \bmu_t - y_t\right)\right\vert = \frac{1}{T} |V^{G,i,j}_T|  \le \alpha.
    \end{align*}
Now, we show that we achieve mean-conditioned moment consistency. First note that
\begin{align*}
    & \frac{1}{T} |M^{G,i,j}_T| =   \frac{1}{T}\left\vert\sum_{t \in G_T(i,j)} \bmk_t -  \left(\centerbucket - y_t\right)^k \right\vert    \le \beta.\\
    \intertext{Now,}
    &\left\vert m^k(G_T(i,j)) - \frac{1}{|G_T(i,j)|}\sum_{t \in G_T(i,j)} \left(y_t - \centerbucket\right)^k \right\vert \\
    =& \left\vert \frac{1}{|G_T(i,j)|}\sum_{t \in G_T(i,j)}   \left(\left(y_t -\centerbucket\right) +\left(\centerbucket- \mu(G_T(i,j))\right) \right)^k - \left(y_t - \centerbucket\right)^k \right\vert, \\
     \le& \frac{k}{|G_T(i,j)|}   \sum_{t \in G_T(i,j)}  \left\vert \centerbucket - \mu(G_T(i,j)) \right\vert, \\
     =& \frac{k}{|G_T(i,j)|}   \sum_{t \in G_T(i,j)}  \left\vert \centerbucket - \bmu(G_T(i,j)) + \bmu(G_T(i,j)) -  \mu(G_T(i,j)) \right\vert, \\
     \leq& \frac{k}{|G_T(i,j)|}   \sum_{t \in G_T(i,j)}  \left\vert \centerbucket - \bmu(G_T(i,j))\right\vert + \left\vert\bmu(G_T(i,j)) -  \mu(G_T(i,j)) \right\vert, \\
     \le& \frac{Tk(\alpha  + \frac{1}{2n})} {|G_T(i,j)|},
\end{align*}
where the first inequality follows from the fact that $|a^k - b^k| \le k |a-b|$ for any $a,b \in [0,1]$ with $a=\left(y_t -\centerbucket\right) +\left(\centerbucket- \mu(G_T(i,j))\right) $ and $b=y_t - \centerbucket$. The last inequality follows from the guarantee of mean consistency as shown above in the proof and the fact that $\bmu(G_T(i,j)) \in \Bm(i)$ and $|\centerbucket - x| \le \frac{1}{2n}$ for any $x \in \Bm(i)$.

Therefore, we can invoke the triangle inequality to conclude
\begin{align*}
    &\left\vert m^k(G_T(i,j)) -  \bmk(G_T(i,j)) \right\vert \\
    \le& \left\vert m^k(G_T(i,j)) -  \frac{1}{|G_T(i,j)|}\sum_{t \in G_T(i,j)} \left(y_t - \centerbucket\right)^k \right\vert + \left\vert \frac{1}{|G_T(i,j)|}\sum_{t \in G_T(i,j)} \left(y_t - \centerbucket\right)^k - \bmk(G_T(i,j)) \right\vert \\
    \le&  \frac{(\beta + k\alpha + \frac{k}{2n})T}{|G_T(i,j)|}. \qedhere
\end{align*}
\end{proof}

This lemma implies that if we can force each term  $V^{G,i,j}_s, M^{G,i,j}_s$ to be small, then we will have achieved our desired goal of mean-conditioned moment multicalibration (Definition~\ref{def:mean-conditioned-moment-multi}). 
\begin{obs}\label{obs:momentcalibration}
Suppose a transcript $\pi_T$ is such that for all $i \in [n],j \in [n']$ and $G \in \cG$, we have that $|V^{G,i,j}_T|, |M^{G,i,j}_T| \leq \alpha T$.  Then the predictions are $(\alpha, \beta,n,n')$-mean-conditioned moment multicalibrated in the sense of Definition~\ref{def:mean-conditioned-moment-multi} for $\beta = (k+1)\alpha + \frac{k}{2n}$. 
\end{obs}
\begin{remark}
Note that with this parametrization, we can take $\alpha$ as small as we like relative to $n$, and by choosing an appropriately large value of $n$, we can take $\beta = (k+1)\alpha + \frac{k}{2n}$ as small as we like relative to $n'$.
\end{remark}

As before, we define a surrogate  loss function at each round $s$. 

\begin{definition}[Surrogate Loss]\label{def:meanmomentloss}
Fixing a transcript $\pi_s \in \Pi^*$ and a parameter $\eta \in [0,\frac12]$, define:
\begin{align*}
    L_s(\pi_s) =& \sum_{\substack{G \in \cG,\\i \in [n],j \in [n']}}\left(\exp(\eta V_s^{G,i,j}) + \exp(-\eta V_s^{G,i,j}) + \exp(\eta M_s^{G,i,j}) + \exp(-\eta M_s^{G,i,j}) \right),
\end{align*}
where $V$ and $M$ are functions of $\pi_s$ as defined in Definition~\ref{def:meanmomentlosses}. When the transcript $\pi_s$ is clear from context we will sometimes simply write $L_s$. 
\end{definition}

As before, our goal is to find a strategy for the learner that guarantees that our surrogate loss $L_T$ remains small. Towards this end, we define $\Delta_{s+1}(\pi_s,x_{s+1},\bmu, \bmk)$ to be the expected increase in the surrogate loss function in the event that the adversary plays feature vector $x_{s+1}$ \emph{and} the learner predicts $(\bmu, \bmk)$. Here the expectation is over the only remaining source of randomness after the conditioning --- the distribution over labels $y_{s+1}$, which for any adversary is defined once we fix $\pi_s$ and $x_{s+1}$. 

\begin{definition}[Conditional Change in Surrogate Loss]
\[
\Delta_{s+1}(\pi_s,x_{s+1},\bmu, \bmk) = \E_{\ty_{s+1}}\left[\tL_{s+1}-L_s \Bigr\vert \pi_s, x_{s+1},  \bmu, \bmk \right].
\] 
\end{definition}

We again show a simple bound on $\Delta_{s+1}(\pi_s,x_{s+1},\bmu, \bmk)$:
\begin{lemma}
\label{lem:boundincrease2}
For any transcript $\pi_s \in \Pi^*$, any $x_{s+1} \in \cX$, and any predictions $\bmu, \bmk \in [0,1]$ such that $\bmu \in \Bm(i)$ and $\bmk \in \Bmp(j)$ for some $i \in [n]$ and $j \in [n']$:
\[
    \Delta_{s+1}(\pi_s,x_{s+1},\bmu, \bmk) \leq \eta \left(\E_{\ty_{s+1}}[\ty_{s+1}]-\bmu\right) C^{\bmu,\bmk}_s(x_{s+1})
    +\eta \left( \E_{\ty}\left(\ty_{s+1} - \centerbucketmu\right)^k - \bmk \right)D^{\bmu,\bmk}_s(x_{s+1})
    + 2\eta^2\cdotk  L_{s},
\]
where
\begin{align}
    C^{\bmu, \bmk}_s(x_{s+1}) &= C^{i,j}_s(x_{s+1}) = \sum_{G \in \cG(x_{s+1})}\exp(\eta V_s^{G,i,j}) -  \exp(-\eta V_s^{G,i,j}),\label{eqn:def-C}\\
    D^{\bmu, \bmk}_s(x_{s+1}) &= D^{i,j}_s(x_{s+1}) = \sum_{G \in \cG(x_{s+1})}\exp(\eta M_s^{G,i,j}) -  \exp(-\eta M_s^{G,i,j}). \label{eqn:def-D}
\end{align}

For economy of notation, we will generally elide the dependence on $x_{s+1}$ for the $C$ and $D$ quantities and simply write $C^{i,j}_s, D^{i,j}_s$ when the feature vector is clear from context.
\end{lemma}
\begin{proof}
To see this, observe that by definition:
\begin{align*}
 &\hphantom{\E_{y_{s+1}}\Bigg[}\Delta_{s+1}(\pi_s,x_{s+1}, \bmu, \bmk) \\
 =& \E_{\ty_{s+1}}\Bigg[\sum_{\cG(x_{s+1})} \underbrace{\exp(\eta V_s^{G,i,j})\left(\exp\left(\eta \left(\ty_{s+1} - \bmu\right)\right)-1\right) + \exp(-\eta V_s^{G,i,j})\left(\exp\left(-\eta \left(\ty_{s+1} - \bmu\right)\right)-1\right)}_{*}\\
 &\hphantom{\E_{y_{s+1}}\Bigg[}+ \underbrace{\exp(\eta M_s^{G,i,j})\exp\left(\eta \left(\left(\ty_{s+1} - \centerbucketmu\right)^k - \bmk\right)-1\right) + \exp(-\eta M_s^{G,i,j})\exp\left(-\eta \left(\left(\ty_{s+1} - \centerbucketmu\right)^k - \bmk \right) - 1 \right)}_{**}\Bigg]. 
\end{align*}

 Using the fact that for $0 < |x| < \frac{1}{2}$, $\exp(x) \leq 1+ x+2x^2$, we have that
 \begin{align*}
     * &\le \exp(\eta V_s^{G,i,j})\left(\eta \left(y_{s+1} - \bmu\right) + 2\eta^2 \right) + \exp(-\eta V_s^{G,i,j})\left(-\eta \left(y_{s+1} - \bmu\right) + 2\eta^2\right),\\
     ** &\le \exp(\eta M_s^{G,i,j})\left(\eta \left(\left(y_{s+1} - \centerbucketmu\right)^k -\bmk \right) + 2\eta^2 \right) + \exp(-\eta M_s^{G,i,j})\left(-\eta \left(\left(\ty_{s+1} - \centerbucketmu\right)^k - \bmk\right) + 2\eta^2 \right).
 \end{align*}
 Now, using the linearity of expectation and distributing the outer expectation to each relevant term where $\ty_{s+1}$ appears, we get
 \begin{align*}
     &\Delta_{s+1}(\pi_s,x_{s+1}, \bmu, \bmk)  \\
     \le& \sum_{\cG(x_{s+1})} \exp(\eta V_s^{G,i,j})\left(\eta \left(\E[\ty_{s+1} ]-\bmu\right) + 2\eta^2 \right) + \exp(-\eta V_s^{G,i,j})\left(-\eta\left(\E[\ty_{s+1}] - \bmu\right) + 2\eta^2\right) \\
     & \hphantom{\sum_{\cG(x_{s+1})}} +\exp(\eta M_s^{G,i,j})\left(\eta \left(\E\left[\left(\ty_{s+1} - \centerbucketmu\right)^k \right] - \bmk\right) + 2\eta^2 \right) + \exp(-\eta M_s^{G,i,j})\left(-\eta \left( \E\left[\left(\ty_{s+1} - \centerbucketmu\right)^k\right] - \bmk\right) + 2\eta^2 \right).
 \end{align*}
 
Collecting terms appropriately and observing that
\begin{align*}
     \sum_{\cG(x_{s+1})}\left(\exp(\eta V_s^{G,i,j}) + \exp(-\eta V_s^{G,i,j}) + \exp(\eta M_s^{G,i,j}) + \exp(-\eta M_s^{G,i,j}) \right) \leq L_s,
\end{align*}
 we have the desired bound. 
\end{proof}

As before, we proceed by defining a zero-sum game between the learner and the adversary and using the minimax theorem to conclude that the learner always has a strategy that guarantees a bounded per-round increase in surrogate loss. To satisfy the convexity and compactness requirements of the minimax theorem, we will again consider a game where the learner's pure strategy space is a finite subset of $\cPmm$. To this end, we define the following grids for any $r \in \mathbb{N}$ ($n$ and $n'$ are the coarseness parameters of our bucketings from above): 
\newcommand{\cPmomentr}{\cP_{\text{moment}}^r}
\begin{align*}
&\cP^{rn} = \left\{0, \frac{1}{rn}, \frac{2}{rn}, \dots, 1\right\},\\
&\cP^{rn'} = \left\{0, \frac{1}{rn'}, \frac{2}{rn'}, \dots, 1\right\}.
\end{align*}

As in the previous section, the need to discretize is only for technical reasons, and our algorithm has no dependence --- neither in runtime nor in its convergence rate --- on the value of $r$ that we choose, so we can imagine the discretization to be arbitrarily fine.

\begin{lemma}
\label{lem:exists2}
For any transcript $\pi_s \in \Pi^*$ and any $x_{s+1} \in \cX$, there exists a distribution over predictions for the learner $Q_{s+1}^L \in \Delta (\cP^{rn} \times \cP^{rn'})$, such that regardless of the adversary's choice of distribution of $y_{s+1}$ over $\Delta \cY$,  we have that:
\[
    \E_{(\bmu, \bmk) \sim Q^L_{s+1}}\left[\Delta_{s+1}(\pi_s,x_{s+1}, \bmu, \bmk)\right] \le L_s\cdotk \left(\frac{\eta}{rn} + \frac{\eta}{rn'} +2\eta^2\right).
\]
\end{lemma}
\begin{proof}

Fix the transcript $\pi_s$ and the feature vector $x_{s+1}$. As before, we define a zero-sum game played between the learner (the minimization player) and the adversary (the maximization player), where the objective function of the game equals the upper bound on $\Delta_{s+1}(\pi_s,x_{s+1}, \bmu, \bmk)$ from Lemma~\ref{lem:boundincrease2}. Then, we again show that for every strategy of the adversary (i.e. distribution over $y$), there  exists a best response for the learner that guarantees the objective function of the game is small. Finally, we appeal to the minimax theorem to conclude that there always exists a strategy for the learner that guarantees small objective value against any strategy of the adversary.

More precisely, consider the following objective function for the game:
\begin{align*}
    u((\bmu, \bmk), y) &=\eta \left(y- \bmu\right) C^{\bmu,\bmk}_s
    +\eta \left( \left(y-\centerbucketmu\right)^k - \bmk\right)D^{\bmu,\bmk}_s + 2\eta^2\cdotk  L_{s}\\
    &= \eta \left(y- \bmu\right) C^{\bmu,\bmk}_s
    +\eta \left(\left(\sum_{\ell=0}^k {k \choose \ell}  (-\centerbucketmu)^{k-\ell} y^{\ell}\right)-\bmk \right)D^{\bmu,\bmk}_s + 2\eta^2\cdotk  L_{s}
\end{align*}
where the pure strategy space for the learner is $X_1 = \cP^{rn} \times \cP^{rn'}$ and that of the adversary is (a priori) the set of all distributions over $[0,1]$. However, we observe that the expected value of the objective for any label distribution  over $[0,1]$ is linear in  $\E[y], \ldots, \E[y^k]$. So the payoff for any mixed strategy of the adversary is determined only by the associated $k$ terms: $\E[y], \ldots, \E[y^k]$.  

With this observation in mind, we perform a change of variables and define a new game with an enlarged strategy space for the adversary.  In the new game, the strategy space for the learner remains $\mathcal{Q}^L = \Delta (\cP^{rn} \times \cP^{rn'})$. The strategy space for the adversary becomes $\mathcal{Q}^A = [0,1]^k$, representing a choice for each of the values $\E[y],\ldots\E[y^k]$. Note that this strategy space for the adversary is unencumbered by the requirement that these chosen values actually correspond to any feasible label distribution over $[0,1]$. The objective function of the game is obtained by replacing each term $\E[y^\ell]$ from our previous objective function with $\psi_\ell$:
\begin{align*}
u((\bmu, \bmk), \psi) = \hphantom{+}\eta \left(\psi_1 - \bmu\right) C^{\bmu,\bmk}_s
    +\eta \left(\left(\centerbucketmu^k + \sum_{\ell=1}^k {k \choose \ell}(-\centerbucketmu)^{k-\ell} \psi_\ell\right) - \bmk\right)D^{\bmu,\bmk}_s + 2\eta^2\cdotk  L_{s}.
\end{align*}
As we have noted, in the original game, the set of achievable moments $\E[y], \dots, \E[y^k]$ is a strict subset of $[0,1]^k$. However, enlarging the strategy space of the maximization player can only increase the ($\max\min$) value of the game, so the upper bound we are about to prove on the game value against this more powerful adversary also applies to the adversary who is implicitly choosing  moments $\E[y], \dots, \E[y^k]$ via some distribution over $[0,1]$. 

Note that $u$ thus defined is linear in both players' strategies, and the strategy spaces for both players $\mathcal{Q}^L$ and $\mathcal{Q}^A$ are compact and convex. Hence, Sion's minimax theorem (Theorem~\ref{thm:minimax}) applies to this game. We now establish (a bound on) the value of this game. Observe that for any strategy of the adversary, the learner can pick $\bmu \!\in\! \cP^{rn}$ as close as possible to $\psi_1$, and then pick $\bmk \!\in\! \cP^{rn'}$ as close as possible to $\centerbucketmu^k + \sum_{\ell=1}^k {k \choose \ell}(\!-\centerbucketmu)^{k-\ell} \psi_\ell$. Therefore, since $C_s^{\bmu,\bmk}, D_s^{\bmu, \bmk} \leq L_s$ by definition, we have that: 
\begin{align*}
    \forall \psi \in [0,1]^k, \exists (\bmu, \bmk) \in (\cP^{rn} \times \cP^{rn'}) \text{   s.t. }  u((\bmu, \bmk), \psi) \leq L_s \left( \frac{\eta}{rn} + \frac{\eta}{rn'} + 2\eta^2 \right). 
\end{align*}
We can now apply the minimax theorem (Theorem~\ref{thm:minimax}) to conclude that there exists a fixed distribution $Q^L_{s+1} \in\mathcal{Q}^L$ for the learner that guarantees objective value that is at most  the above bound for every choice of the adversary, i.e. 
\begin{align*}
    \exists Q^L_{s+1}\in \mathcal{Q}^L  \text{  s.t. }  \forall \psi \in [0,1]^k:\;  u(Q^L_{s+1},\psi) \leq L_s \left( \frac{\eta}{rn} + \frac{\eta}{rn'} + 2\eta^2\right),
\end{align*}
as desired.
\end{proof}

\begin{corollary}
\label{cor:exists-moment}
For every $s \in [T]$, $\pi_s \in \Pi^*$, $x_{s+1} \in \cX$ (which fixes $L_s$ and $Q^L_{s+1}$), and every adversary (which fixes a distribution over $\cY$):
\[
    \E_{Q^L_{s+1}}[\tL_{s+1}|\pi_s] = L_s + \E_{Q^L_{s+1}}[\Delta_{s+1} (\pi_s, x_{s+1}, \bmu, \bmk) | \pi_s] \leq L_s\cdotk\left(1 + \frac{\eta}{rn} +\frac{\eta}{rn'}+  2\eta^2\right).
\]  
\end{corollary}

Lemma~\ref{lem:exists2} defines (existentially) an algorithm that the learner can use to make predictions---Algorithm~\ref{alg:meanexistential2}. We will now show that Algorithm~\ref{alg:meanexistential2} (if we could compute the distributions $Q^L_t$) results in mean-conditioned moment multicalibrated predictions. In Section~\ref{sec:momentalg} we show how to compute $Q^L_t$.

\begin{algorithm}[H]
\SetAlgoLined
\begin{algorithmic}
\FOR{$t=1, \dots, T$}
	\STATE Observe $x_t$. Given $\pi_{t-1}$ and $x_t$, let $Q^L_t \in \Delta (\cP^{rn} \times \cP^{rn'})$ be the distribution over predictions whose existence is established in Lemma~\ref{lem:exists2}.
	\STATE Sample $\bmu, \bmk \sim Q^L_t$ and predict $(\bmu_t,\bmk_t) = (\bmu,\bmk)$. 
\ENDFOR
\end{algorithmic}
\caption{A Generic Mean Moment Multicalibrator}
\label{alg:meanexistential2}
\end{algorithm}

\bigskip 

We are now ready to bound our multicalibration error. The results that follow mirror the structure of Section~\ref{sec:meanexistential}: essentially, we apply Theorem~\ref{thm:general-bounds} to the surrogate loss function of this section.  As a straightforward consequence of Corollary~\ref{cor:exists-moment} and the first part of Theorem~\ref{thm:general-bounds}, we have the following result. 
\begin{corollary}
\label{cor:surrogateloss-mean-moment}
Against any adversary, Algorithm~\ref{alg:meanexistential2} instantiated with discretization parameter $r$ results in surrogate loss satisfying:
\[
    \E_{\tpi_T}[\tL_T] \le 4 |\cG|n \cdot n'  \cdot \exp\left(\frac{T\eta}{rn} + \frac{T\eta}{rn'} +2T\eta^2\right).
\]
\end{corollary}
\begin{proof}
Note that the first part of Theorem~\ref{thm:general-bounds} applies in this case to the process $L$, with $L_0 = 4 |G|n \cdot n'$ and $c = \frac{1}{rn} + \frac{1}{rn'}$. The bound follows by plugging these values into~\eqref{eqn:expectedbound}. 
\end{proof}

Next, we can convert this into a bound on Algorithm~\ref{alg:meanexistential}'s expected calibration error, using Theorem~\ref{thm:general-bounds}. The proof mirrors the argument in Section~\ref{sec:onlinemeanmulti} and can be found in the Appendix. 
\newcommand{\tM}{\tilde{M}}
\begin{restatable}{theorem}{momentmulti}
\label{thm:momentmulti}
When Algorithm~\ref{alg:meanexistential2} is run using bucketing coarseness parameters $n$ and $n'$, discretization parameter $r \in \mathbb{N}$, and $\eta = \sqrt{\frac{\ln(4|\cG|n \cdot n')}{2T}} \in (0, 1/2)$, then against any adversary, its sequence of mean-moment predictions is $(\alpha,\beta,n,n')$-mean-conditioned moment multicalibrated with respect to $\cG$, where $\beta = (k+1)\alpha + \frac{k}{2n}$ and:
$$\E[\alpha] \leq \frac{1}{r n} + \frac{1}{rn'} + 2\cdotk \sqrt{\frac{2\ln(4|\cG|n\cdot n')}{T}}.$$
For $r = \frac{\sqrt{T}(n+n')}{\varepsilon n \cdot n'\cdot \sqrt{2\ln(4|\cG|n \cdot n'})}$, this gives:
$$\E[\alpha] \leq \left( 2+\varepsilon\right) \cdotk \sqrt{\frac{2}{T} \ln\left(4|\cG|n\cdot n' \right)}.$$
Here the expectation is taken over the randomness of the transcript $\pi_T$.
\end{restatable}

We can similarly use the second part of Theorem~\ref{thm:general-bounds} to prove a high probability bound on the multicalibration error of Algorithm~\ref{alg:meanexistential2}. The proof is in the Appendix.
\begin{restatable}{theorem}{momenthp}\label{thm:hpcalibration2}
When Algorithm~\ref{alg:meanexistential2} is run using bucketing coarseness parameters $n$ and $n'$, discretization $r \in \mathbb{N}$ and $\eta = \sqrt{\frac{\ln(4|\cG|n \cdot n')}{2T}} \in (0, 1/2)$, then against any adversary, with probability $1-\lambda$ over the randomness of the transcript, its sequence of predictions is $(\alpha,\beta,n,n')$-mean-conditioned moment multicalibrated with respect to $\cG$ for $\beta = (k+1)\alpha + \frac{k}{2n}$ and:
\[ 
    \alpha \leq \frac{1}{rn} + \frac{1}{rn'} +  4 \cdotk \sqrt{\frac{2}{T} \ln\left(\frac{4|\cG|n \cdot n'}{\lambda} \right)}.
\]
For $r = \frac{\sqrt{T} (n+n')}{\epsilon n \cdot n' \sqrt{2\ln(4|\cG|n \cdot n'/\lambda})}$, this gives:
\[\alpha \leq \left(4 +\epsilon\right) \cdotk \sqrt{\frac{2}{T} \ln\left(\frac{4|\cG|n \cdot n'}{\lambda} \right)}.\]
\end{restatable}

\subsection{Deriving an Efficient Algorithm via Equilibrium Computation}
\label{sec:momentalg}
Previously, we derived Algorithm~\ref{alg:meanexistential2} and proved that it results in mean-conditioned moment multicalibrated predictions. But Algorithm~\ref{alg:meanexistential2} is not explicitly defined, as it relies on the distributions $Q^L_t$ whose existence we showed in Lemma~\ref{lem:exists2} but which we did not explicitly construct. In this section, we show how to efficiently solve for this distribution $Q^L_t$ using a linear program with $4n\cdot n'$ variables and $2^k+1$ constraints. If $k$ is a small constant (e.g. $k = 2$ for variance multicalibration), then this linear program can be explicitly described and solved. But even when $k$ is too large to enumerate all $2^k$ constraints, we show that there is a separation oracle that runs in time $O(k)$, allowing us to efficiently solve this linear program (i.e. in time polynomial in $n, n',T,|\cG|$,  and $k$) using the Ellipsoid algorithm.

Recall that in our simplified game, the learner has pure strategies $(\bmu, \bmk) \in  \cP^{rn} \times \cP^{rn'}$, and the adversary has strategy space $\mathcal{Q}^A = [0,1]^k$. Since the objective function is linear in the adversary's action $\psi$, we can view this as the set of mixed strategies over the $2^k$ pure strategies $\psi \in \{0,1\}^k$. We recall the objective function: 
\begin{align}
&u((\bmu, \bmk), \psi) = \hphantom{+}\eta \left(\psi_1 - \bmu\right) C^{\bmu,\bmk}_s
    +\eta \left(\left(\centerbucketmu^k + \sum_{\ell=1}^k {k \choose \ell}(-\centerbucketmu)^{k-\ell} \psi_\ell\right) - \bmk\right)D^{\bmu,\bmk}_s + 2\eta^2\cdotk  L_{s}.\nonumber
\intertext{Since the equilibrium structure stays the same under positive affine transformations of the objective function, for the purposes of computing equilibria, we may redefine the objective function to be:}
&u((\bmu, \bmk), \psi) = \hphantom{+} \left(\psi_1 - \bmu\right) C^{\bmu,\bmk}_s
    + \left(\left(\centerbucketmu^k + \sum_{\ell=1}^k {k \choose \ell}(-\centerbucketmu)^{k-\ell} \psi_\ell\right) - \bmk\right)D^{\bmu,\bmk}_s .\label{eqn:objfn-moment}
\end{align}

The specific values of $C^{\bmu,\bmk}_s$, $\centerbucketmu$ and $D^{\bmu,\bmk}_s$ do not matter for the analysis that follows---but what is relevant is that by definition, they are constant for any two $(\bmu, \bmk)$ and $(\bmu', \bmkp)$ both in the same bucket --- in other words, if  $\exists i \in [n], j \in [n']$ such that $(\bmu, \bmk),(\bmu', \bmkp) \in B_{n,n'}(i,j)$.
We wish to find a minimax strategy for the learner in this game, i.e. to find a solution to
$$\argmin_{Q^L \in \mathcal{Q}^L} \max_{Q^A \in \mathcal{Q}^A} u(Q^L, Q^A).$$

A priori, the learner has $r^2 n'n$ pure strategies (i.e. $|\cP^{rn} \times \cP^{rn'}| = r^2 n'n$), and a minimax strategy could potentially be supported over all of them (causing our algorithm to  have running time depending on $r$). However, we  prove that we can without loss of generality reduce the size of the learner's pure strategy space  to $4n'n$ (Lemma~\ref{lem:reducingstrats-moment}), which will eliminate any running time dependence on $r$ and allow us to choose as fine a discretization as we like. We also show in Appendix~\ref{sec:momentellipsoid} that the learner always has a minimax strategy that randomizes over a support of at most $k+1$ actions. Thus, as with mean multicalibration, we need only make limited use of randomness (at least for $k$ small).

\newcommand{\cPmeanreduced}{\hat{\cP}}
\newcommand{\cPmomentreduced}{\hat{\cP}}
\newcommand{\learnerstrat}{\hat{\mathcal{Q}}^L_{r,n,n'}}

We first reduce the space of ``relevant'' pure strategies for the learner --- intuitively, points that are at---or just barely below---the boundary of a bucket:
\begin{align*}
    \cPmeanreduced^{r,n} &= \bigcup_{i \in [n-1]} \left\{\frac{i-1}{n}, \frac{i}{n}- \frac{1}{rn}\right\} \bigcup \left\{\frac{n-1}{n}, 1\right\} \subset \cP^{rn},\\
    \cPmomentreduced^{r,n'} &= \bigcup_{i \in [n'-1]}\left\{\frac{i-1}{n'}, \frac{i}{n'}- \frac{1}{rn'}\right\} \bigcup \left\{ \frac{n'-1}{n'},1\right\} \subset \cP^{rn'}.
\end{align*}
Given these sets, define $\learnerstrat \equiv \Delta \left( \cPmeanreduced^{r,n} \times \cPmeanreduced^{r,n'}\right)\subset \mathcal{Q}^L$.

\begin{lemma}\label{lem:reducingstrats-moment}
In the game with objective function $u$ as defined in \eqref{eqn:objfn-moment}, the value of the game is unaffected if the learner is restricted to mixed strategies in $\learnerstrat$, a set of distributions which in particular have support over at most $4nn'$ actions. In other words:
\[
 \min_{Q^L \in \mathcal{Q}^L} \max_{Q^A \in \mathcal{Q}^A} u(Q^L, Q^A) = \min_{\hat{Q}^L \in \learnerstrat} \max_{Q^A \in \mathcal{Q}^A} u(\hat{Q}^L, Q^A).  
\]
\end{lemma}
\begin{proof}
Fix any strategy $Q^L \in \mathcal{Q}^L$. Since $\learnerstrat \subseteq \mathcal{Q}^L$, it is sufficient to show that there exists a strategy $\hat{Q}^L \in \learnerstrat$ such that:
\begin{align*}
    \max_{Q^A \in \mathcal{Q}^A} u(Q^L, Q^A) \geq \max_{Q^A \in \mathcal{Q}^A} u(\hat{Q}^L, Q^A). 
\end{align*}
To see this, first observe that we can regroup terms in the objective function \eqref{eqn:objfn-moment} and write it as:
\begin{align}
u((\bmu, \bmk), \psi) &= -\bmu C^{\bmu,\bmk}_s  + \centerbucketmu^k D^{\bmu,\bmk}_s - \bmk D^{\bmu,\bmk}_s  + \sum_{\ell=1}^{k} \psi_\ell F_\ell^{\bmu,\bmk} \label{eqn:obj-moment-rewrite}\\
\text{where }
F_1^{\bmu,\bmk} &= C^{\bmu,\bmk}_s - k \centerbucketmu^{k-1} C^{\bmu,\bmk}_s,\label{eq:def-F1}\\
\forall \ell >1, \ell \in [n]: \; F_\ell^{\bmu,\bmk} &= {k \choose \ell}(-\centerbucketmu)^{k-\ell} D^{\bmu,\bmk}_s.\label{eq:def-Fk}
\end{align}
Further, by definition for any $\bmu, \bmu' \in \Bm(i)$ for some $i \in [n]$ and $\bmk, {\bmk}' \in \Bmp(j)$, we have, for $X = C,D$, 
\begin{align*}
&X^{\bmu,\bmk}_s = X^{\bmu',{\bmk}'}_s = X^{i,j}_s,\\
&\centerbucketmu = \hat{\mu}_{\bmu'},
\end{align*}
and therefore this equality holds for $X = F$ as well. Against a given strategy $Q^L$ for the learner, the adversary' payoff from pure strategy $\psi$ is:
\begin{align*}
   u(Q^L, \psi) =& \sum_{(\bmu, \bmk)} Q^L(\bmu, \bmk) \left( -\bmu C^{\bmu,\bmk}_s  + \centerbucketmu^k D^{\bmu,\bmk}_s - \bmk D^{\bmu,\bmk}_s  + \sum_{\ell=1}^{k} \psi_\ell F_\ell^{\bmu,\bmk} \right), 
\intertext{which, given the previous fact about $F$, can be rewritten as}
 u(Q^L, \psi) =& \underbrace{\sum_{(\bmu, \bmk)} Q^L(\bmu, \bmk) \left( -\bmu C^{\bmu,\bmk}_s  + \centerbucketmu^k D^{\bmu,\bmk}_s - \bmk D^{\bmu,\bmk}_s \right)}_{(*)}\\
 &+ \underbrace{\sum_{\ell=1}^{k} \psi_\ell \sum_{\substack{i \in [n],\\ j \in [n']}}F_\ell^{i,j}\left( \sum_{(\bmu, \bmk) \in B(i,j)} Q^L(\bmu,\bmk) \right)}_{(**)}.
\end{align*}
Observe that term $(*)$ is independent of $\psi$. Therefore, fixing a $Q^L$, it is equivalent for the adversary to maximize $(**)$. By observation, for any mixed strategy of the learner $Q^L$, the adversary's incentives are only affected through the induced distribution over buckets.

So, given $Q^L$, the best response of the adversary is preserved for any other strategy $\hat{Q}^L$ that maintains the same mass on each bucket, i.e. for all $i \in [n]$ and $j \in [n']$, $\sum_{(\bmu, \bmk)\in B(i,j)}\left(Q^L(\bmu, \bmk) - \hat{Q}^L(\bmu, \bmk)\right)=0.$ Consider the learner's problem of minimizing the objective value among strategies of this form, i.e. preserving the mass on each bucket. This reduces to solving, for each $i \in [n], j\in [n']$, the optimization problem
\begin{align*}
    \min_{\hat{Q}^L \geq 0}&\sum_{(\bmu, \bmk)\in B(i,j)} \hat{Q}^L(\bmu, \bmk) \left( -\bmu C^{i,j}_s  + \centerbucket^k D^{i,j}_s - \bmk D^{i,j}_s \right)\\
    \text{s.t. }&\sum_{(\bmu, \bmk)\in B(i,j)}\left(Q^L(\bmu, \bmk) - \hat{Q}^L(\bmu, \bmk)\right)=0.
\end{align*}
Within a bucket, the coefficients $\left( -\bmu C^{i,j}_s  + \centerbucket^k D^{i,j}_s - \bmk D^{i,j}_s \right)$ are linear in $\bmu, \bmk$ and therefore there must exist a solution that puts all mass $\sum_{(\bmu, \bmk)\in B(i,j)}Q^L(\bmu, \bmk)$ on an extreme point of the bucket. For example, if $i \in [n-1]$, $j \in [n'-1]$; all mass can be placed without loss of generality on one of the four points in $\left\{\frac{i-1}{n}, \frac{i}{n} - \frac{1}{rn}\right\} \times \left\{\frac{j-1}{n'}, \frac{j}{n} - \frac{1}{rn'}\right\}$. If $i = n$,   the corresponding set is $\{\frac{n-1}{n},1\}$, and if $j = n'$, the corresponding set is  $\{\frac{n'-1}{n'},1\}$.  Moving all the mass in each bucket to the optimal corner point, we have that for any strategy $Q^L$ of the learner, there exists $\hat{Q}^L \in \learnerstrat$ such that 
$    \max_{Q^A \in \mathcal{Q}^A} u(Q^L, Q^A) \geq \max_{Q^A \in \mathcal{Q}^A} u(\hat{Q}^L, Q^A), $ as desired. This concludes the proof. 
\end{proof}

The result is that to compute the equilibrium strategy for the learner, it suffices to solve:
$$\argmin_{Q^L \in \learnerstrat} \max_{\psi \in \{0,1\}^k} u(Q^L, \psi).$$

We can directly express this as a linear program with $4nn'$ variables and $2^k+1$ constraints --- see Linear Program~\ref{lp:moments}.
\begin{figure}[H]
\begin{align*}
    &\min_{Q^L\in \learnerstrat} \,\,  \gamma  \text{ s.t.}\\
\forall \psi \in \{0,1\}^k:   & u(Q^L,\psi)  \leq \gamma, \\ 
    &\sideset{}{_{(\bmu,\bmk) \in \cPmeanreduced^{r,n}\times \cPmeanreduced^{r,n'}}}\sum  Q^L((\bmu,\bmk)) = 1, \\
 \forall \,  (\bmu,\bmk) \in \cPmeanreduced^{r,n}\times \cPmeanreduced^{r,n'}:    & \,\, Q^L((\bmu,\bmk)) \geq 0 .
\end{align*}
\caption{A Linear Program for Computing a Minimax Equilibrium Strategy for the Learner at Round $t$. }
\label{lp:moments}
\end{figure}

This is a linear program in $4nn'+1$  variables, with $2^k+1$ constraints. If $k$ is a constant, this is a polynomially sized linear program that can be solved explicitly. If $k$ is superconstant, we will see that we can still solve the linear program with the Ellipsoid algorithm, because we can efficiently find violated constraints. 

\begin{algorithm}[H]
\SetAlgoLined
\begin{algorithmic}
\STATE \textbf{INPUT}: $\epsilon > 0$.
\FOR{$t=1, \dots, T$}
	\STATE Observe $x_t$ and compute  $C_{t-1}^{\bmu, \bmk}(x_t), D_{t-1}^{\bmu,\bmk}(x_t), (F_{\ell, t-1}^{\bmu,\bmk}(x_t))_{\ell=1}^{n}$ for each $(\bmu,\bmk) \in \cPmeanreduced^{r,n} \times \cPmomentreduced^{r,n'}$ as in Equations~(\ref{eqn:def-C}, \ref{eqn:def-D}, \ref{eq:def-F1}, \ref{eq:def-Fk}).
    \STATE Find an $\epsilon$-approximate solution to the linear program from Figure~\ref{lp:moments}, to obtain solution $Q_t^L \in \learnerstrat$. 
   \STATE Predict $(\bmu_t,\bmk_t) = (\bmu, \bmk)$ with probability $Q^L_t((\bmu,\bmk))$.
\ENDFOR
\end{algorithmic}
\caption{Von Neumann's Mean Moment Multicalibrator}
\label{alg:momentalg}
\end{algorithm}

We thus obtain the following theorem:

\begin{theorem} \label{thm:momentalg-efficient}
Algorithm~\ref{alg:momentalg} implements Algorithm~\ref{alg:meanexistential2}. In particular, it obtains multivalidity guarantees arbitrarily close to those of Theorems \ref{thm:momentmulti} and \ref{thm:hpcalibration2}. Namely, for any desired $\epsilon > 0$, we have the following.

Choosing $\eta = \sqrt{\frac{\ln(4|\cG|n \cdot n' +\epsilon)}{2T}} \in (0, 1/2)$, against any adversary, over the randomness of the transcript, the sequence of mean-moment predictions produced by Algorithm~\ref{alg:momentalg} is $(\alpha,\beta,n,n')$-mean-conditioned moment multicalibrated with respect to $\cG$ where $\beta = (k+1)\alpha + \frac{k}{2n}$ and:
\[\E [\alpha] \leq \frac{1}{rn} +\frac{1}{rn'} + 2\sqrt{\frac{2\ln(4|\cG|n \cdot n' +\epsilon)}{T}}.\]
For $r = \frac{\sqrt{T}(n+n')}{\epsilon' n \cdot n'\cdot \sqrt{2\ln(4|\cG|n \cdot n' + \epsilon)}}$, this gives:
$$\E[\alpha] \leq \left( 2+\epsilon' \right) \cdotk \sqrt{\frac{2}{T} \ln\left(4|\cG|n\cdot n' + \epsilon \right)}.$$

Moreover, choosing $\eta = \sqrt{\frac{\ln(4|\cG|n \cdot n') + \epsilon T}{2T}} \in (0, 1/2)$, with probability $1-\lambda$ over the randomness of the transcript $\pi_T$ we have
\[\alpha \leq \frac{1}{rn} +\frac{1}{rn'} + 4 \cdotk \sqrt{\frac{2}{T} \ln\left(\frac{4|\cG|n \cdot n'}{\lambda} \right) + 2\epsilon}.\]
For $r = \frac{(n+n')}{\epsilon' n \cdot n' \sqrt{\frac{2}{T}\ln(4|\cG|n \cdot n'/\lambda) + 2\epsilon}}$, this gives:
\[\alpha \leq \left(4 +\epsilon' \right) \cdotk \sqrt{\frac{2}{T} \ln\left(\frac{4|\cG|n \cdot n'}{\lambda} \right) + 2\epsilon}.\]

The runtime of Algorithm~\ref{alg:intervalalg} scales as $O(|\cG|)$ with the total number of groups $|\cG|$, and is polynomial in $n, n', T, k$, and $\log(\frac{1}{\epsilon})$ (and is independent of $r$).
\end{theorem}
\begin{remark}
As before, if $|\cG(x_t)|$ is efficiently enumerable, then the running time dependence on $|\cG|$ can be replaced with a dependence on $|\cG(x_t)|$. 
\end{remark}
\begin{proof}
First consider the running time of the algorithm. The quantities $C_{t-1}^{\bmu, \bmk}(x_t), D_{t-1}^{\bmu,\bmk}(x_t),F_{\ell, t-1}^{\bmu,\bmk}(x_t)$ are simple sums, which can be computed in time linear in $|\cG|$ (or $|\cG(x_t)|$ if it is efficiently enumerable) and $T$. The linear program has $4nn'+1$ variables, and $2^k+1$ constraints. If $k$ is a constant, this is polynomially sized. Now consider the case in which $k$ is large. In this case we will solve the linear program by applying the Ellipsoid algorithm to its ``rational" modification (see below). The runtime of this approach is polynomial under several well-known conditions, which are given in the following theorem:
\begin{theorem}[\cite{schrijver98}, Corollary 14.1a] \label{thm:ellipsoid}
For an optimization program of a linear objective with rational coefficients over a rational polyhedron $P$ in $\mathbb{R}^q$ for which we are given a separation oracle, the Ellipsoid algorithm solves it exactly in time polynomial in the following parameters: the number of variables $q$, the largest bit complexity $\phi$ of any linear inequality defining $P$, the bit complexity $c$ of the objective function, and the runtime of a separation oracle.
\end{theorem}

Linear Program~\ref{lp:moments} has finitely many constraints so its feasible region is a polyhedron. However, exponential terms in the coefficients of the constraints associated with the adversarial best-responses (which are due to our definition of the soft-max surrogate loss) prevent it from being \emph{rational}. To fix this, we only keep $O(\log \frac{1}{\epsilon})$ bits of precision after the integer part of every coefficient of  LP \ref{lp:moments}, resulting in a new LP whose coefficients are all rational and within $\pm \frac{\epsilon}{2}$ from their original values in LP~\ref{lp:moments}. The new LP indeed has a rational polyhedron as its feasible region.
We now pause to see that solving the rational LP achieves value within $\epsilon$ of the desired optimum of LP~\ref{lp:moments}. This is shown more generally in the following technical lemma, which we will reuse in Section~\ref{sec:intervalalg}; its proof is deferred to the Appendix.

\begin{restatable}{lemma}{lpapprox}
\label{lem:lpapprox}
Consider a linear program of the following form, with variables $x \in \mathbb{R}^m$, $\gamma \in \mathbb{R}$ for some $m$:
\[
    \text{Minimize } \gamma,
    \quad \text{subject to: } \quad Ax \leq \gamma \textbf{1}^m, \,
    x \cdot \textbf{1}^m = 1, \,
    x \geq 0.
\]
Here, $\textbf{1}^m \in \mathbb{R}^m$ is the all-ones vector, and $A = (a_{ji})$ is a finite matrix with real entries. 

Take any $\epsilon > 0$. Modify the above linear program by replacing matrix $A$ with matrix $\tilde{A} = (\tilde{a}_{ji})$, where each $\tilde{a}_{ji}$ is a rational number within $\pm \frac{\epsilon}{2}$ from $a_{ji}$, obtained by truncating $a_{ji}$ to $O(\log \frac{1}{\epsilon})$ bits of precision. Then, any optimal solution $(x^{*, r}, \gamma^{*, r})$ of the resulting rational linear program is an $\epsilon$-approximately optimal feasible solution of the original linear program.
\end{restatable}

Linear Program~\ref{lp:moments} is of the type given in Lemma~\ref{lem:lpapprox}, so we  have that solving the rational LP gives the desired $\epsilon$-approximation to the optimum of Linear Program~\ref{lp:moments}. Now we verify that all linear constraints of the rational version of LP~\ref{lp:moments} have polynomial bit complexity. Recall that the left side of any constraint bounding the objective function can be written as:
\begin{align*}
u(Q^L, \psi) =\!\! \underbrace{\sum_{(\bmu, \bmk)} Q^L(\bmu, \bmk) \left(\!\! -\bmu C^{\bmu,\bmk}_{t-1} \!+\! \centerbucketmu^k D^{\bmu,\bmk}_{t-1} \!\!-\! \bmk D^{\bmu,\bmk}_{t-1} \!\right)}_{(*)}
+ \underbrace{\sum_{\ell=1}^{k} \psi_\ell \sum_{\substack{i \in [n],\\ j \in [n']}}F_\ell^{i,j}\left( \sum_{(\bmu, \bmk) \in B(i,j)} Q^L(\bmu,\bmk) \right)}_{(**)}.
\end{align*} 
There are $4nn' +1$ variables. 
We can bound the coefficient in which any $Q^L(\bmu, \bmk)$ appears in (*) by:
\begin{align*}
\max_{\bmu, \bmk} \sum_G \exp(\eta V_{t-1}^{G, i, j}) \!-\! \exp(-\eta V_{t-1}^{G, i, j}) \!+\! 2\left(\!\exp(\eta M_{t-1}^{G, i, j}) \!-\! \exp(-\eta M_{t-1}^{G, i, j})\!\right)
\leq |\cG| ( 6\exp(\eta 2T))
\leq 6|\cG| \exp(2T).
\end{align*}
The coefficient of any variable $Q^L(\bmu, \bmk)$ in (**) is at most:
\begin{align*}
\sum_{\ell=1}^k \psi_{\ell} \sum_{\substack{i \in [n],\\ j \in [n']}}F_\ell^{i,j} \leq k \cdot (n n') \cdot \max_{i, j} \left\{ 2^k \left( \sum_G 2\exp(\eta M_T^{G, i, j}) \right) \right\}\leq 2^{k+1}k|\cG|n n' \cdot \exp(2T).
\end{align*} 
Recalling that we are also keeping $O(\log \frac{1}{\epsilon})$ bits of precision for each coefficient, it follows that the maximum bit complexity of any constraint is bounded by \[O\left(2 \cdot 4nn' \cdot \left(\log \left(2^{k+1}k|\cG|n n' \cdot \exp(2T) \right) + \log \frac{1}{\epsilon} \right)\right) = \mathrm{poly}\left(n, n', |\cG|, T, k, \log \frac{1}{\epsilon}\right).\]
Of course, the objective value, which is simply $\gamma$, also has polynomial bit complexity.

Next, we describe an efficient separation oracle for the LP. Consider a candidate solution $(Q^L, \gamma)$. 
The constraint requiring that $Q^L$ be a probability distribution can be checked explicitly. Thus, it remains to either find a violated constraint corresponding to some pure strategy $\psi \in \{0,1\}^k$ of the adversary, or to assert that none exists. But this reduces to the problem of finding the most violated such constraint, which corresponds to the adversary's pure best response problem. Note that only the (**) term of the objective function (see the formula above) depends on the adversary's action. Thus, the best response problem of the adversary corresponds to finding
\[\psi^* = \arg\max_{\psi \in \{0,1\}^k}\sum_{\ell=1}^{k} \psi_\ell \sum_{i \in [n],\\ j \in [n']}F_\ell^{i,j} \sum_{(\bmu, \bmk) \in B(i,j)} Q^L(\bmu,\bmk).\]

The best response for the adversary given a fixed distribution $Q^L$ can be computed by setting each coordinate $\ell \in [k]$ independently to be either $0$ or $1$: namely, $\psi_\ell = 1$ if $\sum_{\substack{i \in [n],\\ j \in [n']}}F_\ell^{i,j}\left( \sum_{(\bmu, \bmk) \in B(i,j)} Q^L(\bmu,\bmk) \right) \geq 0$ and $\psi_\ell = 0$ otherwise. This takes $O(k)$ iterations, at each of which the expression whose sign determines $\psi_\ell$ is computed in polynomial time. Once the adversary's best response has been computed, the oracle simply outputs the corresponding constraint if it is violated, and otherwise it asserts that the proposed solutions is feasible. Thus, we have a polynomial-time separation oracle for Linear Program~\ref{lp:moments}.

This completes the proof that Linear Program~\ref{lp:moments} can be solved, at each round, to precision $\epsilon > 0$ in time polynomial in $n, n', \log |\mathcal{G}|, T, k, \log \frac{1}{\epsilon}$. The runtime of Algorithm~\ref{alg:momentalg} is therefore also $\mathrm{poly}(n, n', |\cG|, T, k, \log \frac{1}{\epsilon})$, where the dependence on $|\cG|$ is  $O(|\cG|)$ --- since at the beginning of each round $t$, we precompute the coefficients of the linear program in time linear in $|\cG|$, and the Ellipsoid runs in time polynomial in $\log |\cG|$.

Finally, we need to demonstrate that the claimed multivalidity guarantees (which are a function of the chosen $\epsilon > 0$) indeed hold. If we were exactly solving the linear program, this would be immediate  from Lemma~\ref{lem:reducingstrats-moment} and the fact that Linear Program~\ref{lp:moments} is directly solving for:
$$\argmin_{Q^L \in \learnerstrat} \max_{\psi \in \{0,1\}^k} u(Q^L, \psi).$$ We only need to verify that our approximate guarantees follow from approximately solving the linear program. 
\begin{restatable}{lemma}{momentepsmultivalidity}\label{lem:momentepsmultivalidity}
Algorithm~\ref{alg:momentalg} achieves the multivalidity guarantees specified in Theorem  \ref{thm:momentalg-efficient}. 
\end{restatable}
The proof of this lemma involves repeating several calculations from Section \ref{sec:momentexistential} with an $\epsilon$ error term, and so is deferred to the Appendix.
\end{proof}

\section{Online Multivalid Marginal Coverage}
\label{sec:onlinemultivalid}

\subsection{An Outline of Our Approach}
\label{sec:interval-outline}
In this section, we derive an online algorithm for supplying prediction intervals with a coverage target $1-\delta$ that are multivalid with respect to some collection of groups $\cG$. When $\cG = \{\cX\}$, this corresponds to giving simple marginal prediction intervals --- a similar problem as solved by conformal prediction\footnote{In fact, even with $\cG = \{\cX\}$ the guarantees are  stronger than the marginal guarantees promised by conformal prediction techniques, because they remain valid even conditioning on the prediction. This is important and rules out trivial solutions, like predicting the full  interval with probability $1-\delta$ and an empty interval with probability $\delta$.}, but without requiring distributional assumptions. For richer classes $\cG$, we obtain correspondingly stronger guarantees. We follow the same basic strategy that we developed in Section~\ref{sec:onlinemeanmulti} for making multicalibrated mean predictions, with a couple of important deviations.
\begin{enumerate}
\item First, we observe that even in the distributional setting, it is \emph{not} always possible to provide prediction intervals that have coverage probability exactly $1-\delta$. Consider, for example, the case in which the label distribution is a point mass. Then, any prediction interval will have coverage probability either $0$ or $1$ --- in both cases, bounded away from the target $1-\delta$. More generally, if we are giving prediction intervals with endpoints in some discrete set $\{0,1/rn,\ldots,1\}$, in order for there to exist prediction intervals with approximately the desired coverage probability in the distributional setting, the distribution must not be overly concentrated on any sub-interval of width $1/rn$. We define a sufficient smoothness condition (Definition~\ref{def:smoothness}) for appropriately tight prediction intervals to be guaranteed to exist in the distributional setting --- a condition that becomes increasingly mild as we take our discretization parameter $r$ to be larger.  We then derive  --- existentially, using the minimax theorem --- the existence of an online algorithm that gives prediction intervals that are multivalid at the desired coverage probability when played against an adversary who is constrained at every round to play smooth label distributions. We observe (Remark~\ref{rem:perturb}) that our smoothness condition is very mild, in the sense that we can \emph{enforce it ourselves} by adding noise $U[-\epsilon,\epsilon]$ to the adversary's labels, rather than making assumptions about the adversary. When we do this, the intervals we obtain continue to have valid coverage if we widen both endpoints by $\epsilon$. 
\item To instantiate our algorithm, we again need to compute equilibrium strategies for an appropriately defined game for our learner to sample from. Unlike in the cases of mean and moment multicalibration, however, the equilibrium strategies in this case do not appear to have any nice  structure. We can still derive an efficient algorithm, however, by solving a linear program at each round to compute an equilibrium of the corresponding game. Because we assume that our adversary plays label distributions that are appropriately smooth, the adversary has exponentially many pure strategies in this game, and so we cannot efficiently enumerate all of the constraints in our equilibrium computation program. Instead, we show that a simple greedy algorithm is able to implement a separation oracle, which allows us to solve the linear program efficiently using the Ellipsoid algorithm. 
\end{enumerate}

\subsection{An Existential Derivation of the Algorithm and Multicoverage Bounds}
\label{sec:intervalexistential}
Our goal in this section is to derive an algorithm which at each round, makes predictions $(\bell_t,\bu_t) \in \cPinterval$ that are multivalid with respect to some target coverage probability $1-\delta$.

Towards this end, we define the coverage error of a group $G$ and interval $(\ell, u)$:

\begin{definition}
\label{def:coverageerror}
Given a transcript $\pi_s = (x_t,(\bell_t,\bu_t),y_t)_{t=1}^s$,  we define the coverage  error for a group $G \in \cG$ and bucket $(i,j) \in [n]\times [n]$ at time $s$ to be:
\begin{align*}
&V_s^{G,(i,j)} = \sum_{t=1}^s \ind[x_t \in G, (\bell_t,\bu_t) \in \Bm(i,j)] \cdot  v_\delta((\bell_t, \bu_t), y_t),\\
\text{where }&v_\delta((\ell,u), y) = \mathrm{Cover}((\ell, u), y) - (1-\delta).
\end{align*}
\end{definition}
Just as before, our coverage error serves as a bound on our multicoverage error. 

\begin{obs} 
\label{obs:multicoverage}
Fix a transcript $\pi_T$. 
If for all $G \in \cG$, and buckets $(i,j) \in [n]\times [n]$, we have that:
$$\left|V_T^{G,(i,j)}\right| \leq \alpha T$$
then the corresponding sequence of prediction intervals are $(\alpha,n)$-multivalid with respect to $\cG$. 
\end{obs}

We now pause to observe that even in the easier distributional setting where data are drawn from a fixed distribution: $(x,y) \sim \cD$ --- there may not be any interval $(\ell,u) \in \cPinterval$ that satisfies the desired target coverage value, i.e. that guarantees that $|\E_{(x,y) \sim \cD}[v_\delta((\ell,u),y]|$ is small. Consider for example a label distribution that places all its mass on a single value $y = i \in [0,1]$. Then any interval $(\ell,u)$ covers the label with probability $1$ or probability $0$, which for $\delta \not\in \{0,1\}$ is bounded away from our target coverage probability. Of course, if achieving the target coverage is impossible in the easier distributional setting, then it is also impossible in the more challenging online adversarial setting. With this in mind, we define a class of smooth distributions for which achieving (approximately) the target coverage is always possible for some interval $(\ell,u)$ defined over an appropriately finely discretized range:
$$
\cPinterval^{rn} = \left\{(i,j) \in \cPinterval : i,j \in  \cP^{rn} \right\},
$$
where as before, $\cP^{rn}$ is the uniform grid on $[0,1]$, $\{ 0, \frac{1}{rn}, \ldots, 1\}$. 
We show that we can similarly achieve (approximately) our target coverage goals in the online adversarial setting when the adversary is constrained to playing smooth distributions. 

\begin{definition}
\label{def:smoothness}
A label distribution $Q \in \Delta\cY$ is $(\rho,rn)$-\emph{smooth} if for any $0 \leq a \leq b \leq 1$ such that $|a-b| \leq \frac{1}{rn}$,
$$\Pr_{y \sim Q}[y \in [a,b]] \leq \rho.$$
We say that a joint distribution $\cD \in \Delta (\cX \times \cY)$ is $(\rho,rn)$-\emph{smooth} if for every $x \in \cX$, the marginal label distribution conditional on $x$,  $\cD\vert_{x}$, is $(\rho,rn)$-smooth. 
\end{definition}

\begin{obs}
\label{obs:smooth}
For any $\delta \in [0,1]$ and any fixed ($\rho,rn$)-smooth label distribution $Q$, there always exists some interval $(\bell,\bu) \in \cPinterval^{rn}$  such that $\vert \Pr_{y \sim Q}[\mathrm{Cover}((\ell, u), y)] - (1-\delta)| \le \rho$. 
\end{obs}

\begin{remark}
The assumption of $(\rho,rn)$-smoothness  becomes more mild for any $\rho$ as $r \rightarrow \infty$.  Just as for mean and moment multicalibration, in which our error bounds inevitably depend on the level of discretization $r$ that we choose, here our error bounds will depend on the smoothness level $\rho$ of the adversary's distributions at the discretization level $r$ that we choose. Finally, observe that smoothness is an extremely mild condition in that we can enforce it ourselves if we so choose, rather than assuming that the adversary is constrained. We elaborate on this in Remark~\ref{rem:perturb}.
\end{remark}

\begin{definition}
We write $\mathcal{Q}_{\rho,rn}$ for the set of all $(\rho,rn)$ smooth distributions over $[0,1]$. We write $\hat{\mathcal{Q}}_{\rho,rn}$ for the set of all $(\rho,rn)$-smooth distributions whose support belongs to the grid $\cP^{rn} = \{0,\frac{1}{rn},\ldots,1\}$:
\[\hat{\mathcal{Q}}_{\rho,rn} \equiv  \Delta \cP^{rn}  \cap \mathcal{Q}_{\rho,rn}.\]
\end{definition}

We will show (in Lemma~\ref{lem:discreteok}) that when the learner is restricted to selecting intervals from $\cPinterval^{rn}$, without loss of generality, rather than considering adversaries that play arbitrary distributions over $\mathcal{Q}_{\rho,rn}$, it suffices to consider adversaries that play discrete distributions from  $\hat{\mathcal{Q}}_{\rho,rn}$, which will be more convenient for us. 

To bound the maximum absolute value of our coverage errors across all groups and interval predictions, we again introduce the same style of surrogate loss function:

\begin{definition}[Surrogate loss]
Fixing a transcript $\pi_s \in \Pi^*$ and a parameter $\eta \in (0,1/2)$, define a surrogate coverage loss function at day $s$ as:
\[
L_s(\pi_s) = \sum_{\substack{G \in \cG,\\(i, j) \in [n]\times[n]}}\left(\exp(\eta V_s^{G,(i, j)}) + \exp(-\eta V_s^{G,(i, j)}) \right),
\]
where $V_s^{G,(i, j)}$ are implicitly functions of $\pi_s$. When the transcript is clear from context we will sometimes simply write $L_s$. 
\end{definition}
Once again, $0 < \eta < \frac{1}{2}$ is a parameter that we will set later. 

As before, we proceed by bounding the conditional change in the surrogate loss function:

\begin{definition}[Conditional Change in Surrogate Loss]
Fixing $\pi_s \in \Pi^*$, $x_{s+1} \in \cX$ and an interval $(\ell,u) \in \cPinterval^{rn}$, define the conditional change in surrogate loss to be:  
\[\Delta_{s+1}(\pi_s,x_{s+1},(\bell_{t+1},\bu_{t+1})) = \E_{\tilde y_{s+1}}[\tilde L_{s+1}-L_s | x_{s+1}, (\bell_{s+1},\bu_{s+1}), \pi_s].\]
\end{definition}

\begin{restatable}{lemma}{deltaboundinterval}
\label{lem:deltaboundinterval}
For every transcript $\pi_s \in \Pi^*$, every $x_{s+1} \in \cX$, and every $(\bell_{s+1},\bu_{s+1}) \in \Bm(i,j)$ we have that:
\begin{align*}
   &\Delta_{s+1}(\pi_s,x_{s+1},(\bell_{s+1},\bu_{s+1})) \leq  \left(\eta (\E_{\ty_{s+1}}[v_\delta((\bell_{s+1},\bu_{s+1}), \ty_{s+1})])\right)C^{i,j}_s(x_{s+1})  + 2\eta^2 L_s,\\
   \intertext{where for each $i \leq j \in [n]$, we have defined}
   &C^{i,j}_s(x_{s+1}) \equiv \sum_{\cG(x_{s+1})}\exp(\eta V_s^{G,(i,j)}) - \exp(-\eta V_s^{G,(i,j)}).
\end{align*}
When $x_{s+1}$ is clear from context, for notational economy, we will elide it and simply write $C_s^{i,j}$.
\end{restatable}

As in Section~\ref{sec:onlinemoment}, we defer proofs that mirror previous arguments to the Appendix.

Next, we abuse notation and write $V_s^{G, (\ell,u)}$ to denote $V_s^{G,(i,j)}$ for $i,j \in [n] \times [n]$ such that $(\ell,u) \in \Bm(i,j)$. Given $(\ell, u) \in \cPinterval$ such that $(\ell, u) \in B_n(i,j)$, we let $C^{\ell,u}_s \equiv C^{i,j}_s$, with the latter defined in the statement of Lemma~\ref{lem:deltaboundinterval}. That is, fixing $\pi_s$ and $x_{s+1}$, for any $(\ell, u) \in \cPinterval$ such that $(\ell, u) \in B_n(i,j)$,
\begin{equation}\label{eqn:def-C-interval}
    C^{\ell,u}_s(x_{s+1}) \equiv C^{i,j}_s(x_{s+1}) = \sum_{\cG(x_{s+1})}\exp(\eta V_s^{G,(i,j)}) - \exp(-\eta V_s^{G,(i,j)}),
\end{equation}
where in turn the $V$'s are as defined in Definition~\ref{def:coverageerror}.

\begin{lemma}[Value of the Game]
\label{lem:existsinterval}
For any $x_{s+1} \in \cX$, any adversary restricted to playing $(\rho,rn)$-smooth distributions, and any transcript $\pi_s \in \Pi^*$, there exists a distribution over predictions for the learner $Q^L_{s+1} \in \Delta \cPinterval^{rn}$ which guarantees that:
\[
    \E_{(\bell,\bu) \sim Q^L_{s+1}}\left[\Delta_{s+1}(\pi_s,x_{s+1},(\bell_{s+1},\bu_{s+1}))\right] \le L_s\cdotk \left(\eta \rho + 2\eta^2\right).
\]
\end{lemma}

\begin{proof}
We again proceed by defining a zero-sum game with objective function equal to the upper bound on $\Delta_{s+1}(\pi_s,x_{s+1},(\bell_{s+1},\bu_{s+1}))$ that we proved in Lemma~\ref{lem:deltaboundinterval}: 
\[
u((\ell,u), y) = \eta \cdot v_\delta((\ell, u), y) \cdot C^{\ell,u}_s  + 2\eta^2 L_s.
\]
Here, the strategy space for the learner (the minimization player) is the set of all distributions over $\cPinterval^{rn}$: $\mathcal{Q}^L = \Delta\cPinterval^{rn}$. A priori, the strategy space for the adversary is $\mathcal{Q}_{\rho,rn}$  the set of all $(\rho,rn)$-smooth distributions, but we show that it suffices to take $\mathcal{Q}^A = \hat{\mathcal{Q}}_{\rho,rn}$, the set of all discrete $(\rho,rn)$-smooth distributions (i.e. restricting the adversary in this way does not change the value of the game). 

\begin{lemma}
\label{lem:discreteok}
For any strategy $Q^L \in \Delta \cPinterval^{rn}$ for the learner, the adversary has a best response amongst the set of all $(\rho,rn)$-smooth distributions with support only over the discretization $\{0,1/{rn},\ldots,1\}$.  In other words, for any   $Q^L \in \Delta \cPinterval^{rn}$, there exists a $\hat{Q}^A \in \hat{\mathcal{Q}}_{\rho,rn}$ such that:
\[\hat{Q}^A \in \argmax_{Q^A \in  \mathcal{Q}_{\rho,rn}} \E_{\substack{(\ell,u) \sim Q^L, \\y \sim Q^A}}[u((\ell,u), y)].\]

\end{lemma}
\begin{proof}
Fix any $Q^{A'} \in \argmax_{Q^A \in  \mathcal{Q}_{\rho,rn}} \E_{(\ell,u) \sim Q^L, y \sim Q^A}[u((\ell,u), y)]$ --- i.e.  an arbitrary $(\rho,rn)$-smooth best response for the maximization player. We will construct a discrete $(\rho,rn)$-smooth $\hat{Q}^A \in \hat{\mathcal{Q}}_{\rho,rn}$ that obtains the same objective value, as follows. For each $\tfrac{i}{rn} \in \{0, 1/rn, \ldots, 1 \}$, let:
\[\Pr_{y \sim Q^A}\left[y = \tfrac{i}{rn}\right] = \Pr_{y \sim Q^{A'}}\left[y \in \Bigl[\frac{i}{rn}, \frac{i+1}{rn}\Bigr)\right].\]
Observe first by construction that $Q^A$ is a discrete probability distribution (because $Q^{A'}$ is a probability distribution over $[0,1]$, and the set of intervals $[\tfrac{i}{rn}, \tfrac{i+1}{rn})$ partition the unit interval), and that $Q^A$ is $(\rho,rn)$-smooth because $Q^{A'}$ is $(\rho,rn)$-smooth --- we have $\Pr_{y \sim Q^A} [y = \tfrac{i}{rn}] \leq \rho$ for all $i$. Finally observe that  (by definition) for any $(\ell,u) \in \cPinterval^{rn}$, $\ell,u \in \{0, 1/rn, \ldots, 1\}$. 

Therefore, we have that for any $(\ell, u) \in \cPinterval^{rn}$, any $i \in \{0, 1,\ldots, n\}$, and any $y, y' \in \Bigl[\tfrac{i}{rn}, \frac{i+1}{rn}\Bigr)$, $u((\ell,u), y) = u((\ell,u), y')$. To see this, note that $y \geq \ell$ if and only if $y' \geq \ell$, and $y < u$ if and only if $y' < u$. Since $v_\delta((\ell,u),y)$ is a function only of the indicators of the event that $\ell \leq y < u$, this proves the claim. 
\end{proof}

Recall (from Observation~\ref{obs:smooth}) that for any $(\rho,rn)$-smooth label distribution $Q^A$, there exists an interval $(\ell,u) \in \cPinterval^{rn}$ such that $| \Pr_{y \sim Q^A}[y \in [\ell,u)] - (1-\delta)|  \le \rho$, meaning there exists $(\bell, \bu)$ such that $\E_{\ty_{s+1}}[v_\delta((\bell,\bu), \ty_{s+1})] \le \rho$. We can thus bound the value of the game we have defined as follows:

\begin{eqnarray*}
\max_{Q^A \in \hat{\mathcal{Q}}_{\rho,rn}}\min_{(\ell,u) \in \cPinterval^{rn}} \E_{y \sim Q^A} [u(\ell,u),y] &\leq& 
\sum_{\cG(x_{s+1})}\exp(\eta V_s^{G,(\ell,u)})\left(\eta \rho \right) + \exp(-\eta V_s^{G,(\ell,u)})\left(\eta \rho \right) + 2\eta^2 L_s,\\
&\leq& L_s (\eta \rho + 2\eta^2 ).
\end{eqnarray*}
It is easy to verify that $\Delta \cPinterval^{rn}$ and $\hat{\mathcal{Q}}_{\rho,rn}$ are both compact sets (closed and bounded in a finite dimensional Euclidean space) and convex.  The lemma then follows by applying the minimax theorem (Theorem~\ref{thm:minimax}).
\end{proof}

\begin{corollary}
\label{cor:intervalexists}
For every $s \in [T]$, $\pi_s \in \Pi^*$, and $x_{s+1} \in \cX$ (which fixes $L_s$ and $Q^L_{s+1}$), and any distribution over $\cY$:

\[
    \E_{ (\ell,u) \sim Q^L_{s+1}}[\tL_{s+1}|\pi_s] \leq L_s +
     \E_{(\bell,\bu) \sim Q^L_{s+1}}\left[\Delta_{s+1}(\pi_s,x_{s+1},(\bell_{s+1},\bu_{s+1}))\right] < L_s\cdotk\left(1 + \eta\rho + 2\eta^2\right).
\]  
\end{corollary}

As with mean multicalibration, Lemma~\ref{lem:existsinterval} defines (existentially) an algorithm that the learner can use to make predictions --- Algorithm~\ref{alg:intervalexistential}. We will now show that Algorithm~\ref{alg:intervalexistential} (if we could compute the distributions $Q^L_t$) results in multivalid prediction intervals. 

\begin{algorithm}[H]
\SetAlgoLined
\begin{algorithmic}
\FOR{$t=1, \dots, T$}
	\STATE Observe $x_t$. Given $\pi_{t-1}$ and $x_t$, let $Q^L_t \in \Delta \cPinterval^{rn}$ be the distribution over prediction intervals whose existence is established in Lemma~\ref{lem:existsinterval}.
	\STATE Sample $(\bell,\bu) \sim Q^L_t$ and predict $(\bell_t,\bu_t) = (\bell,\bu)$
\ENDFOR
\end{algorithmic}
\caption{A Generic Multivalid Predictor}
\label{alg:intervalexistential}
\end{algorithm}
\begin{lemma}
\label{lem:surrogatelossinterval}
Against any adversary who is constrained to playing $(\rho,rn)$-smooth distributions, Algorithm~\ref{alg:intervalexistential} results in surrogate loss satisfying:
\[
    \E_{\tpi_T}[\tL_T] \le 2|\cG|n^2\cdotk \exp\left(T\eta \rho+2T\eta^2\right).
\]

\end{lemma}

\begin{proof}
Using Corollary~\ref{cor:intervalexists}, the first part of Theorem~\ref{thm:general-bounds} applies in this case to the process $L$ with $L_0 = 2 |G|n^2$ and $c = \rho$. The bound follows by plugging these values into \eqref{eqn:expectedbound}. 
\end{proof}

Finally, we can calculate a bound on our expected multivalidity  error. The proof (which mirrors similar claims in previous sections) is in the Appendix.

\begin{restatable}{theorem}{intervalmulti}
\label{thm:intervalmulti}
When Algorithm~\ref{alg:intervalexistential} is run using $n$ buckets, discretization parameter $r$ and $\eta = \sqrt{\frac{\ln(2|\cG|n^2)}{2T}} \in (0, 1/2)$, then against any adversary constrained to playing $(\rho,rn)$-smooth distributions, its sequence of interval predictions is $\alpha$-multivalid with respect to $\cG$ in expectation over the randomness of the transcript $\pi_T$, where:
$$\E [\alpha] \leq \rho + 2\cdotk \sqrt{\frac{2\ln(2|\cG|n^2)}{T}}.$$
\end{restatable}

We can also use the second part of Theorem~\ref{thm:general-bounds} to prove a high probability bound on the multicalibration error of Algorithm~\ref{alg:intervalexistential}. The proof is in the Appendix. 

\begin{restatable}{theorem}{hpintervalcalibration}
\label{thm:hp-interval-calibration}
When Algorithm~\ref{alg:intervalexistential} is run using $n$ buckets, discretization parameter $r$ and $\eta = \sqrt{\frac{\ln(2|\cG|n^2)}{2T}} \in (0, 1/2)$, then against any adversary who is constrained to playing $(\rho,rn)$-smooth distributions, its sequence of interval predictions is $\alpha$-multivalid with respect to $\cG$ with probability $1-\lambda$ over the randomness of the transcript $\pi_T$: 
\[ 
    \alpha \leq \rho +  4 \cdotk \sqrt{\frac{2}{T} \ln\left(\frac{2|\cG|n^2}{\lambda} \right)}.
\]
\end{restatable}

\begin{remark}
\label{rem:perturb}
The hypothesis of our theorems has an assumption: that the adversary is restricted to playing $(\rho,rn)$-smooth distributions. This may be reasonable if we are not in a truly adversarial setting, and are simply concerned with unknown distribution shift. But what if we are truly in an adversarial environment? It turns out that in order to have a useful algorithm, we need not make \emph{any} assumptions on the adversary at all. Observe that if we randomly perturb observed labels with uniform noise: $\hat y_t = y_t + U(-\epsilon,\epsilon)$, then the distribution on our perturbed points will be $\left(\frac{1}{2rn\epsilon},rn\right)$-smooth by construction. Now recall that $r$ is a parameter that we can select. By taking $r = \frac{1}{2\rho n\epsilon}$, we obtain that the distribution on the perturbed points is $(\rho, rn)$-smooth, for a value of $\rho$ that we can take as small as we like. Taking $\rho = 1/\sqrt{T}$ ($r = \frac{\sqrt{T}}{2n\epsilon}$) makes the contribution of $\rho$ to the multivalidity error a low order term. If we feed these perturbed labels to our algorithm, we will obtain prediction intervals that are multivalid for the perturbed labels. But observe that if we simply widen each of our prediction intervals by $\epsilon$ at each end, so that we predict the interval $[\bell_t - \epsilon, \bu_t +  \epsilon)$, then our intervals continue to have coverage probability at least $1-\delta$ for the original, unperturbed labels. We can similarly take $\epsilon$ as small as we like. Our algorithm in Section~\ref{sec:intervalalg} will have running time depending polynomially on $r$, so with this construction obtains a polynomial dependence on $1/\epsilon$. 
\end{remark}

\subsection{Deriving an Efficient Algorithm via Equilibrium Computation}
\label{sec:intervalalg}

In this section, we show how to implement Algorithm~\ref{alg:intervalexistential} to efficiently sample from the distributions $Q^L_t$ whose existence we established in Lemma~\ref{lem:existsinterval}. We do this by efficiently computing an equilibrium strategy $Q^L_t$ using the Ellipsoid algorithm by solving the linear program in Figure~\ref{lp:init}. This linear program has $(rn)^2+1$ variables and (a priori) an infinite number of constraints. However, as we will show:
\begin{enumerate}
    \item The number of constraints can in fact be taken to be finite (albeit exponentially large), and
    \item We have an efficient separation oracle to identify violated constraints.
\end{enumerate}
Together, this allows us to apply the Ellipsoid algorithm. 
\vspace{-7mm}
\begin{figure}[H]
\begin{align*}
    &\min_{Q^L \in \cPinterval^{rn}} \,\,  \gamma  \text{ s.t.}\\
\forall Q^A \in \hat{\mathcal{Q}}_{\rho,rn}:   &\sideset{}{_{y \in \cP^{rn}}} \sum
     Q^A(y)   
    \left(\sideset{}{_{(\ell,u) \in \cPinterval^{rn}}}\sum  Q^L((\ell,u))  \left(v_{\delta}((l,u),y)  C^{\ell,u}_{t-1}(x_t)\right) \right)  \leq \gamma, \\
    &\sideset{}{_{(\ell,u) \in \cPinterval^{rn}}}\sum \,\, Q^L((\ell,u)) = 1, \\
 \forall \, (\ell, u) \in \cPinterval^{rn}:    & \,\, Q^L((\ell,u)) \geq 0 .
\end{align*}
\caption{A Linear Program for Computing a Minimax Equilibrium Strategy for the Learner at Round $t$. }
\label{lp:init}
\end{figure}


\begin{algorithm}[H]
\SetAlgoLined
\begin{algorithmic}
\STATE \textbf{INPUT}: $\epsilon > 0$.
\FOR{$t=1, \dots, T$}
	\STATE Observe $x_t$ and compute  $C_{t-1}^{\ell, u}(x_t)$ for each $(\ell,u) \in \cPinterval^{rn}$ as in \eqref{eqn:def-C-interval}.
\STATE Solve the Linear Program from Figure~\ref{lp:init} using the Ellipsoid algorithm, with Algorithm~\ref{alg:intervalseparation} as a separation oracle, to obtain an $\epsilon$-approximate solution $Q^L_t \in \Delta \cPinterval^{rn}$.
\STATE Predict $(\bell_t,\bu_t) = (\ell, u)$ with probability $Q^L_t((\ell,u))$.
\ENDFOR
\end{algorithmic}
\caption{Von Neumann's Multivalid Predictor}
\label{alg:intervalalg}
\end{algorithm}

\begin{theorem}
\label{thm:interval-efficient}
Algorithm~\ref{alg:intervalalg} implements Algorithm~\ref{alg:intervalexistential}. In particular, it obtains multivalidity guarantees arbitrarily close to those of Theorems \ref{thm:intervalmulti} and \ref{thm:hp-interval-calibration}. Namely, for any desired $\epsilon > 0$, we have the following.

Choosing $\eta = \sqrt{\frac{\ln(2|\cG| n^2 + \epsilon)}{2T}} \in (0, 1/2)$, we have against any adversary constrained to playing $(\rho,rn)$-smooth distributions that the sequence of prediction intervals produced by Algorithm~\ref{alg:intervalalg} is $\alpha$-multivalid with respect to $\cG$ in expectation over the randomness of the transcript $\pi_T$, where:
\[\E [\alpha] \leq \rho + 2\sqrt{\frac{2\ln(2|\cG| n^2 + \epsilon)}{T}}.\]

Moreover, choosing $\eta = \sqrt{\frac{\ln(2|\cG|n^2) + \epsilon T}{2T}} \in (0, 1/2)$,
we have, with probability $1-\lambda$ over the randomness of the transcript $\pi_T$,
\[\alpha \leq \rho + 4 \cdotk \sqrt{\frac{2}{T} \ln\left(\frac{2|\cG|n^2}{\lambda} \right)  + 2\epsilon}.\]

The runtime of Algorithm~\ref{alg:intervalalg} is linear in $|\cG|$, and polynomial in $r, n, T$, and $\log(\frac{1}{\epsilon})$.
\end{theorem}
\begin{remark}
As with all of our other algorithms, the dependence on $|\cG|$ can be replaced at each round with a possibly substantially smaller dependence on the number of groups which contain $x_t$, $|\cG(x_t)|$, whenever this set is efficiently enumerable.
\end{remark} 
\begin{proof}
Recall that at each round $t$ we need to find an equilibrium strategy for the learner in the zero-sum game defined by the objective function:
\begin{eqnarray*}
u((\ell,u), y)
&=& \eta v_\delta((\ell, u), y)C^{\ell,u}_{t-1}  + 2\eta^2 L_{t-1}\\
&=& \eta \left(\mathrm{Cover}((\ell,u), y) - (1-\delta)\right) C^{\ell,u}_{t-1} + 2\eta^2 L_{t-1}.
\end{eqnarray*}
In this game, the strategy space for the learner is the set of all distributions over discrete intervals: $\mathcal{Q}^L = \Delta \cPinterval^{rn}$, and (by Lemma~\ref{lem:discreteok}), the action space for the adversary can be taken to be the set of all discrete smooth distributions: $\mathcal{Q}^A = \hat{\mathcal{Q}}_{\rho,rn}$.

The equilibrium structure of a game is invariant to adding and multiplying the objective function by a constant. Hence we can proceed to solve the game with the objective function: 
$$u((\ell,u), y) = \left(\mathrm{Cover}((\ell,u), y) -(1- \delta)\right)C^{\ell,u}_{t-1}.$$
To compute an  equilibrium of the game, we need to solve for a distribution $Q^L$ satisfying:
\[
    Q^L \in \argmin_{Q^L \in \Delta \cPinterval^{rn}}\max_{Q^A \in \hat{\mathcal{Q}}_{\rho,rn}}\E_{\substack{y \sim Q^A,\\ (\ell,u) \sim Q^L}}[u(\ell,u),y)].
\]

We can write this as a linear program, over the $O((rn)^2)$ variables $Q^L((\ell,u))$: see Figure~\ref{lp:init}. A priori, this linear program has infinitely many constraints.\footnote{Although in fact, in the proof of Lemma~\ref{lem:greed}, we will show that without loss of generality we can equivalently impose only finitely (but exponentially) many constraints.} Nevertheless, we show that we can efficiently implement a \emph{separation oracle}, which given a candidate solution $(Q^L, \gamma)$, can find a violated constraint whenever one exists. This is sufficient to efficiently find, using the Ellipsoid algorithm, a feasible solution of the linear program achieving value within any desired $\epsilon > 0$ of the optimum. 

\begin{algorithm}[H]
\label{alg:oracle1}
\SetAlgoLined
\begin{algorithmic}
\STATE \textbf{INPUT}: A proposed solution $Q^L$, $\gamma$ for Linear Program~\ref{lp:init}
\STATE \textbf{OUTPUT}: A violated constraint of Linear Program~\ref{lp:init} if one exists, or a certification of feasibility. 
\FOR{$i=0, 1 \dots, rn$}
	\STATE Compute
	$$W_i \equiv \sideset{}{_{(\ell, u) \in \cPinterval^{rn} :  \mathrm{Cover}((\ell,u),\frac{i}{rn}) = 1}}\sum Q^L((\ell,u)) C^{\ell,u}_{t-1}$$
\ENDFOR
\STATE Let $\sigma:\{0,\ldots, rn\}\rightarrow \{0,\ldots,rn\}$ be a permutation such that:
$$W_{\sigma(0)} \geq W_{\sigma(1)} \geq \ldots \geq W_{\sigma(rn)}.$$
\FOR{$i=0, 1 \dots, rn$}
    \STATE Set $Q^A(\sigma(i)) = \min(\rho, 1 - \sum_{j=0}^{i-1} Q^A(\sigma(j))$
\ENDFOR
\IF{$\sideset{}{_{y \in \cP^{rn}}} \sum
     Q^A(y)   
    \left(\sideset{}{_{(\ell,u) \in \cPinterval^{rn}}}\sum  Q^L((\ell,u))  \left(v_{\delta}((l,u),y)  C^{\ell,u}_{t-1}\right) \right)  > \gamma,$ or $Q^L$ not a prob.\ dist.\ }
    \RETURN the violated constraint.
\ENDIF
\RETURN FEASIBLE
\end{algorithmic}
\caption{A Separation Oracle for Linear Program~\ref{lp:init}}
\label{alg:intervalseparation}
\end{algorithm}

We will identify the output of Algorithm~\ref{alg:intervalexistential} with the distribution $Q^A$ associated with the constraint it outputs. Observe that if there is a violation (i.e. the proposed solution $Q^L, \gamma$ is infeasible), and there are ties, i.e. indices $i$ and $j$ such that $W_i = W_j$, then there are multiple candidate $Q^A$'s that could be the output of Algorithm~\ref{alg:intervalseparation}.  To that end, note that a solution $Q^A$ can be output by Algorithm~\ref{alg:intervalseparation} if and only if it is \emph{greed-induced}:
\begin{definition}
Let $W_i$ be defined as in Algorithm~\ref{alg:intervalseparation} for $i \in \{0,\ldots,rn\}$. We say that a distribution $Q^L \in \hat{\mathcal{Q}}_{\rho,rn}$ is \emph{greed-induced} if for every pair of indices $i$ and $j$ such that $W_i > W_j$:
$$Q^A(j) > 0 \implies Q^A(i) = \rho.$$
\end{definition}

\begin{lemma}
\label{lem:greed}
    Algorithm~\ref{alg:oracle1} is a separation oracle for the Linear Program in Figure~\ref{lp:init}. It runs in time $O((rn)^3)$.
\end{lemma}

\begin{proof}
Recall that a separation oracle is given a candidate distribution $Q^L \in \Delta\cPinterval^{rn}$ and a value $\gamma \in \mathbb{R}$, and must determine if there is any $Q^A \in \hat{\mathcal{Q}}_{\rho,rn}$ such that:
$$\sideset{}{_{y \in \cP^{rn}}} \sum
     Q^A(y)   
    \left(\sideset{}{_{(\ell,u) \in \cPinterval^{rn}}}\sum  Q^L((\ell,u))  \left(v_{\delta}((l,u),y)  C^{\ell,u}_{t-1}\right) \right)  > \gamma. $$
Suppose the learner is playing a distribution $Q^L \in \Delta \cPinterval^{rn}$ over intervals.
The adversary will seek to maximize the objective function over the set of $(\rho,rn)$-smooth distributions $Q^A \in \hat{\mathcal{Q}}_{\rho,rn}$. Recall that $v_{\delta}((\ell,u),y) = \textrm{Cov}((\ell,u),y) - (1-\delta).$ Therefore, fixing a distribution $Q^L$ for the learner, there are terms in the objective function that are independent of the adversary's actions (roughly, those corresponding to the $(1-\delta)$ term), and hence irrelevant to the inner maximization problem (i.e the adversary's best response). We define the following quantity $\tilde u$ which eliminates these $y$-independent terms: 
\begin{eqnarray*}
\tilde u(Q^L,Q^A) &=& \sum_{i \in \{0,\ldots,rn\}}Q^A\left(\frac{i}{rn}\right)\sideset{}{_{(\ell, u) \in \cPinterval^{rn} :  \mathrm{Cover}((\ell,u),\frac{i}{rn}) = 1}}\sum Q^L((\ell,u)) C^{\ell,u}_{t-1},\\
&=& \sum_{i \in \{0,\ldots, rn\}}Q^A\left(\frac{i}{rn}\right) W_i.
\end{eqnarray*}
Observe that for any $Q^L \in \Delta \cPinterval$:
$$\argmax_{Q^A \in \hat{\mathcal{Q}}_{\rho,rn}} \left(\E_{\substack{\tilde y \sim Q^A, \\(\tilde\ell,\tilde u)\sim Q^L}}[u((\tilde \ell,\tilde u),\tilde y)] \right) = \argmax_{Q^A \in \hat{\mathcal{Q}}_{\rho,rn}} \tilde u (Q^L,Q^A).$$

Hence, to derive a separation oracle, it suffices to find an algorithm which maximizes $\tilde u$ given a fixed distribution over intervals $Q^L$ for the learner. This is how we proceed.

Observe that by the argument above, the adversary's problem is equivalent to solving: 
\begin{align*}
    \max_{Q^A} &\sideset{}{_{i \in \{0,\ldots, rn\}}} \sum Q^A\left(\frac{i}{rn}\right) W_i,\\ 
    & \sideset{}{_{i \in \{0,\ldots, rn\}}} \sum Q^A\left(\frac{i}{rn}\right)=1,\\
    \forall \, i \in \{0,\ldots, rn\}:& \,\, Q^A\left(\frac{i}{rn}\right) \leq \rho,\\
    \forall \, i \in \{0,\ldots, rn\}:& \,\,Q^A\left(\frac{i}{rn}\right) \geq 0.
\end{align*}
By observation, this is a fractional knapsack problem---the value of each item $i \in \{0, \dots, rn\}$ is $W_i$, the quantity of each item $i$ is $\rho$, and the total capacity is $1$. Therefore the optimal solution is greed-induced. 



To bound the runtime of Algorithm~\ref{alg:intervalseparation}, first observe that checking that $Q^L$ is a probability distribution  takes time $O((rn)^2 \log rn)$. Now, we focus on the remaining constraints.
Since the quantities $C_{t-1}^{\ell, u}$ are precomputed at the beginning of round $t$, the separation oracle computes $W_i$ for each $i \in \{0,\ldots, rn\}$ in time $O((rn)^2)$, and hence we can compute all $W_i$'s in time $O((rn))^3$.  All that remains is to sort the indices $W_i$ which takes time $O(rn\ln rn)$, which is a low order term. Altogether, this results in a runtime of $O((rn)^3)$ for Algorithm~\ref{alg:intervalseparation}.
\end{proof}

Now, we verify that Algorithm~\ref{alg:intervalalg} runs efficiently --- to do so, we need to show that the Ellipsoid algorithm can efficiently (approximately) solve Linear Program~\ref{lp:init}. 
\begin{lemma}
Each run of the Ellipsoid algorithm within Algorithm~\ref{alg:intervalalg} solves the LP to a desired accuracy $\epsilon> 0$ in runtime $\mathrm{poly}(rn,
\log |\mathcal{G}|, T, \log \frac{1}{\epsilon})$.  Consequently, Algorithm~\ref{alg:intervalalg} runs in time $\mathrm{poly}(rn, |\mathcal{G}|, T, \log \frac{1}{\epsilon})$, where the dependence on $|\cG|$ is $O(|\cG|)$.
\end{lemma}

\begin{proof}
To ensure the Ellipsoid has polynomial runtime, we need to satisfy the conditions of Theorem~\ref{thm:ellipsoid}.

We first check that the feasible set of Linear Program~\ref{lp:init} is a polyhedron, i.e.\ that it has \emph{finitely many faces}. By Lemma~\ref{lem:greed} above, the adversary always has a \emph{greed-induced} best-response $Q^A$ constructed by Algorithm~\ref{alg:oracle1}. Every distribution $Q^A$ output by Algorithm~\ref{alg:oracle1} corresponds to selecting $\lfloor \frac{1}{\rho}\rfloor$ ``full" buckets that will have probability $\rho$ each and one bucket for the remaining probability mass, so there are at most $rn \cdot {rn \choose \lfloor \frac{1}{\rho}\rfloor} = O(rn \cdot 2^{rn})$ such distributions. The feasible set of Linear Program~\ref{lp:init} is thus equivalently given by the corresponding finitely many ($O(rn \cdot 2^{rn})$) constraints. 

Thus, the feasible region of LP~\ref{lp:init} is indeed a polyhedron; however, exponential terms in the coefficients of the constraints associated with the adversarial best-responses (which are due to our definition of the soft-max surrogate loss) prevent it from being \emph{rational}. To fix this, we only keep $O(\log \frac{1}{\epsilon})$ bits of precision after the integer part of every coefficient of the original LP, resulting in a new LP whose coefficients are all rational and within $\pm \frac{\epsilon}{2}$ from their original values in LP~\ref{lp:init}. The new LP indeed has a rational polyhedron as its feasible region.

We now observe that Linear Program~\ref{lp:init} has the form given in Lemma~\ref{lem:lpapprox}. This implies that by solving the just described rational LP corresponding to LP~\ref{lp:init} \emph{exactly}, we will obtain the desired \emph{$\epsilon$-approximate} solution to Linear Program~\ref{lp:init}. With this in mind, it remains to bound the bit complexity of the rational LP.


Consider any constraint of the rational LP. The coefficient of each variable $Q^L((\ell,u))$ has absolute value at most:
\begin{eqnarray*}
\max_{(\ell,u) \in \cPinterval}\sum_{G \in \cG}\exp(\eta V_{t-1}^{G,(\ell, u)}) -  \exp(-\eta V_{t-1}^{G,(\ell,u)}) &\leq& |\mathcal{G}| 2\exp\left(\eta \max_{G \in \cG, (\ell,u) \in \cPinterval} \left|V_{t-1}^{G,(\ell,u)} \right|\right) \\
&\leq& 2|\cG|\exp(\eta T) \\
&\leq& 2|\cG|\exp(T).
\end{eqnarray*}
Thus, every constraint in the rational LP has bit complexity at most: \[O\left((rn)^2 \cdot \left(\log |\cG| + T + \log \frac{1}{\epsilon}\right) \right),\] where the $\log \frac{1}{\epsilon}$ term reflects the chosen precision. This is polynomial in $r, n, T, \log |\cG|$, and $\log \frac{1}{\epsilon}$. Also, the objective function, which is simply $\gamma$, takes $O((rn)^2)$ bits to write down.

We may now apply Theorem~\ref{thm:ellipsoid} with the parameters $q = O((rn)^2)$, $\phi = O\left((rn)^2(\log |\cG| + T + \log \frac{1}{\epsilon}) \right)$, 
$c = O((rn)^2)$. The runtime of the separation oracle (which, we note, applies to the rational LP just as it did for the original LP) is $O((rn)^3)$ by Lemma~\ref{lem:greed}. Hence, the Ellipsoid algorithm will solve Linear Program~\ref{lp:init} with accuracy $\epsilon$ in time $\mathrm{poly}(rn, \log |\mathcal{G}|, T, \log \frac{1}{\epsilon})$.

Hence, Algorithm~\ref{alg:intervalalg} has time complexity $\mathrm{poly}(rn, |\cG|, T, \log \frac{1}{\epsilon})$ --- where the dependence on $|\cG|$ is linear, because we precompute the $C_{t-1}^{\ell, u}$'s once at the beginning of each round $t$, taking time linear in $|\cG|$, and the runtime of the Ellipsoid algorithm is polylogarithmic in $|\cG|$. (We remark once more that the dependence on $|\cG|$ can be reduced to a dependence on $|\cG(x_t)|$ if $\cG(x_t)$ is efficiently enumerable, and that this might be much smaller.)
\end{proof}

Finally, we need to demonstrate that the claimed multivalidity guarantees (which are a function of the chosen $\epsilon > 0$) indeed hold.
\begin{restatable}{lemma}{intervalepsmultivalidity}
Algorithm~\ref{alg:intervalalg} achieves the  multivalidity guarantees stated in Theorem \ref{thm:interval-efficient}.
\end{restatable}
The proof of this lemma involves repeating several calculations from Section \ref{sec:intervalexistential} with an $\epsilon$ error term, and so is deferred to the Appendix.
\end{proof}

\section{Augmenting an Existing Learning Algorithm}
For simplicity of exposition, throughout this paper, we have described our algorithms as predicting properties of the arriving labels $y_t$ directly. But often that is not what we want: instead, we have some procedure $f_t:\cX\rightarrow \cY$ making point predictions --- that is, mapping features to labels --- and we are interested in properties of the \emph{residuals} $f_t(x_t) - y_t$. For example, $f_t$ may be some complicated (but powerful) learning procedure --- for example, maybe at every round, we train a neural network on the data we have observed so far to predict the labels of new observations. It may be that the labels $y$ have high variance, but that the residuals $y_t-f_t(x_t)$ are tightly concentrated around zero (because $f_t$ is highly accurate). To quantify the uncertainty of our predictions, we want to provide prediction intervals related to our predictions $f_t(x_t)$ --- that is, to compute prediction intervals for the residuals. We may similarly be interested in the variance of the residuals, etc. 

We can easily use the algorithms we have developed in this paper for this. We have no understanding of $f_t$ or the distribution on predictions $f_t(x_t)$ it induces (say, because $f_t$ varies substantially from round to round because of retraining) --- but because our algorithms handle adversarially chosen sequences of examples, they apply equally well when we feed them the residuals rather than the original labels. We have derived our algorithms under the scaling that $y_t \in [0,1]$, and the residuals $y_t - f_t(x_t)$ may lie in $[-1,1]$, so to apply the same bounds we have derived, we need to compute \emph{centered residuals} $y'_t = \frac{1}{2} + \frac{1}{2}(y_t-f_t(x_t))$. (This simply corresponds to a rescaling and a shift so that the residuals again lie in $[0,1]$. Thus, the following algorithm is able to provide prediction intervals around the predictions of an arbitrary sequence of predictors $f_t$ (and similar constructions work for predicting means and variances of the residuals): 

\begin{algorithm}[H]
\label{alg:conformal}
\SetAlgoLined
\begin{algorithmic}
\STATE Instantiate $\mathcal{A}$, a copy of Algorithm~\ref{alg:intervalalg}.
\FOR{$t=1,\ldots,T$}
    \STATE Observe $x_t$, and compute a point prediction $f_t(x_t)$ (for an arbitrary procedure $f_t$).
    \STATE Feed $x_t$ to $\mathcal{A}$ and receive a prediction interval $(\bell_t,\bu_t)$.
    \STATE Output point prediction $f_t(x_t)$ and prediction interval $(f_t(x_t) + 2\bell_t-1, f_t(x_t) + 2\bu_t-1)$.
    \STATE Observe $y_t$ and feed the centered residual $y'_t = \frac{1}{2} + \frac{1}{2}(y_t-f_t(x_t))$ to $\mathcal{A}$
\ENDFOR
\end{algorithmic}
\caption{Endowing Arbitrary Point Predictors with Prediction Intervals}
\label{alg:conformal-arbit}
\end{algorithm}

We observe that $y_t \in [f_t(x_t) + 2\bell_t-1, f_t(x_t) + 2\bu_t-1)$ if and only if $y'_t \in [\bell_t,\bu_t)$ by construction, and so the prediction intervals produced by Algorithm~\ref{alg:conformal-arbit} inherit the $(\alpha,n)$-multivalidity guarantees of Algorithm~\ref{alg:intervalalg} (Theorems \ref{thm:intervalmulti} and \ref{thm:hp-interval-calibration}): that with probability $1-\lambda$:
\[ 
\alpha \leq \rho + 4 \cdotk \sqrt{\frac{2}{T} \ln\left(\frac{2|\cG|n^2}{\lambda} \right)}.
\]
(the bound on expected multivalidity error holds as well).
Here $\rho$ is a smoothness parameter that depends on both the discretization $r$ we choose for our algorithm and the distribution over residuals at each round. Note that as discussed in Remark~\ref{rem:perturb}, with an appropriate selection of $r$, for any $\epsilon > 0$, we can make $\rho$ as small as we like by perturbing the centered residuals $y_t'$ with uniform noise $U(-\epsilon,\epsilon)$, at the cost of needing to widen our prediction intervals by $\epsilon$ on each end, i.e. predicting at each round: 
$$(f_t(x_t) + 2\bell_t-1-\epsilon, f_t(x_t) + 2\bu_t-1+\epsilon).$$
The computational cost of this is polynomial in $1/\epsilon$ and $1/\rho$, and the gain that we get by applying these perturbations is that we need assume nothing at all about either the adversarial sequence of examples, or about the properties of our predictors $f_t$. 

\subsection*{Acknowledgements}
We thank Aaditya Ramdas for helpful discussions about conformal prediction, as well as pointers to the literature. We thank Sergiu Hart, Dean Foster, Drew Fudenberg, and Rakesh Vohra for helpful discussions about calibration, as well as pointers to the literature. We also thank Ashish Rastogi for discussions about uncertainty estimation in practice. 
Gupta, Jung, Noarov, and Roth are supported in part by NSF grants CCF-1763307 and CCF-1934876, and a grant from the Simons Foundation. Pai is supported in part by NSF grant CCF-1763349.

\bibliographystyle{plainnat}
\bibliography{refs}

\appendix

\section{Batch Prediction}
\label{sec:batch}
\subsection{Preliminaries}
In the batch setting, there is an (unknown) probability distribution $\cD$ over $\cX \times \cY$. Let $\cD_\cX$ refer to the induced marginal distribution on $\cX$ and let $\cD_{\cY}$ refer to the induced marginal distribution on $\cY$. In the batch setting, rather than talking about a sequence of predictions, we need to refer to calibration properties of a single predictor with respect to the data distribution. We  here modify the definition of consistency and calibration accordingly --- but we will show how to convert calibration guarantees from the online setting to calibration guarantees in the offline setting. 

Given $n$ independent draws from  $\cD$, denoted by  $D=\{(x_t, y_t)\}_{t=t}^T$ the corresponding dataset. Given some $x$, our goal is to predict various properties of $\cD|x$.

\paragraph{Mean Predictions}
For mean prediction, we use a (possibly randomized) predictor $\bmu: \cX \to [0,1]$ that tries to predict the conditional mean $\E[y | x]$. Given a set $S \subseteq \cX$, we write 
\[
\mu(S) = \E_{\bmu}[\E_{\cD}[y|x \in S]], \quad \bmu(S) = \E_{\bmu}[\E_{\cD}[\bmu(x)] | x \in S]]
\]
for the conditional mean of labels on the distribution conditional on $x \in S$ and our conditional mean prediction. For calibration guarantees, we will be  concerned with sets that  depend on realizations of the randomized predictor $\bmu$, so it is important that in the above expressions, $S$ appears inside the expectation over $\bmu$.  Otherwise, we essentially use the same notation as in the online setting except instead of averaging over the empirical distribution, we average over the true distribution.

As in the online setting, we ``bucket'' our real valued predictions into $n$ buckets of width $\frac{1}{n}$, which serves as a measure of granularity of our calibration guarantee. Given a set $S\subseteq \cX$ and mean  predictor $\bmu$, we write 
\[S(\bmu, i) \equiv \left\{x \in S: \bmu(x) \in \Bm(i) \right\}\]
to be the set of points in $S$ whose mean predictions fall into the $i\tth$ bucket. When $\bmu$ is a randomized predictor, we think of $S(\bmu,i)$ as a random set where the randomness is over the random bits of $\bmu$.

\begin{definition}[Mean Consistency]
Call a mean predictor $\bmu$  $\alpha$-mean consistent on a set $S$ over distribution $\cD$ if 
\[
\left\vert \mu\left(S\right) - \bmu\left(S\right)  \right\vert \le \frac{\alpha }{\Pr_{\bmu, \cD_{\cX}}[x \in S]}.
\]

We note that we include the randomness of $\bmu$ when writing the measure of the set $S$ because we will be interested in random sets $S$ defined as a function of randomized predictors $\bmu$.
\end{definition}

We are now ready to define calibration, which asks for mean consistency on particular sets defined by the mean predictor itself:
\begin{definition}[Mean-Multicalibration] Fix a set $S \subseteq \cX$ and a true distribution $\cD$.
 A mean predictor $\bmu$ is $(\alpha, n)$-mean calibrated on a set $S$ over distribution $\cD$ if it is $\alpha$-mean consistent on every set $S(\bmu,i)$ over $\cD$, i.e. if  for each $i \in [n]$:
\[
\left\vert \mu\left(S(\bmu, i)\right) - \bmu(S(\bmu, i)) \right\vert \le \frac{\alpha }{\Pr_{\bmu, \cD_{\cX}}[x \in S(\bmu, i)]}.
\]
We say that $\bmu$ is  $\alpha$-mean multicalibrated  with respect to (a collection of sets) $\cG$ over $\cD$ if it is  $\alpha$-mean calibrated on every $G \in \cG$ over $\cD$.
\end{definition}

\paragraph{(Mean, Moment) Prediction}
In this case, we use a (randomized) predictor $\bmu: \cX \to [0,1]$ that tries to predict the conditional label mean $\E[y|x]$ and a  (randomized) predictor $\bmk: \cX \to [0,1]$ that tries to predict the conditional $k\tth$ central moment of the label distribution $\mk(x) = \E[(y - \E[y|x])^k | x]$. We again assume that $k$ is even so that the range of the $k\tth$ moment remains non-negative, but there is no obstacle other than notation to handling odd moments as well. Although for notational convenience we write $\bmk$ and $\bmu$ as separate functions, they may use correlated randomness. 

Analogously to our notation for mean prediction, we write for any $S \subseteq \cX$,
\[
\mk(S) = \E_{\bmu,\bmk}[\E_\cD[(y- \mu(S))^k | x \in S]] \quad \bmk(S) = \E_{\bmu,\bmk}[\E_{\cD}[ \bmk(x)| x \in S]].
\]
to denote the empirical $k^{\text{th}}$ central moment of the label distribution on the subsequence $S$ and for the average of the moment prediction on $S$, respectively. 

\begin{definition}[Moment Consistency]
We say that $(\bmu,\bmk)$ is $\alpha$-moment consistent on set $S \subseteq \cX$ if 
\[ 
|\mk(S) - \bmk(S)| \le \frac{\alpha}{\Pr_{\bmu, \bmk, \cD_\cX}[x \in S]}.
\]
Once again we include the randomness of $\bmu,\bmk$ because we will be concerned with sets that are defined in terms of $\bmu$ and $\bmk$. 
\end{definition}

For any $S \subseteq \cX$ and $i \in [n], j \in [n']$, we write
\[
    S(\bmu, i, \bmk, j) = \left\{x \in S: \bmu(x) \in \Bm(i), \bmk(x) \in \Bmp(j) \right\}.
\]
In words, $S(\bmu, i, \bmk, j)$ corresponds to the subset of points in $S$ in which our predicted mean falls in $\Bm(i)$ and $\Bmp(j)$. 

\begin{definition}[Mean-Conditioned Moment Multicalibration] 
We say that $(\bmu, \bmk)$ is $(\alpha,\beta, n, n')$-mean-conditioned moment multicalibrated with respect to $\cG$ over $\cD$, if for every $i \in [n], j \in [n']$, and $G \in \cG$, we have that $\bmu$ is $\alpha$-mean consistent on $G(\bmu, i, \bmk, j)$ and  $\bmk$ is $\beta$-moment consistent on $G(\bmu, i, \bmk, j)$:
\begin{align*}
&    |\mu(G(\bmu, i, \bmk, j)) - \bmu(G(\bmu, i, \bmk, j))| \le \frac{\alpha}{\Pr_{\bmu, \bmk, \cD_\cX}[x \in G(\bmu, i, \bmk, j)]},\\
&    |\mk(G(\bmu, i, \bmk, j)) - \bmk(G(\bmu, i, \bmk, j))| \le \frac{\beta}{\Pr_{\bmu, \bmk, \cD_\cX}[x \in G(\bmu, i, \bmk, j)]} .
\end{align*}
\end{definition}

For convenience, we sometimes combine the mean and moment predictor into a single predictor $h: \cX \to [0,1] \times [0,1]$ and write $\hbmu(x) = h(x)[0]$ to refer to its mean prediction and $\hbmk(x) = h(x)[1]$ to refer to its moment prediction. Also, we write $h(x) \in \Bmmp(i,j)$ if $\hbmu(x) \in \Bm(i)$ and $\hbmk(x) \in \Bmp(j)$. If $n$ and $n'$ are clear from the context, we just write $h(x) \in B(i,j)$.

\paragraph{Interval Prediction}
In this case, we want to come up with randomized predictors $\bell: \cX \to [0,1]$ and $\bu: \cX \to [0,1]$ such that the probability that $y$ falls between $\bell(x)$ and $\bu(x)$ is approximately $1-\delta$ for some specified failure probability $\delta$.  Although for notational convenience we write $\bell$ and $\bu$ as separate functions, they may use correlated randomness.  Using the notation given in Section \ref{sec:prelims}, we wish to devise $\bell, \bu$ such that $\E[\mathrm{Cover}((\bell(x), \bu(x)), y) | x] \approx 1-\delta$.

For any $S \subseteq \cX$, we write
\[  
    \bH_{\bell, \bu}(S) = \E_{\bell, \bu}[\E_{\cD}[\mathrm{Cover}((\bell(x), \bu(x)), x) | x \in S]].
\]

We again bucket our coverage intervals using a discretization parameter $n$, using the same notation as we used for moment predictions. For any $S \subseteq \cX$ and $i \leq j \in [n]$, we write
\[
    S(\bell, i, \bu, j) = \left\{x \in S: \bell(x) \in \Bm(i), \bu(x) \in \Bm(j) \right\}.
\]
\newcommand{\hbu}{h^{\bu}}
\newcommand{\hbell}{h^{\bell}}

For simplicity, we combine $\bell$ and $\bu$ into a single predictor $h: \cX \to [0,1] \times [0,1]$ and write $\hbell(x) = h(x)[0]$ and $\hbu(x) = h(x)[1]$. We say $h(x) \in \Bm(i,j)$ if $\hbell(x) \in \Bm(i)$ and $\hbu(x) \in \Bm(j)$. Also, when $n$ is clear from the context, we just write $B(i,j)$.

We can now define multivalidity in a way analogous to how we have defined multicalibration. 

\begin{definition}
We say that interval predictor $(\bell, \bu)$ is $\alpha$-consistent on set $S$ with respect to the failure probability $\delta \in (0,1)$, if we have the following
\[
    |\bH_{\bell,\bu}(S) - (1-\delta)| \le \frac{\alpha }{\Pr_{\bell, \bu, \cD}[ x \in S]}.
\]
\end{definition}

\begin{definition}
The interval predictors $(\bell, \bu)$ are $(\alpha,n)$-multivalid with respect to $\delta$ and $\cG$ over $\cD$, if for every $i\le j \in [n]$ and $G \in \cG$, we have that the interval predictions are  $\alpha$-consistent on $G(\bell, i, \bu, j)$ with respect to coverage probability $1-\delta$:
\[
    |\bH_{\bell,\bu}(G(\bell, i, \bu, j)) - (1-\delta)| \le \frac{\alpha}{\Pr_{\bell, \bu, \cD}[G(\bell, i, \bu,  j)]}.
\]
\end{definition}


\subsection{Online to Batch Conversion}
In this section, we show how to use our online algorithms to solve the corresponding batch multicalibration problems. In doing so we obtain improved sample complexity bounds for mean and mean-conditioned moment multicalibration for the batch problem, compared to prior work \cite{multicalibration,momentmulti}. However, in contrast to prior work which in the batch case solves for deterministic predictors, we obtain a randomized predictor via our online-to-offline reduction.

Previously, for any sequence of feature and label pairs $\{(x_t, y_t)\}_{t=1}^T$, we have shown how to construct a sequence of randomized predictors $\{h_t\}_{t=1}^T$ such that the sequence of predictions made from the predictors $\{p_t = h_t(x_t)\}_{t=1}^T$ is multivalid. We viewed the functions $h_t(x)$ only implicitly before, but we consider them explicitly here: for mean multicalibration, $h_t(x)$ is simply the distribution on label predictions $\bmu$ that would be made by Algorithm \ref{alg:meantemplate} at round $t$, given as input $x_t = x$ after a history defined by the sequence of examples $\{(x_s,y_s)\}_{s=1}^{t-1}$. 

In this section, we show that if we have a sample $D = \{(x_t, y_t)\}_{t=1}^T$ that is drawn independently from $\cD$, we can feed each element in this sample $D$ one-by-one to our online learning algorithm so as to obtain a sequence of predictors $\{h_t\}_{t=1}^T$. From this, we construct a single (randomized) predictor $h$  that is multivalid over the distribution $\cD$. $h$ will simply be the uniform mixture over the set of predictors  $\{h_t\}_{t=1}^T$.

\subsubsection{Mean prediction}

\newcommand{\hmeanavg}{h^{\text{mean}}}
\newcommand{\hmomentavg}{h^{\text{mean, moment}}}
\begin{algorithm}[H]
\label{alg:online-batch-mean}
\SetAlgoLined
\begin{algorithmic}
\STATE \textbf{INPUT}: Training dataset $D = \{(x_t, y_t)\}_{t=1}^T$
\STATE \textbf{Training:} Run Algorithm \ref{alg:meantemplate} on the sequence of examples $D$ to generate a transcript $\pi_T$. 
\STATE Denote by $h_t(x)$ the (randomized) mapping from $\cX$ to $[0,1]$ that Algorithm \ref{alg:meantemplate} induces as a function of transcript $\pi_{t-1}$ (the prefix of $\pi_T$ of length $t-1$). 
\STATE \textbf{Prediction:} On input $x$, sample $\hmeanavg(x)$ by selecting $t \sim [T]$ uniformly at random, and then sampling from $h_t(x)$. 
\STATE \hrulefill 
\STATE More explicitly, select $t \sim [T]$ uniformly at random and:
	\STATE Compute for each $i \in [n]$	$C^i_{t-1}(x) $ as defined in \eqref{eqn:csi} conditioning on $\pi_{t-1}$. 
    \IF {$C^i_{t-1}(x) > 0$ for all $i \in [n]$}
        \STATE Predict $\hmeanavg(x) = 1$.
    \ELSIF {$C^i_{t-1}(x) < 0$ for all $i \in [n]$}
        \STATE Predict $\hmeanavg(x) = 0$.  
    \ELSE 
        \STATE Find $i^* \in [n-1]$ such that $C^{i^*}_{t-1}(,x) \cdot C^{i^*+1}_{t-1},x) \leq 0$
        \STATE Define $0 \leq q_t \leq 1$:  (using the convention that 0/0 = 1)
        $$q_t = \frac{|C^{i^*+1}_{t-1}(x)|}{|C^{i^*+1}_{t-1}(x)| + |C^{i^*}_{t-1}(x)|}$$
       
        \STATE Predict $\hmeanavg(x) = \frac{i^*}{n}- \frac{1}{rn}$ with probability $q_t$ and $\hmeanavg(x) = \frac{i^*}{n}$ with probability $1-q_t$.
    \ENDIF
\end{algorithmic}
\caption{Von Neumann's Batch Mean Multicalibrator}
\end{algorithm}

\begin{theorem}
\label{thm:batch-mean}
Let $D=\{(x_t, y_t)\}_{t=1}^T$ be a dataset drawn i.i.d. from $\cD$, and suppose $T$ is large enough such that $\eta$ specified in Theorem \ref{thm:hpcalibration} falls in $(0,1/2)$. Let $\epsilon, \lambda > 0$. For an appropriately small choice of the discretization parameter $r$, with probability $1-\lambda$, Algorithm \ref{alg:online-batch-mean} produces a predictor $\hmeanavg$ that is $(\alpha, n)$-mean multicalibrated with respect to $\cG$ over $\cD$ where \[
\alpha = \left(6+\epsilon\right) \cdotk \sqrt{\frac{2}{T} \ln\left(\frac{4|\cG|n}{\lambda} \right)}
\]
\end{theorem}
\begin{proof}
In order to show that $\hmeanavg$ is $(\alpha, n)$-mean multicalibrated with respect to $\cG$ over $\cD$, it is sufficient to show for all $G \in \cG$ and $i \in [n]$
\begin{align*}
     \left\vert \E_{(x,y) \sim \cD, \hmeanavg}\left[\ind[\hmeanavg(x) \in B(i), G(x) = 1] \cdot \left(y - \hmeanavg(x) \right)\right]  \right\vert \le \alpha.
\end{align*}
We can calculate:
\begin{align}
    &\E_{(x,y) \sim \cD, \hmeanavg}\left[ \ind[\hmeanavg(x) \in B(i), G(x) = 1] \cdot \left( y-\hmeanavg(x)\right)\right] \nonumber \\
    &=  \sum_{(x,y)} \sum_{t=1}^T \cD[(x,y)] \cdot \Pr[\hmeanavg = h_t]  \cdot \Pr[h_t(x) \in B(i)] \cdot \ind[G(x) = 1] \cdot \left( y-h_t(x)\right)  \nonumber\\
    &=  \frac{1}{T} \sum_{(x,y)} \sum_{t=1}^T \cD[(x,y)] \cdot \Pr[h_t(x) \in B(i)] \cdot \ind[G(x) = 1] \cdot \left( y-h_t(x)\right)  \nonumber\\
    &=\frac{1}{T} \sum_{t=1}^T  \E_{(x, y) \sim \cD, h_t}\left[\ind[h_t(x) \in B(i), G(x) = 1]\cdot \left( y-h_t(x)\right)\right]  \label{eqn:online-to-batch-cal-error}
\end{align} 
Therefore, our goal is to upper bound the absolute value of \eqref{eqn:online-to-batch-cal-error}. We will show that if $D=\{(x_t, y_t)\}_{t=1}^T$ is sampled i.i.d. from $\cD$, the  empirical calibration error on the transcript $\pi_T$ generated during training  serves as a good estimate for \eqref{eqn:online-to-batch-cal-error}. And because we know from Theorem \ref{thm:hpcalibration} that for every sequence of examples, Algorithm $2$ produces predictions that will be empirically calibrated with high probability, our bound will follow. 

In particular, we know from Theorem \ref{thm:hpcalibration} that (for an appropriate choice of $r$) with probability $1-\lambda/2$ over the randomness of $\pi_T$ produced in training that for all $i \in [n], G \in \cG$:
\begin{align*}
    \left| \frac{1}{T}\sum_{t=1}^T \ind\left[\bmu_t \in B(i) , G(x_t)=1\right] \cdot \left(y_t - \bmu_t \right)\right| \le \left( 2+\epsilon\right) \cdotk \sqrt{\frac{2}{T} \ln\left(\frac{4|\cG|n}{\lambda} \right)}.
\end{align*}

Now, fixing $G \in \cG$ and $i \in [n]$, we use the following martingale argument to show that $\eqref{eqn:online-to-batch-cal-error}$ is close to the empirical calibration error with respect to $G$ and $i$  with high probability. Consider the following martingale sequence adapted to the filtration $\mathcal{F}_s = \sigma(\{(x_t, y_t), \bmu_t\}_{t=1}^s)$:
\[
    \tZ_s = Z_{s-1} + \E_{(x, y) \sim \cD, h_s}\left[\ind\left[h_s(x) \in B(i) , G(x)=1\right] \cdot \left(y - h_s(x)\right) | \pi_{s-1} \right] - \ind\left[\bmu_s \in B(i) , G(x_s)=1\right] \cdot \left(y_s - \bmu_s \right).
\]
 It's easy to see that the above sequence is a martingale: because
 \begin{align*}
     &\E_{(x, y) \sim \cD, h_s}\left[\ind\left[h_s(x) \in B(i) , G(x)=1\right] \cdot \left(y - h_s(x)\right) | \pi_{s-1} \right]  \\
     &= \E_{(x_s, y_s) \sim \cD, \bmu_s}\left[\ind\left[\bmu_s \in B(i) , G(x_s)=1\right] \cdot \left(y_s - \bmu_s \right) | \pi_{s-1}\right],
 \end{align*}
and so we have $\E[\tZ_s] = Z_{s-1}$.

Therefore, because $|Z_s - Z_{s-1}| \le 2$, we can apply Azuma's inequality (Lemma~\ref{lem:azuma}) to get that with probability $1-\lambda/2$ over the randomness of $\pi_T$ and $D$,
\[
     \left\vert \sum_{t=1}^T \E_{(x, y) \sim \cD, h_t}\left[\ind\left[h_t(x) \in B(i) , G(x)=1\right] \cdot \left(y - h_t(x) \right)\right] - \sum_{t=1}^T \ind\left[\bmu_t \in B(i) , G(x_t)=1\right] \cdot \left(y_t - \bmu_t \right) \right\vert \le 2\sqrt{2T\ln\left(\frac{4}{\lambda}\right)}.
\]
Therefore, Union bounding the above Azuma's inequality over all $i \in [n]$ and $G \in \cG$ gives us the result: we have with probability $1-\lambda$ over the randomness of $\cD$ and $\pi_T$,
\begin{align*}
    &\frac{1}{T} \left\vert \sum_{t=1}^T \E_{(x, y)\sim \cD, h_t}\left[\ind\left[h_t(x) \in B(i) , G(x)=1\right] \cdot \left(y - h_t(x) \right)\right] \right\vert, \\
     \le& \frac{1}{T} \left\vert \sum_{t=1}^T \ind\left[\bmu_t \in B(i) , G(x_t)=1\right] \cdot \left(y_t - \bmu_t \right)\right\vert + 2\sqrt{\frac{2 \ln\left(\frac{4|\cG|n}{\lambda}\right)}{T}},\\
     \le& \left(6+\epsilon\right) \cdotk \sqrt{\frac{2}{T} \ln\left(\frac{4|\cG|n}{\lambda} \right)}
\end{align*}
for every $i \in [n]$ and $G \in \cG$.
\end{proof}

\subsubsection{(Mean, Moment) Prediction}
We can use the same argument to show that we can feed $D=\{(x_t, y_t)\}_{t=1}^T$ drawn i.i.d. from $\cD$ into our Algorithm \ref{alg:momentalg} to obtain a randomized predictor $\hmomentavg$ that is $(\alpha, \beta, n, n')$-mean-conditioned-moment multicalibrated with respect to $\cG$ over $\cD$.

\begin{algorithm}[H]
\label{alg:online-batch-moment}
\SetAlgoLined
\begin{algorithmic}
\STATE \textbf{INPUT}: Training dataset $D = \{(x_t, y_t)\}_{t=1}^T$
\STATE \textbf{Training:} Run Algorithm \ref{alg:momentalg} on the sequence of examples $D$ to generate a transcript $\pi_T$. 
\STATE Denote by $h_t(x)$ the (randomized) mapping from $\cX$ to $[0,1] \times [0,1]$ that Algorithm \ref{alg:momentalg} induces as a function of transcript $\pi_{t-1}$ (the prefix of $\pi_T$ of length $t-1$). 
\STATE \textbf{Prediction:} On input $x$, sample $\hmomentavg(x)$ by selecting $t \sim [T]$ uniformly at random, and then sampling from $h_t(x)$. 
\STATE \hrulefill 
\STATE More explicitly, select $t \sim [T]$ uniformly at random and:
    \STATE Compute  $C_{t-1}^{\bmu, \bmk}(x), D_{t-1}^{\bmu,\bmk}(x),F_{\ell, t-1}^{\bmu,\bmk}(x)$ for each $(\bmu,\bmk) \in \cPmeanreduced^{r,n} \times \cPmomentreduced^{r,n'}$ as in (\ref{eqn:def-C}, \ref{eqn:def-D}, \ref{eq:def-F1}, \ref{eq:def-Fk}) conditioning on $\pi_{t-1}$.
    \STATE Find an $\epsilon$-approximate solution to the linear program from Figure~\ref{lp:moments}, to obtain solution $Q_t^L \in \learnerstrat$. 
   \STATE Predict $\hmomentavg(x) = (\bmu, \bmk)$ with probability $Q^L_t((\bmu,\bmk))$.
\end{algorithmic}
\caption{Von Neumann's Batch Mean Moment Multicalibrator}
\end{algorithm}

\begin{theorem}
\label{thm:batch-moment}
Assume $T > 2 \ln(\frac{8|\cG|n \cdot n'}{\delta})$ and $T$ is sufficiently large such that $\eta$ used in Theorem \ref{thm:momentalg-efficient} is in $(0,1/2)$. Let $D=\{(x_t, y_t)\}_{t=1}^T$ be a dataset drawn i.i.d. from $\cD$. Let $\epsilon, \delta > 0$. For an appropriately small choice of the discretization parameter $r$, with probability $1-2\lambda$, Algorithm \ref{alg:online-batch-moment} produces a predictor $\hmomentavg$ that is $(\alpha, \beta, n, n')$-mean-conditioned moment multicalibrated with respect to $\cG$ over $\cD$ where \[
\alpha = \left(6 +\epsilon' \right) \cdotk \sqrt{\frac{2}{T} \ln\left(\frac{8|\cG|n \cdot n'}{\lambda} \right) + 2\epsilon} 
\]
\[
    \beta = (k+3)\left(\left(5 +\epsilon' \right) \cdotk \sqrt{\frac{2}{T} \ln\left(\frac{8|\cG|n \cdot n'}{\lambda} \right) + 2\epsilon} \right) + \frac{k}{2n}
\]
\end{theorem}
\begin{proof}
Note that in order to show that $\hmomentavg$ is $(\alpha, \beta, n,n')$-mean-conditioned moment multicalibrated with respect to $\cG$ over $\cD$, it's sufficient to prove the following for every $i\in [n]$, $j \in [n']$, and $G \in \cG$:
\begin{enumerate}
    \item Mean Consistency 
        \begin{align*}
            &\left\vert \E_{(x,y) \sim \cD, \hmeanavg}\left[ \ind[\hmomentavg(x) \in B(i,j), G(x) = 1] \cdot \left(y-\hmomentavg(x)[0] \right)\right] \right\vert\\
            &=  \left\vert \sum_{(x,y)} \sum_{t=1}^T \cD[(x,y)] \cdot \Pr[\hmomentavg = h_t]  \cdot \Pr_{h_t}[h_t(x) \in B(i,j)] \cdot \ind[G(x) = 1] \cdot \left( y-\hbmu_t(x)\right) \right\vert \\
            &=  \frac{1}{T} \left\vert \sum_{(x,y)} \sum_{t=1}^T \cD[(x,y)] \cdot \Pr_{h_t}[h_t(x) \in B(i,j)] \cdot \ind[G(x) = 1] \cdot \left( y- \hbmu_t(x)\right)  \right\vert \\
            &=\frac{1}{T} \left\vert \sum_{t=1}^T  \E_{(x, y) \sim \cD, h_t}\left[\ind[h_t(x) \in B(i,j), G(x) = 1]\cdot \left( y-\hbmu_t(x)\right) \right]\right\vert \\
            &\le \alpha
        \end{align*}
    \item Moment Consistency
        \begin{align*}
            &\left\vert \E_{(x,y) \sim \cD, \hmeanavg}\left[ \ind[\hmomentavg(x) \in B(i,j), G(x) = 1] \cdot \left((y - A^G_{i,j})^k-\hmomentavg(x)[1] \right)\right] \right\vert \\
            &=  \left\vert \sum_{(x,y)} \sum_{t=1}^T \cD[(x,y)] \cdot \Pr[\hmomentavg = h_t]  \cdot \Pr_{h_t}[h_t(x) \in B(i,j)] \cdot \ind[G(x) = 1] \cdot \left( (y - A^G_{i,j})^k - \hbmk_t(x)\right) \right\vert \\
            &=  \frac{1}{T} \left\vert \sum_{(x,y)} \sum_{t=1}^T \cD[(x,y)] \cdot \Pr_{h_t}[h_t(x) \in B(i,j)] \cdot \ind[G(x) = 1] \cdot \left( (y-A^G_{i,j})^k - \hbmk_t(x)\right)  \right\vert \\
            &=\frac{1}{T} \left\vert \sum_{t=1}^T  \E_{(x, y) \sim \cD, h_t}\left[\ind[h_t(x) \in B(i,j), G(x) = 1]\cdot \left( (y-A^G_{i,j})^k -\hbmk_t(x)\right) \right] \right\vert \\
            &\le \beta,
        \end{align*}
        where $A^G_{i,j}$ is the true conditional mean for $G(\bmu, i, \bmk, j)$:
        \begin{align*}
            A^G_{i,j} &= \E_{(x,y), \hmomentavg}\left[ \ind\left[ \hmomentavg(x) \in B(i,j), G(x) = 1\right] \cdot y \right]\\
            &= \frac{1}{T} \sum_{t=1}^T \E_{(x,y), h_t}\left[\ind\left[ h_t(x) \in B(i,j), G(x) = 1\right] \cdot y\right]
        \end{align*}
\end{enumerate}

As for mean consistency, the same approach works as in the proof of Theorem \ref{thm:batch-mean}. 
\begin{lemma}
    With probability $1-\lambda$ over the randomness of $\pi_T$, $\{(\bmu_t, \bmk_t)\}$, Algorithm \ref{alg:online-batch-moment} produces $\{h_t\}_{t=1}^T$ such that for every $i \in [n]$, $j \in [n']$, and $G \in \cG$
\begin{align*}
    \frac{1}{T} \left\vert \sum_{t=1}^T \E_{(x, y)\sim \cD, h_t}\left[\ind\left[h_t(x) \in B(i,j) , G(x)=1\right] \cdot \left(y - \hbmu_t(x) \right)\right] \right\vert \le \left( 4 +\epsilon\right) \cdotk \sqrt{\frac{2}{T} \ln\left(\frac{4|\cG|n \cdot n'}{\lambda} \right)}
\end{align*}
\end{lemma}
\begin{proof}
Fix $i \in[n], j\in[n']$ and $G \in \cG$, and consider the following martingale sequence adapted to the filtration $\mathcal{F}_s = \sigma(\{(x_t, y_t), h_t\}_{t=1}^s)$:
\begin{align*}
    \tZ_s = Z_{s-1} + \E_{(x, y) \sim \cD, h_t}\left[\ind\left[h_s(x) \in B(i,j) , G(x)=1\right] \cdot \left(y - \hbmu_s(x)\right)\right] - \ind\left[\bmu_s \in B(i,j) , G(x_s)=1\right] \cdot \left(y_s - \bmu_s \right).
\end{align*}
Applying Azuma's inequality (Lemma \ref{lem:azuma}) gives us that with probability $1-\lambda/2$ over the randomness of drawing $D$ from $\cD$ and $\pi_T$,
\begin{align*}
    &\frac{1}{T} \left\vert \sum_{t=1}^T \E_{(x, y)\sim \cD}\left[\ind\left[h_t(x) \in B(i,j) , G(x)=1\right] \cdot \left(y - \hbmu_t(x) \right)\right] \right\vert \\
     &\le \frac{1}{T} \left\vert \sum_{t=1}^T\ind\left[(\bmu_t, \bmk_t) \in B(i,j) , G(x_t)=1\right] \cdot \left(y_t - \bmu_t \right)\right\vert + \sqrt{\frac{8 \ln\left(\frac{4}{\lambda}\right)}{T}}
\end{align*}

Now, applying Theorem \ref{thm:momentalg-efficient} with failure probability $\frac{\lambda}{2}$ and union bounding the above azuma's inequality over every $i \in [n]$, $j \in [n']$ and $G \in \cG$ gives us the result: we have that with probability $1-\lambda$ over $\pi_T$ and $\cD$, 
\begin{align*}
    &\frac{1}{T} \left\vert \sum_{t=1}^T \E_{(x, y)\sim \cD, h_t}\left[\ind\left[h_t(x) \in B(i,j) , G(x)=1\right] \cdot \left(y - \hbmu_t(x) \right)\right] \right\vert \\
    &\le \left(4 +\epsilon' \right) \cdotk \sqrt{\frac{2}{T} \ln\left(\frac{8|\cG|n \cdot n'}{\lambda} \right) + 2\epsilon} + \sqrt{\frac{8 \ln\left(\frac{4|\cG|n\cdot n'}{\lambda}\right)}{T}}\\
    &\le \left(6 +\epsilon' \right) \cdotk \sqrt{\frac{2}{T} \ln\left(\frac{8|\cG|n \cdot n'}{\lambda} \right) + 2\epsilon} 
\end{align*}
for every $i \in [n]$, $j \in [n']$ and $G \in \cG$.
\end{proof}

As for the moment consistency, due to higher moments' non-linearity, we need an additional application of Azuma's inequality to show that the empirical conditional mean and the true conditional mean, denoted as $A$ above, must be similar. This is to handle the fact that the empirical moment is centered around the empirical mean but the true moment is centered around the true mean. 

For convenience, we write 
\[
    {A'}^G_{i,j} = \frac{1}{T} \sum_{t=1}^T \ind\left[ (\bmu_t, \bmk_t) \in B(i,j), G(x_t) = 1\right] \cdot y_t
\]
to denote the empirical conditional mean. 
\begin{lemma} 
\label{lem:emp-true-mean-batch}
Fix $i \in [n]$, $j \in [n']$, and $G \in \cG$.
    With probability $1-\lambda$ over the randomness of drawing $D$ from $\cD$ and $\pi_T$, we have
    \begin{align*}
        |{A}^G_{i,j} - {A'}^G_{i,j}| \le \sqrt{\frac{2 \ln\left(\frac{2}{\lambda}\right)}{T}}
    \end{align*}
\end{lemma}
\begin{proof}
Consider the following martingale sequence once again adapted to the filtration $\mathcal{F}_s = \sigma(\{(x_t,y_t)\}_{t=1}^s)$:
\begin{align*}
    \tZ_s = Z_{s-1} + \ind[(\bmu_s, \bmk_s) \in B(i,j), G(x_s) = 1] \cdot y_s - \E_{(x,y), h_s}[\ind[h_s(x) \in B(i,j), G(x) = 1] \cdot y]. 
\end{align*}

Applying Azuma's inequality (Lemma \ref{lem:azuma}) to the above martingale gives us the result.
\end{proof}

Finally, we show that the true and empirical conditional moments when centered around $A^{G}_{i,j}$ must be close through Azuma's inequality. 
\begin{lemma}
\label{lem:batch-moment-lem-true-mean}
Fix $i \in [n]$, $j \in [n']$, and $G \in \cG$. With probability $1-\lambda$ over the randomness of drawing $D$ from $\cD$ and $\pi_T$, we have 
    \begin{align*}
    &\Bigg\vert \frac{1}{T} \sum_{t=1}^T \E_{(x, y) \sim \cD, h_t}\left[\ind[h_t(x) \in B(i,j), G(x) = 1]\cdot \left( (y-A^{G}_{i,j})^k -\hbmk_t(x)\right) \right]  \\
    &- \frac{1}{T} \sum_{t=1}^T \ind[(\bmu_t, \bmk_t) \in B(i,j), G(x_t) = 1]\cdot \left( (y_t - A^{G}_{i,j})^k - \bmk_t\right) \Bigg\vert \\
    &\le \sqrt{\frac{8 \ln\left(\frac{2}{\lambda}\right)}{T}}
\end{align*}
\end{lemma}
\begin{proof}
    Consider the following martingale sequence adapted to the filtration $\mathcal{F}_s = \sigma(\{(x_t,y_t), h_t\}_{t=1}^s)$: \begin{align*}
        &\tZ_s = Z_{s-1} \\
    &+ \E_{(x, y) \sim \cD, h_s}\left[\ind\left[h_s(x) \in B(i,j) , G(x)=1\right] \cdot \left((y-A)^k - \hbmk_s(x)\right)\right] \\
    &- \ind\left[(\bmu_s,\bmk_s) \in B(i,j) , G(x_s)=1\right] \cdot \left((y_s-A)^k - \bmk_s \right). 
\end{align*}

Applying Azuma's to the above martingale gives us the result.
\end{proof}

Note that because $\{\bmu_t, \bmk_t\}_{t=1}^T$ is $(\alpha,\beta, n, n')$-mean-conditioned-moment mutlticalibrated with respect to $\cG$, we have
\begin{align*}
    \left\vert \frac{1}{T} \sum_{t=1}^T  \ind[(\bmu_t, \bmk_t) \in B(i,j), G(x_t) = 1]\cdot \left( (y_t-{A'}^G_{i,j})^k -\bmk_t\right) \right\vert \le \beta.
\end{align*}

Therefore, by union bounding over every $i \in[n]$, $j \in [n']$ and $G \in [\cG]$, we can show with probability $1-\lambda$ that for every $i,j$, and $G$
\begin{align*}
    &\frac{1}{T} \left\vert \sum_{t=1}^T  \E_{(x, y) \sim \cD, h_t}\left[\ind[h_t(x) \in B(i,j), G(x) = 1]\cdot \left( (y-{A}^G_{i,j})^k -\hbmk_t(x)\right) \right] \right\vert \\    
    &\le \frac{1}{T} \left\vert \sum_{t=1}^T  \ind[(\bmu_t, \bmk_t) \in B(i,j), G(x_t) = 1]\cdot \left( (y_t - {A}^G_{i,j})^k -\bmk_t\right)\right\vert + \sqrt{\frac{8 \ln\left(\frac{8|\cG|nn'}{\lambda}\right)}{T}} \tag{with probability $1-\lambda/4$ Lemma \ref{lem:batch-moment-lem-true-mean}}\\
    &\le  \frac{1}{T} \left\vert  \sum_{t=1}^T  \ind[(\bmu_t, \bmk_t) \in B(i,j), G(x_t) = 1]\cdot \left( (y_t-{A'}^G_{i,j})^k -\bmk_t\right) \right\vert + k\sqrt{\frac{2 \ln\left(\frac{8|\cG|nn'}{\lambda}\right)}{T}} + \sqrt{\frac{8 \ln\left(\frac{8|\cG|nn'}{\lambda}\right)}{T}} \tag{with probability $1-\lambda/4$ Lemma \ref{lem:emp-true-mean-batch}}\\
    &\le (k+2)\left(\left(4 +\epsilon' \right) \cdotk \sqrt{\frac{2}{T} \ln\left(\frac{8|\cG|n \cdot n'}{\lambda} \right) + 2\epsilon} \right) + \frac{k}{2n} + k\sqrt{\frac{2 \ln\left(\frac{8|\cG|nn'}{\delta}\right)}{T}} + \sqrt{\frac{8 \ln\left(\frac{8|\cG|nn'}{\delta}\right)}{T}} \tag{with probability $1-\lambda/2$ Theorem \ref{thm:momentalg-efficient}}\\
    &\le (k+3)\left(\left(5 +\epsilon' \right) \cdotk \sqrt{\frac{2}{T} \ln\left(\frac{8|\cG|n \cdot n'}{\lambda} \right) + 2\epsilon} \right) + \frac{k}{2n},
\end{align*}
where the second inequality holds because $|a^k - b^k| \le k|a-b|$ for any $a,b \in [0,1]$ and $T > 2 \ln(\frac{8|\cG|n \cdot n'}{\lambda})$. 

Because the mean consistency holds with probability $1-\lambda$ and the moment consistency holds with probability $1-\lambda$, $\hmomentavg$ is $(\alpha,\beta, n, n')$-mean-conditioned-moment multicalibrated with respect to $\cG$ over $\cD$ with probability $1-2\lambda$.
\end{proof}

\newcommand{\hintervalavg}{h^{\text{interval}}}

\subsubsection{Interval Prediction}
\begin{algorithm}[H]
\label{alg:online-batch-interval}
\SetAlgoLined
\begin{algorithmic}
\STATE \textbf{INPUT}: Training dataset $D = \{(x_t, y_t)\}_{t=1}^T$
\STATE \textbf{Training:} Run Algorithm \ref{alg:intervalalg} on the sequence of examples $D$ to generate a transcript $\pi_T$. 
\STATE Denote by $h_t(x)$ the (randomized) mapping from $\cX$ to $[0,1] \times [0,1]$ that Algorithm \ref{alg:momentalg} induces as a function of transcript $\pi_{t-1}$ (the prefix of $\pi_T$ of length $t-1$). 
\STATE \textbf{Prediction:} On input $x$, sample $\hintervalavg(x)$ by selecting $t \sim [T]$ uniformly at random, and then sampling from $h_t(x)$. 
\STATE \hrulefill 
\STATE More explicitly, select $t \sim [T]$ uniformly at random and:
\STATE Observe $x_t$ and compute  $C_{t-1}^{\ell, u}(x_t)$ for each $(\ell,u) \in \cPinterval^{rn}$ as in \eqref{eqn:def-C-interval} conditioning on $\pi_{t-1}$.
\STATE Solve the Linear Program from Figure \ref{lp:init} using the Ellipsoid algorithm, with Algorithm \ref{alg:intervalseparation} as a separation oracle, to obtain a solution $Q^L_t \in \Delta \cPinterval^{rn}$.
\STATE Predict $\hintervalavg(x) = (\ell, u)$ with probability $Q^L_t((\ell,u))$.
\end{algorithmic}
\caption{Von Neumann's Batch Multivalid Predictor}
\end{algorithm}

\begin{theorem}
Assume that $\cD$ is a $(\rho, rn)$-smooth distribution. Let $D=\{(x_t, y_t)\}_{t=1}^T$ be a dataset drawn i.i.d. from $\cD$. Let $\delta, \lambda > 0$. With probability $1-\lambda$, Algorithm \ref{alg:online-batch-interval} produces a predictor $\hintervalavg$ that is $(\alpha, n)$-multivalid with respect to $\delta$ and $\cG$ over $\cD$ where \[
\alpha = \rho + 6 \cdotk \sqrt{\frac{2}{T} \ln\left(\frac{4|\cG|n^2}{\lambda} \right)  + 2\epsilon}.
\]
\end{theorem}
\begin{proof}
    In order to show that $\hintervalavg$ is $(\alpha, n)$-multivalid with respect to $\delta$ and $\cG$ over $\cD$, it is sufficient to show for all $G \in \cG$ and $i \le j \in [n]$
\begin{align*}
     \left\vert \E_{(x,y) \sim \cD, \hintervalavg}\left[\ind[\hintervalavg(x) \in B(i, j), G(x) = 1] \cdot \left(\mathrm{Cover}(\hintervalavg(x) ,x) - (1-\delta) \right)\right]  \right\vert \le \alpha.
\end{align*}
\newcommand{\cover}{\mathrm{Cover}}
We can calculate:
\begin{align*}
    &\E_{(x,y) \sim \cD, \hintervalavg}\left[\ind[\hintervalavg(x) \in B(i, j), G(x) = 1] \cdot \left(\cover(\hintervalavg(x), y) - (1-\delta) \right)\right]  \\
    &=  \sum_{(x,y)} \sum_{t=1}^T \cD[(x,y)] \cdot \Pr[\hintervalavg = h_t]  \cdot \Pr_{h_t}[h_t(x) \in B(i,j)] \cdot \ind[G(x) = 1] \cdot \left(\cover(h_t(x), y)-(1-\delta)\right) \\
    &=  \frac{1}{T} \sum_{(x,y)} \sum_{t=1}^T \cD[(x,y)] \cdot \Pr_{h_t}[h_t(x) \in B(i,j)] \cdot \ind[G(x) = 1] \cdot \left(\cover(h_t(x), y)-(1-\delta)\right) \\
    &=\frac{1}{T} \sum_{t=1}^T  \E_{(x, y) \sim \cD, h_t}\left[\ind[h_t(x) \in B(i,j), G(x) = 1]\cdot \left(\cover(h_t(x), y)-(1-\delta)\right)\right]  
\end{align*} 

Consider the following martingale sequence adapted to the filtration $\mathcal{F}_s = \sigma(\{(x_t, y_t), \bmu_t\}_{t=1}^s)$:
\begin{align*}
    \tZ_s &= Z_{s-1} + \E_{(x, y) \sim \cD, h_t}\left[\ind[h_s(x) \in B(i,j), G(x) = 1]\cdot \left(\cover(h_s(x), y)-(1-\delta)\right) | \pi_{s-1}\right] \\
    &- \ind[(\bell_s, \bu_s) \in B(i,j), G(x_s) = 1] \cdot \left(\cover((\bell_s, \bu_s), y_s)-(1-\delta)\right).
\end{align*}
Because $|Z_s - Z_{s-1}| \le 2$, we can apply Azuma's inequality (Lemma~\ref{lem:azuma}) to get that with probability $1-\lambda/2$ over the randomness of $\pi_T$ and $D$,
\begin{align*}
     &\Bigg\vert \sum_{t=1}^T  \E_{(x, y) \sim \cD, h_t}\left[\ind[h_t(x) \in B(i,j), G(x) = 1]\cdot \left(\cover(h_t(x), y)-(1-\delta)\right)\right]   \\
     &- \sum_{t=1}^T  \ind[(\bell_t, \bu_t) \in B(i,j), G(x_t) = 1] \cdot \left(\cover((\bell_t, \bu_t), y_t)-(1-\delta)\right) \Bigg\vert \\
     &\le 2\sqrt{2T\ln\left(\frac{4}{\lambda}\right)}.
\end{align*}

Note that from Theorem \ref{thm:interval-efficient} that with probability $1-\lambda/2$ over the randomness of $\pi_T$ produced in training that for all $i \le j \in [n], G \in \cG$:
\begin{align*}
    \left| \frac{1}{T} \sum_{t=1}^T  \ind[(\bell_t, \bu_t) \in B(i,j), G(x_t) = 1] \cdot \left(\cover((\bell_t, \bu_t), y_t)-(1-\delta)\right)\right| \le \rho + 4 \cdotk \sqrt{\frac{2}{T} \ln\left(\frac{4|\cG|n^2}{\lambda} \right)  + 2\epsilon}.
\end{align*}

Therefore, taking the union bound for the above Azuma's inequality over all $i \le j \in [n], G \in \cG$, we have that with probability $1-\lambda$,
\begin{align*}
    \left\vert \frac{1}{T}\sum_{t=1}^T  \E_{(x, y) \sim \cD, h_t}\left[\ind[h_t(x) \in B(i,j), G(x) = 1]\cdot \left(\cover(h_t(x), y)-(1-\delta)\right)\right] \right\vert \le \rho + 6 \cdotk \sqrt{\frac{2}{T} \ln\left(\frac{4|\cG|n^2}{\lambda} \right)  + 2\epsilon}
\end{align*}
for every $i \le j \in [n], G \in \cG$.

\end{proof}

\section{Unboundedly Many Groups, Bounded Group Membership}\label{sec:moregroups}
In this section, we briefly sketch how we can modify our results so that we can handle the case that there are a ``large number'' of groups (i.e. $|\cG|$ is infinite or  larger than $2^T$ --- a range in which the bounds we prove in the main body are vacuous). In this scenario, we maintain the assumption that any given $x \in \cX$ appears in at most $d$ groups, i.e. that $|\cG(x)| \leq d$ for all $x \in \cX$. As we have already noted, in this scenario, our running time dependence on $|\cG|$ can be replaced with $d$ --- here we show that we can do the same in our convergence bounds.

The first step is to redefine our surrogate loss function $L$. The way it was previously defined, $L_0$ was already a quantity at the scale of $|\cG|$, and so it would be hopeless to use it for infinite collections of groups. But a small modification solves this problem:
\begin{definition}[Surrogate loss function]
\label{def:meansurrogate-group}
Fixing a transcript $\pi_s \in \Pi^*$ and a parameter $\eta \in [0,\frac12]$, define a surrogate calibration loss function at day $s$ as:
$$L_s(\pi_s) = 1+ \sum_{\substack{G \in \cG,\\i \in [n]}}\left(\exp(\eta V_s^{G,i}) + \exp(-\eta V_s^{G,i}) -2 \right).$$
When the transcript $\pi_s$ is clear from context, we will sometimes simply write $L_s$. 
\end{definition}

Observe that this modified function  satisfies $L_0 = 1$, independently of the size of $|\cG|$, and still allows us to tightly upper bound our calibration loss:
\begin{obs}\label{obs:bounding-group} For any transcript $\pi_T$, and any $\eta \in [0,\frac12]$, we have that:
$$\max_{G \in \cG, i \in [n]}\left| V_T^{G,i}\right|\leq \frac{1}{\eta}\ln(L_T + 2dT) \leq \max_{G \in \cG, i \in [n]}\left| V_T^{G,i}\right| + \frac{\ln\left(dT \right)}{\eta}.$$
\end{obs}
This observation uses the fact that because (by assumption) $|\cG(x_t)| \leq d$ for all $t$, after $T$ time steps, there are at most $dT$ quantities $V_T^{G,i}$ that are non-zero.

We can now provide a modified bound on $\Delta_{s+1}(\pi_s,x_{s+1}, \bmu_{s+1})$:
\begin{lemma}
\label{lem:boundincrease-group}
For any transcript $\pi_s \in \Pi^*$, any $x_{s+1} \in \cX$, and any $\bmu_{s+1} \in \cPmean$ such that $\bmu_{s+1} \in B(i)$ for some $i \in [n]$:
    \[
    \Delta_{s+1}(\pi_s,x_{s+1}, \bmu_{s+1}) \leq \eta \left(\E_{\tilde{y}_{s+1}}[\tilde{y}_{s+1}]-\bmu_{s+1}\right) C_s^i(x_{s+1}) + 2\eta^2\cdotk  L_{s} + 4d\eta^2,
    \]
where for each $i \in [n]$:
\begin{align*}
C_s^{i}(x_{s+1})  \equiv \sum_{\cG(x_{s+1})}\exp(\eta V_s^{G,i}) -  \exp(-\eta V_s^{G,i}).
\end{align*}
\end{lemma}
\begin{proof}
Fix any transcript $\pi_s \in \Pi^*$ (which defines $L_s$), feature vector $x_{s+1} \in \cX$, and $\bmu_{s+1}$ such that $\bmu_{s+1} \in B(i)$ for some $i \in [n]$. By direct calculation, we obtain:
\begin{align*}
&\,\,\;\Delta_{s+1}(\pi_s,x_{s+1}, \bmu_{s+1})\\ 
=& \E_{\ty_{s+1}}\left[\sum_{G \in \cG(x_{s+1})}\exp(\eta V_s^{G,i})\left(\exp(\eta (\ty_{s+1}-\bmu_{s+1}))-1\right) + \exp(-\eta V_s^{G,i})\left(\exp(-\eta (\ty_{s+1}-\bmu_{s+1}))-1\right)\right], \\
\leq& \E_{\ty_{s+1}}\left[\sum_{G \in \cG( x_{s+1})}\exp(\eta V_s^{G,i})\left(\eta (\ty_{s+1}-\bmu_{s+1})+2\eta^2\right) + \exp(-\eta V_s^{G,i})\left(-\eta (\ty_{s+1}-\bmu_{s+1})+2\eta^2\right)\right], \\
=& \eta \left(\E_{\ty_{s+1}}[\ty_{s+1}]-\bmu_{s+1}\right) \sum_{G \in \cG(x_{s+1})} \left(\exp(\eta V_s^{G,i}) \!-\!  \exp(-\eta V_s^{G,i})\right) + 2\eta^2 \!\! \sum_{G \in \cG(x_{s+1})}\left(\exp(\eta V_s^{G,i}) + \exp(-\eta V_s^{G,i}) \right), \\
\leq& \eta \left(\E_{\ty_{s+1}}[\ty_{s+1}]-\bmu_{s+1}\right)\cdotk \left(\sum_{G \in \cG(x_{s+1})}\exp(\eta V_s^{G,i}) -  \exp(-\eta V_s^{G,i})\right) + 2\eta^2\cdotk  L_{s} + 4d \eta^2,\\
=&\eta \left(\E_{\ty_{s+1}}[\ty_{s+1}]-\bmu_{s+1}\right) C_s^i(x_{s+1})  + 2\eta^2\cdotk  L_{s} + 4d \eta^2.
\end{align*}
Here, the first inequality follows from the fact that for $0 < |x| < \frac{1}{2}$, $\exp(x) \leq 1+x+2x^2$. 
\end{proof}

We can use this to provide a modified bound to Lemma \ref{lem:exists}.

\begin{lemma}
\label{lem:exists-group}
For any transcript $\pi_s \in \Pi^*$, any $x_{s+1} \in \cX$, and any $r \in \mathbb{N}$ there exists a distribution over predictions for the learner $Q^L_{s+1} \in \Delta \cP^{rn}$, such that regardless of the adversary's choice of distribution of $y_{s+1}$ over $\Delta \cY$,  we have that:
\[
    \E_{\bmu \sim Q^L_{s+1}}\left[\Delta_{s+1}(\pi_s,x_{s+1}, \bmu)\right] \le L_s\cdotk \left(\frac{\eta}{rn}+2\eta^2\right) +2d.
\]
\end{lemma}
\begin{proof}
As in the proof of Lemma \ref{lem:exists}, we construct a zero-sum game between the learner and the adversary. Fix the transcript $\pi_s$ and the feature vector $x_{s+1}$. We define the utility of this game to be the upper bound we proved on $ \Delta_{s+1}(\pi_s,x_{s+1}, \bmu)$ in Lemma~\ref{lem:boundincrease-group}. For each $\bmu \in \cP^{rn}$ and each $y \in [0,1]$, we let:
\[
    u(\bmu, y) = \eta \left(y-\bmu\right) C^{\bmu}_s(x_{s+1})   + 2\eta^2\cdotk  L_{s} + 4d\eta^2.
\]

We now establish the value of this game. Observe that for any strategy of the adversary (which fixes $\E[\ty])$, the learner can respond by playing $\bmu^* = \argmin_{\bmu \in \cP^{rn}} |\E[\ty] - \bmu|$, and that because of our discretization, $\min |\E[\ty] - \bmu^*| \leq \frac{1}{rn}$. Therefore, the value of the game is at most: 
\begin{eqnarray*}
\max_{y \in [0,1]} \min_{\bmu^* \in \cP^{rn}} u(\bmu^*,y) &\leq& \max_{\bmu \in \cP^{rn}} \frac{\eta}{rn} \left|C^{\bmu}_s(x_{s+1}) \right| + 2\eta^2\cdotk  L_{s}  + 4d\eta^2, \\
&\leq& L_s\cdotk \left(\frac{\eta}{rn} + 2\eta^2\right) +2d.
\end{eqnarray*}
Here the latter inequality follows since $C^{\bmu}_s(x_{s+1})  \leq L_s + 2d$ for all $\bmu \in \cP^{rn}$, by observation, and then since $\eta \in (0,\frac12)$ we have the bound.  We can now apply the minimax theorem (Theorem~\ref{thm:minimax}) to conclude that there exists a fixed distribution $Q^L_{s+1} \in \mathcal{Q}^L$ for the learner that guarantees that simultaneously for \emph{every} label $y \in [0,1]$ that might be chosen by the adversary:
$$\E_{\bmu \sim Q^L_{s+1}}\left[u(\bmu,y)\right] \leq L_s\cdotk \left(\frac{\eta}{rn}+2\eta^2\right) +2d,$$
as desired.
\end{proof}

\begin{corollary}\label{cor:exists-mean-group}
For every $r \in \mathbb{N}$, $s \in [T]$, $\pi_s \in \Pi^*$, and $x_{s+1} \in \cX$ (which fixes $L_s$ and $Q^L_{s+1}$), and any distribution over $\cY$:

\[
    \E_{\bmu^L_{s+1} \sim Q_{s+1}}[\tL_{s+1}|\pi_s] = L_s + \E_{\bmu_{s+1}\sim Q^L_{s+1}}[\Delta_{s+1}(\pi_{s+1},x_{s+1},\bmu_{s+1})] \leq L_s\cdotk\left(1 + \frac{\eta}{rn} + 2\eta^2\right) + 2d.
\]
\end{corollary}

Lemma~\ref{lem:exists-group} shows that playing the minimax strategy of this zero-sum game (Algorithm \ref{alg:meanexistential}) continues to provide a low value to the learner. We now show the counterpart of the first part of Theorem \ref{thm:general-bounds} for these modified bounds:
\begin{theorem}\label{thm:general-bounds-group}
Consider a nonnegative random process $\tX_t$ adapted to the filtration $\mathcal{F}_t= \sigma(\pi_t)$, where $\tX_0$ is constant a.s.
Suppose we have that for any period $t,$ and any $\pi_{t-1}$, $\E[\tX_t|\pi_{t-1}] \leq X_{t-1} (1 + \eta c + 2\eta^2) +2d$ for some $\eta \in [0,\frac12], c \in [0,1], d>0$.
Then we have that: 
\begin{equation}
    \E_{\tpi_T}[\tX_T] \leq (X_0+2dT) \exp \left(T\eta c+ 2T\eta^2 \right). \label{eqn:expectedbound-group}
\end{equation}
\end{theorem}

\begin{proof}
First, observe that:
\begin{align*}
\E_{\tpi_T}[\tX_T] 
=& \E_{\tpi_{T-1}}\left[ \E[\tX_T | \pi_{T-1}]\right],\\
\le& \E_{\tpi_{T-1}}\left[ \E[\left(1 + \eta c + 2\eta^2\right) X_{T-1} +2d | \pi_{T-1}] \right]\\
=& \left(1 + \eta c + 2\eta^2\right) \E_{\tpi_{T-1}}\left[ \tX_{T-1}  \right]  + 2d,\\
&\;\;\vdots \\
\leq& X_0 \left(1 + \eta c  + 2\eta^2\right)^T + 2d\sum_{t=0}^{T-1} (1+c\eta + 2 \eta^2)^t, \\
\leq& X_0 \left(1 + \eta c  + 2\eta^2\right)^T +2dT (1+c\eta + 2 \eta^2)^T,\\
=& (X_0 +2dT) \exp\left(T\ln\left(1 + \eta c  + 2\eta^2\right) \right) ,\\
\leq& (X_0 +2dT) \exp\left(T \eta c + 2T\eta^2\right),
\end{align*}
where the last inequality holds because $\ln(1+x) \leq x$ for any $x > -1$. This concludes the proof of \eqref{eqn:expectedbound-group}.
\end{proof}

We are now ready to bound our multicalibration error. As a straightforward consequence of Corollary~\ref{cor:exists-mean-group} and  Theorem~\ref{thm:general-bounds-group}, we have the following Corollary. 
\begin{corollary}
\label{cor:surrogateloss-mean-group}
Against any adversary, Algorithm~\ref{alg:meanexistential} instantiated with discretization parameter $r$ results in surrogate loss satisfying:
\[
    \E_{\tpi_T}[\tL_T] \le (1+2dT)\cdotk \exp\left(\frac{T\eta}{rn}+2T\eta^2\right).
\]
\end{corollary}
\begin{proof}
Note that the first part of Theorem~\ref{thm:general-bounds-group} applies to the process $L$ with $L_0 =1$ and $c = \frac{1}{rn}$. The bound follows by plugging these values into \eqref{eqn:expectedbound-group}. 
\end{proof}

Next, we can convert this into a bound on Algorithm~\ref{alg:meanexistential}'s expected calibration error:

\begin{theorem}
\label{thm:meanmulti-group}
When Algorithm~\ref{alg:meanexistential} is run using $n$ buckets for calibration, discretization $r \in \mathbb{N}$, and $\eta = \sqrt{\frac{\ln(1+2dT)}{2T}}$, then against any adversary, its sequence of mean predictions are $(\alpha,n)$-multicalibrated with respect to $\cG$, where:
\begin{equation*}
    \E[\alpha] \leq \frac{1}{rn} + 2\cdotk \sqrt{\frac{2\ln(1+4dT)}{T}}.
\end{equation*}
For $r = \frac{\sqrt{T}}{\epsilon n\sqrt{2\ln(1+4dT})}$ this gives:
$$\E[\alpha] \leq \left( 2+\epsilon\right) \cdotk \sqrt{\frac{2}{T} \ln\left(1+4dT \right)}.$$
Here the expectation is taken over the randomness of the transcript $\pi_T$.
\end{theorem}
\begin{proof}
From Observation~\ref{obs:meancalibration}, it suffices to show that $$\frac{1}{T} \E_{\tpi_T}\left[\max_{G \in \cG, i \in [n]}|\tV_T^{G,i}|\right] \leq \frac{1}{rn} + 2\cdotk \sqrt{\frac{2\ln(1+4dT)}{T}}.$$ 

We begin by computing a bound on the (exponential of) the expectation of this quantity:
\begin{eqnarray*}
\exp\left(\eta \cdotk \E_{\tpi_T}\left[\max_{G,i}|\tV_T^{G,i}|\right]\right)&\leq&\E_{\tpi_T}\left[\exp\left(\eta\cdotk \max_{G,i}|\tV_T^{G,i}|\right)\right], \\
&=&\E_{\tpi_T}\left[\max_{G,i}\exp\left(\eta\cdotk|\tV_T^{G,i}|\right)\right], \\
&\leq& \E_{\tpi_T}\left[\max_{G,i}\left(\exp\left(\eta\cdotk \tV_T^{G,i}\right)+\exp\left(-\eta\cdotk \tV_T^{G,i}\right)\right)\right], \\
&\leq& \E_{\tpi_T}\left[\sum_{\substack{G,i\\G_T(i) \neq \phi}}\left(\exp\left(\eta\cdotk \tV_T^{G,i}\right)+\exp\left(-\eta\cdotk \tV_T^{G,i}\right)\right)\right] ,\\
&=& \E_{\tpi_T}[\tL_T+2dT], \\
&\leq& (1+2dT)\cdotk \exp\left(\frac{T\eta}{rn}+2T\eta^2\right) +2dT,\\
&\leq& (1+4dT)\cdotk \exp\left(\frac{T\eta}{rn}+2T\eta^2\right).
\end{eqnarray*}
Here the first step is by Jensen's inequality and the second last one follows from Corollary~\ref{cor:surrogateloss-mean-group}. Taking the logarithm of both sides and dividing by $\eta\cdotk T$, we have
$$\frac{1}{T}\E_{\tpi_T}\left[\max_{G,i}|\tV_T^{G,i}|\right] \leq \frac{\ln(1+4dT)}{\eta\cdotk T} + \frac{1}{rn} + 2\eta.$$
Choosing $\eta = \sqrt{\frac{\ln(1+4dT)}{2T}}$, we thus obtain the desired inequality
$$ \frac{1}{T}\E_{\tpi_T}\left[\max_{G,i}|\tV_T^{G,i}|\right]\leq \frac{1}{rn} + 2\cdotk \sqrt{\frac{2\ln(1+4dT)}{T}}.$$
\end{proof}

The corresponding high-probability bounds are omitted for brevity. They have the analogous dependence on $dT$ replacing $|\cG|$. Similar bounds can be obtained for the case of moment-multicalibration and multivalid intervals with the same approach.

\section{Mean Conditioned Moment Multicalibrators Can Randomize Over Small Support}
\label{sec:momentellipsoid}

In Section \ref{sec:momentalg}, we derived a linear programming based algorithm for making mean conditioned moment multicalibrated predictors. Although we proved that we could reduce the pure strategy space of the learner from $(r^2nn')$ to $4nn'$, a priori, the solutions we find via linear programming could have full support. Here we prove that this need not be the case --- there always exists a basic feasible solution of the linear program that we solve that has support only over $k+1$ pure strategies for the learner.

\begin{lemma}\label{lem:minimaxstructure}
For any game with objective function (\ref{eqn:objfn-moment}), there exists a minimax strategy for the learner $\hat{Q}^L \in \learnerstrat$, such that
$|\text{support}(\hat{Q}^L)| \leq k+1$. 
\end{lemma}

\begin{proof}
Suppose that $Q^*$ is a minimax strategy for the learner. 

Observe that the adversary's best response in this problem is straightforward: we have that $\psi_\ell = 1$ if $\sum_{\bmu,\bmk} F_\ell^{\bmu,\bmk} Q^*(\bmu,\bmk) >0$, that $\psi_\ell = 0$ if $\sum_{\bmu,\bmk} F_\ell^{\bmu,\bmk} Q^*(\bmu,\bmk) <0$, and otherwise the adversary is indifferent. Define 
\begin{align*}
L_+ &= \{ \ell \in [k]:   \sum_{\bmu,\bmk} F_\ell^{\bmu,\bmk} Q^*(\bmu,\bmk) >0 \},\\
L_- &= \{ \ell \in [k]:   \sum_{\bmu,\bmk} F_\ell^{\bmu,\bmk} Q^*(\bmu,\bmk) <0 \},\\
L_= &= \{ \ell \in [k]:   \sum_{\bmu,\bmk} F_\ell^{\bmu,\bmk} Q^*(\bmu,\bmk) =0 \}.
\end{align*}
Note that $L_+ \cup L_- \cup L_= = [k]$.

Since $Q^*$ is a minimax strategy, it must solve the following linear program, which corresponds to minimizing the learner's objective value over all strategies $Q$ which engender the same best response for the adversary as $Q^*$:
\begin{align*}
    \min_{Q \in \learnerstrat} &\sum_{\bmu,\bmk} Q(\bmu, \bmk) \left( \bmu C^{\bmu,\bmk}_s + \bmk D^{\bmu,\bmk}_s - \centerbucket^k D^{\bmu,\bmk}_s \right)\\
    \text{subject to:}\\
    \forall \ell \in L_+: & \sum_{\bmu,\bmk} F_\ell^{\bmu,\bmk} Q(\bmu,\bmk) \geq 0,\\
    \forall \ell \in L_-: & \sum_{\bmu,\bmk} F_\ell^{\bmu,\bmk} Q(\bmu,\bmk) \leq 0,\\
    \forall \ell \in L_{=}: & \sum_{\bmu,\bmk} F_\ell^{\bmu,\bmk} Q(\bmu,\bmk) =0,\\
    & \sum_{\bmu, \bmk} Q(\bmu, \bmk) =1,\\
    & Q \geq 0. 
\end{align*}
Further, any solution to this LP must also be a minimax strategy for the learner. Observe that this has $k + 1$ linear constraints. Any such linear program has a basic feasible solution: so there exists a solution $\hat{Q}^L$ (viewed as a vector) with exactly the number of non-zero entries as the number of binding constraints, i.e. $\leq k+1$, as desired.%
\footnote{As an aside, we point out that this also implies the square submatrix with rows corresponding to binding constraints and corresponding to non-zero variables is of full rank. Textbook treatments that we are aware of consider either LPs with all inequality constraints or all equality constraints. So for completeness we include the following argument. Convert the LP above into a LP in standard form $\min c^T x \text{ s.t. } Ax =b, x\geq 0$ by adding/subtracting non-negative slack variables to the inequality constraints $L_+, L_-$. This is a system of $k+1$ linear equality constraints in $4 n n' + |L_-| + |L_+| +1$ variables. We know that there exists an optimal of this LP that is a Basic feasible solution (BFS) (see e.g. Theorem 4.7 of \cite{vohra2004advanced}), i.e. an optimal solution with exactly $k+1$ non-zero variable with the corresponding $(k+1) \times (k+1)$ sub-matrix of $A$, denoted $\hat{A}$, of full rank.  By observation, the number of non-zero $Q$'s in this BFS must equal the number of constraints that bind at equality in the original LP (any non-zero slack variable will correspond to a slack constraint in the original). The sub-matrix of $\bar{A}$ corresponding to the non-zero $Q$'s as columns and  binding constraints of the original LP as rows  must be of full rank, because these rows have all $0$'s in the columns corresponding to the slack variables in $\bar{A}$.} 
This is exactly the statement of the Lemma.
\end{proof}

\section{Proofs from Section \ref{sec:onlinemeanmulti}}
\label{app3}

\helperthm*
\begin{proof}
First, observe that:
\begin{align*}
\E_{\tpi_T}[\tX_T] 
=& \E_{\tpi_{T-1}}\left[ \E[\tX_T | \pi_{T-1}]\right],\\
\le& \E_{\tpi_{T-1}}\left[ \E[\left(1 + \eta c + 2\eta^2\right) X_{T-1} | \pi_{T-1}]\right]\\
=& \left(1 + \eta c + 2\eta^2\right) \E_{\tpi_{T-1}}\left[ \tX_{T-1}  \right],\\
&\;\;\vdots \\
\leq& X_0 \left(1 + \eta c  + 2\eta^2\right)^T, \\
=& X_0 \exp\left(T\ln\left(1 + \eta c  + 2\eta^2\right) \right) ,\\
\leq& X_0 \exp\left(T \eta c + 2T\eta^2\right),
\end{align*}
where the last inequality holds because $\ln(1+x) \leq x$ for any $x > -1$. This concludes the proof of \eqref{eqn:expectedbound}.

Towards demonstrating the high-probability bound~\ref{eqn:hpbound}, we first show the following statement.
\begin{lemma}
\label{lem:bounded-expected-increase-main}
For any  $\pi_T$, we have 
\[
        \sum_{t=1}^T \left(\E_{\tilde{\pi_t}}\left[ \ln( \tX_{t} )\middle| \pi_{t-1}\right]  - \ln(X_{t-1}(\pi_{t-1})) \right) \le T \left(\eta c + 2\eta^2\right).
    \]    
\end{lemma}
\begin{proof}
Fixing $\pi_T$ and taking any $t \leq T$, we have 
\begin{align*}
    \E_{\tilde{\pi}_t}\left[ \ln( \tX_{t} ) |\pi_{t-1}\right] &\le \ln\left(\E_{\tilde{\pi}_t}[ \tX_{t}|\pi_{t-1} ]\right), &\text{(Jensen's inequality)}\\
    &\le \ln(X_{t-1}(\pi_{t-1})) + \ln\left(1 + c\eta + 2\eta^2\right), & \text{(by assumption)}\\
    &\le \ln(X_{t-1}(\pi_{t-1})) + \left(c\eta + 2\eta^2\right). &\text{($\ln(1+x) \le x$ for any $ x> -1$)}
\end{align*}
Summing over every round $t \in [T]$ gives us the result.
\end{proof}

Now observe that for any $\pi_{t-1}$, we have $\E[\tZ_t| \pi_{t-1}] = Z_{t-1}$, so the process $\tZ_t$ is a martingale. Further, its increments are bounded by assumption. Recall Azuma's inequality for martingales with bounded increments (see e.g. \cite{dubhashi2009concentration}):
\begin{lemma}[Azuma's Inequality] \label{lem:azuma}
    For any martingale $\{\tZ_{t}\}_{t=1}^T$ with $|Z_{t}- Z_{t-1}| \le c$ a.s., for all $T$ it holds
    \[  
        \Pr\left[\tZ_T - \tZ_0 \ge \epsilon\right ] \le \exp\left( - \frac{\epsilon^2}{2 c^2 T}\right).
    \]
\end{lemma}

By assumption, we may apply Azuma's inequality with $c = 2\eta$, and we obtain
\begin{align*}
&    \Pr_{\tpi_T} \left[  \sum_{t=1}^T \left(\ln(X_t(\pi_t)) - \E_{\tpi_t}[\ln X_t(\tpi_t)|\pi_{t-1}]\right) \geq \epsilon \right] \leq \exp\left(-\frac{\epsilon^2}{8\eta^2 T}\right).
\intertext{
So, with probability $1-\lambda$, it holds that }
&\sum_{t=1}^T \left(\ln(X_t(\pi_t)) - \E_{\tpi_t}[\ln X_t(\tpi_t)|\pi_{t-1}]\right) \leq \eta \sqrt{8 T \ln\left(\frac{1}{\lambda}\right) }\\
\implies& \ln(X_T(\pi_T)) \le \ln(X_0) + \left(\sum_{t=1}^T \E_{\tpi_t}\left[ \ln( X_{t}(\tpi_t)) |\pi_{t-1}\right]  - \ln(X_{t-1}(\pi_{t-1}))\right) + \eta \sqrt{8 T \ln\left(\frac{1}{\lambda}\right) }\\
\implies& \ln(X_T(\pi_T)) \le \ln(X_0 ) + T \left(\eta c + 2\eta^2\right) + \eta \sqrt{8 T \ln\left(\frac{1}{\lambda}\right) },
\end{align*}
where the last inequality follows from Lemma~\ref{lem:bounded-expected-increase-main}.
\end{proof}

\martingalebounded*
\begin{proof} Observe that
\begin{align*}
    |Z_t - Z_{t-1}|
    =& \left|\ln(L_{t}(\pi_t)) - \E\left[ \ln( L_{t}(\tpi_t) ) |\pi_{t-1}\right]  \right|\\
    =& \left|\E \left[\ln\left(\frac{L_{t}(\pi_t)}{L_{t}(\tpi_t)}\right) \middle| \pi_{t-1} \right]\right|
\end{align*}
Note that for any $\pi_t$,
    $$L_t(\pi_{t}) = L_{t-1}(\pi_{t-1}) + \Delta_{t}(\pi_{t-1}, x_t, y_t, \bmu_t)$$
where:    
$$\Delta_{t}(\pi_{t\!-\!1}, x_t, y_t, \bmu_t)  \!=\! \sum_{\cG(x_t)}\exp(\eta V_{t-1}^{G,\Binv(\bmu_t)})\left(\exp(\eta (y_{t}-\bmu_t))\!-\!1\right) + \exp(\!-\eta V_{t-1}^{G,\Binv(\bmu_t)})\left(\exp(-\eta (y_{t}-\bmu_t))-1\right).$$
Since $y_t - \bmu_t$ must lie in $[-1,1]$, we have that:
$$(\exp(-\eta) - 1) \cdotk L_{t-1}(\pi_{t-1}) \le \Delta_{t}(\pi_{t-1}, x_t, y_t,\bmu_t) \le (\exp(\eta) -1) \cdotk L_{t-1}(\pi_{t-1})$$ which implies:
$$\exp(-\eta) \cdotk L_{t-1}(\pi_{t-1}) \le L_t(\pi_t) \le \exp(\eta) \cdotk L_{t-1}(\pi_{t-1}).$$
Hence, for any two transcripts $\pi_t, \pi'_t$ which are equal over the first $t-1$ periods, we have
\begin{align*}
    \left| \ln\left(\frac{L_t(\pi_{t})}{L_t(\pi'_{t})}\right) \right| \le \ln\left(\frac{\exp(\eta)}{\exp(-\eta)}\right) = 2\eta.
\end{align*}
Therefore, $\left|\E \left[\ln\left(\frac{L_{t}(\pi_t)}{L_{t}(\tpi_t)}\right) \middle| \pi_{t-1} \right]\right| \leq 2\eta$ as desired.
\end{proof}

\section{Proofs from Section \ref{sec:onlinemoment}}
\momentmulti*
\begin{proof}

From Observation \ref{obs:momentcalibration}, it suffices to show that: 
\begin{align*}
&\frac{1}{T} \E_{\tpi_T}\left[\max_{G \in \cG, i \in [n], j \in [n']}|\tV_T^{G,i,j}|\right]\leq \frac{1}{r n} + \frac{1}{rn'} + 2\cdotk \sqrt{\frac{2\ln(4|\cG|n\cdot n')}{T}},\\
&\frac{1}{T} \E_{\tpi_T}\left[\max_{G \in \cG, i \in [n], j \in [n']}|\tM_T^{G,i,j}|\right]  \leq \frac{1}{r n} + \frac{1}{rn'} + 2\cdotk \sqrt{\frac{2\ln(4|\cG|n\cdot n')}{T}}.
\end{align*}
We begin by computing a bound on the (exponential of) the expectation of the first quantity:
\begin{eqnarray*}
\exp\left(\eta \cdotk \E_{\tpi_T}[\max_{G,i,j}|\tV_T^{G,i,j}|]\right)&\leq&\E_{\tpi_T}\left[\exp\left(\eta\cdotk \max_{G,i,j}|\tV_T^{G,i,j}|\right)\right], \\
&=&\E_{\tpi_T}\left[\max_{G,i,j}\exp\left(\eta\cdotk|\tV_T^{G,i,j}|\right)\right], \\
&\leq& \E_{\tpi_T}\left[\max_{G,i,j}\left(\exp\left(\eta\cdotk \tV_T^{G,i,j}\right)+\exp\left(-\eta\cdotk \tV_T^{G,i,j}\right)\right)\right], \\
&\leq& \E_{\tpi_T}\left[\sum_{G,i,j}\left(\exp\left(\eta\cdotk \tV_T^{G,i,j}\right)+\exp\left(-\eta\cdotk \tV_T^{G,i,j}\right) + \exp\left(\eta\cdotk \tM_T^{G,i,j}\right)+\exp\left(-\eta\cdotk \tM_T^{G,i,j}\right)\right)\right] ,\\
&=& \E_{\tpi_T}[\tL_T], \\
&\leq& 4|\cG|n \cdot n' \cdot \cdotk \exp\left(\frac{T\eta}{rn}+\frac{T\eta}{rn'}+2T\eta^2\right).
\end{eqnarray*}
Here the first inequality follows from Jensen's inequality and the last one follows from Corollary \ref{cor:surrogateloss-mean-moment}.  Taking the log of both sides and dividing by $\eta\cdotk T$ we obtain
$$\frac{1}{T}\E_{\tpi_T}[\max_{G,i}|\tV_T^{G,i}|] \leq \frac{\ln(4|\cG|n \cdot n')}{\eta\cdotk T} + \frac{1}{rn} + \frac{1}{rn'} + 2\eta.$$
Choosing $\eta = \sqrt{\frac{\ln(4|\cG|n \cdot n')}{2T}}$, we have
$$ \frac{1}{T}\E_{\tpi_T}[\max_{G,i}|\tV_T^{G,i}|]\leq \frac{1}{rn} + \frac{1}{rn'}+ 2\cdotk \sqrt{\frac{2\ln(4|\cG|n\cdot n')}{T}}. $$
Repeating the same steps, we get an identical bound for $\frac{1}{T} \E_{\tpi_T}[\max_{G \in \cG, i \in [n], j \in [n']}|\tM_T^{G,i,j}|]$.\end{proof}

Now, given $\tilde{L}$, define $\tilde{Z}$ analogously to the second part of Theorem \ref{thm:general-bounds}. Next, we can show that the increments of $\tilde{Z}$ thus defined, at any round $t$, can be bounded.
\begin{lemma}
\label{lem:martingale-bounded-moment}
At any round $t \in [T]$ and for any realized transcript $\pi_t$, $|Z_t - Z_{t-1}|\le 2\eta.$
\end{lemma}
\begin{proof}  Observe that
\begin{align*}
    |Z_t - Z_{t-1}|
    =& \left|\ln(L_{t}(\pi_t)) - \E\left[ \ln( L_{t}(\tpi_t) ) |\pi_{t-1}\right]  \right|\\
    =& \left|\E \left[\ln\left(\frac{L_{t}(\pi_t)}{L_{t}(\tpi_t)}\right) \middle| \pi_{t-1} \right]\right|
\end{align*}
Note that for any $\pi_t$,
    $$L_t(\pi_{t}) = L_{t-1}(\pi_{t-1}) + \Delta_{t}(\pi_{t-1}, x_t, y_t, \bmu_t,\bmk_t)$$
where:    
\begin{align*}
&\Delta_{t}(\pi_{t-1}, x_t, y_t, \bmu_t,\bmk_t) \\
=&\hphantom{+} \sum_{\cG(x_t)}\exp(\eta V_{t-1}^{G,\Binv(\bmu_t),\Binv(\bmk_t)})\left(\exp(\eta (y_{t}-\bmu_t))-1\right) + \exp(-\eta V_{t-1}^{G,\Binv(\bmu_t),\Binv(\bmk_t)})\left(\exp(-\eta (y_{t}-\bmu_t))-1\right),\\
& + \sum_{\cG(x_t)}\exp(\eta M_{t-1}^{G,\Binv(\bmu_t),\Binv(\bmk_t)})\left(\exp(\eta ((y_{t}-\centerbucketmut)^k -\bmk_t))-1\right) \\
&+ \exp(-\eta M_{t-1}^{G,\Binv(\bmu_t),\Binv(\bmk_t)})\left(\exp(-\eta ((y_{t}-\centerbucketmut)^k -\bmk_t))-1\right).
\end{align*}
Since $(y_t - \bmu_t)$ and $((y_{t}-\centerbucketmut)^k -\bmk_t)$ must lie in $[-1,1]$, we have that:
$$(\exp(-\eta) - 1) \cdotk L_{t-1}(\pi_{t-1}) \le \Delta_{t}(\pi_{t-1}, x_t, y_t,\bmu_t,\bmk_t) \le (\exp(\eta) -1) \cdotk L_{t-1}(\pi_{t-1})$$ which implies:
$$\exp(-\eta) \cdotk L_{t-1}(\pi_{t-1}) \le L_t(\pi_t) \le \exp(\eta) \cdotk L_{t-1}(\pi_{t-1}).$$
Therefore, for any two $\pi_t, \pi'_t$ such that the corresponding transcripts for the first $t-1$ periods is the same, we have
\begin{align*}
    \left| \ln\left(\frac{L_t(\pi_{t})}{L_t(\pi'_{t})}\right) \right| \le \ln\left(\frac{\exp(\eta)}{\exp(-\eta)}\right) = 2\eta.
\end{align*}
Therefore we have $\left|\E \left[\ln\left(\frac{L_{t}(\pi_t)}{L_{t}(\tpi_t)}\right) \middle| \pi_{t-1} \right]\right| \leq 2\eta$ as desired.
\end{proof}

\momenthp*
\begin{proof} 
By Lemma \ref{lem:martingale-bounded-moment}, the second part of Theorem \ref{thm:general-bounds} applies, and plugging in  $L_0 = 4|\cG|n \cdot n'$ and $c = \frac{1}{rn} + \frac{1}{rn'}$, we have that, with probability $(1-\lambda)$ over the randomness of the transcript:
 \begin{align*}
& \ln(L_T(\pi_T)) \le \ln(4|\cG|n \cdot n) + T \left(\frac{\eta}{rn}  + \frac{\eta}{rn'}+ 2\eta^2\right) + \eta \sqrt{8 T \ln\left(\frac{1}{\lambda}\right) }.
\end{align*}
Now, note that 
\begin{align*}
\exp\left(\eta\cdotk \max_{G,i,j}|V_T^{G,i,j}|\right) &=\max_{G,i,j}\exp\left(\eta\cdotk|V_T^{G,i,j}|\right), \\
&\leq \max_{G,i,j}\left(\exp\left(\eta\cdotk V_T^{G,i,j}\right)+\exp\left(-\eta\cdotk V_T^{G,i,j}\right)\right), \\
&\leq \sum_{G,i,j}\left(\exp\left(\eta\cdotk V_T^{G,i,j}\right)+\exp\left(-\eta\cdotk V_T^{G,i,j}\right) + \exp\left(\eta\cdotk M_T^{G,i,j}\right)+\exp\left(-\eta\cdotk M_T^{G,i,j}\right)\right) , \\
&= L_T(\pi_T). 
\end{align*}
By an analogous argument we have that $\exp\left(\eta\cdotk \max_{G,i,j}|M_T^{G,i,j}|\right) \leq  L_T(\pi_T) $.Taking log on both sides and dividing both sides by $\eta T$, we get
\begin{align*}
    \frac{1}{T} \max_{G,i} |V^{G,i,j}_{T}| \le \frac{1}{\eta T} \ln(L_T(\pi_T))
    \le \frac{\ln(4|\cG|n\cdot n')}{\eta\cdotk T} + \frac{1}{rn}  + \frac{1}{rn'}+ 2\eta + \sqrt{\frac{8  \ln\left(\frac{1}{\lambda}\right)}{T}}.
\end{align*}
Choosing $\eta = \sqrt{\frac{\ln(4|\cG|n \cdot n')}{2T}}$, we obtain:
\begin{align*}
\frac{1}{T} \max_{G,i,j}|V_T^{G,i,j}| &\le \frac{1}{rn} + \frac{1}{rn'}+ 2 \sqrt{\frac{2\ln(4|\cG|n \cdot n')}{T}} + \sqrt{\frac{8  \ln\left(\frac{1}{\lambda}\right)}{T}}\\
&\le \frac{1}{rn} + \frac{1}{rn'} + 4 \cdotk \sqrt{\frac{2}{T} \ln\left(\frac{4|\cG|n \cdot n'}{\lambda} \right)},\\
\intertext{and, by an analogous argument,} 
\frac{1}{T} \max_{G,i,j}|M_T^{G,i,j}| &\leq \frac{1}{rn} + \frac{1}{rn'} + 4 \cdotk \sqrt{\frac{2}{T} \ln\left(\frac{4|\cG|n \cdot n'}{\lambda} \right)},
\end{align*}
as desired.
\end{proof}

\lpapprox*

\begin{proof}
Let $(x^*, \gamma^*)$ be the optimal solution of the original LP. Consider the constraint of the original (resp.\ rational) LP associated with any row $j$ of matrix $A$ (resp.\ $\tilde{A}$). This constraint is written as $\sum_i a_{ji} x_i \leq \gamma$ in the original LP, and $\sum_i \tilde{a}_{ji} x_i \leq \gamma$ in the rational LP. Here and below, $i$ ranges over $[m]$. Now, we have that \[\sum_i \tilde{a}_{ji} x^*_i \leq \sum_i \left(a_{ji} + \frac{\epsilon}{2} \right) x^*_i = \sum_i a_{ji} x^*_i + \frac{\epsilon}{2} \sum_i x^*_i \leq \gamma^* + \frac{\epsilon}{2} \sum_i x^*_i = \gamma^* + \frac{\epsilon}{2}.\] Since this holds for any row $j$ of the matrix, then setting $x = x^*$ achieves value at most $\gamma^* + \frac{\epsilon}{2}$ with respect to the rational LP. 

Conversely, consider an optimal solution $(x^{*, r}, \gamma^{*, r})$ of the rational LP --- by the above, we immediately have $\gamma^{*, r} \leq \gamma^* + \frac{\epsilon}{2}$. We claim it achieves value at most $\gamma^* + \epsilon$ with respect to the original LP. Indeed, for any matrix row $j$,
\[\sum_i a_{ji} x^{*, r}_i \leq \sum_i \left(\tilde{a}_{ji} \!+\! \frac{\epsilon}{2} \right) x^{*, r}_i = \sum_i \tilde{a}_{ji} x^{*, r}_i + \frac{\epsilon}{2} \sum_i x^{*, r}_i = \sum_i \tilde{a}_{ji} x^{*, r}_i + \frac{\epsilon}{2} \leq \gamma^{*, r} + \frac{\epsilon}{2}  \leq \left(\gamma^* + \frac{\epsilon}{2}\right) + \frac{\epsilon}{2} = \gamma^* + \epsilon.\]
Therefore, by solving the rational LP, we obtain an $\epsilon$-approximate solution to the original LP, as desired.
\end{proof}

\momentepsmultivalidity*

\begin{proof}
We briefly argue that the additive $\epsilon$-approximation to the (shifted and rescaled) value of the game results in the claimed dependence of the multivalidity guarantees on $\epsilon$.
When the learner achieves an $\epsilon$ approximation to the value of the game at each round, the statement of Corollary~\ref{cor:exists-moment} becomes: 
\[
    \E_{Q^L_{s+1}}[\tL_{s+1}|\pi_s] \leq L_s\cdotk\left(1 + \frac{\eta}{rn} +\frac{\eta}{rn'}+  2\eta^2\right) + \eta \epsilon  \leq L_s\cdotk\left(1 + \frac{\eta}{rn} +\frac{\eta}{rn'}+  2\eta^2\right) + \epsilon.
\]
Indeed, recall that the linear program that we solve at each round solves for the value of the game that has been shifted by $2\eta^2 L_s$ \emph{and divided by $\eta$}. For the second inequality, recall that $\eta < 1$.

Now, using the telescoping argument from the first part of the proof of Theorem~\ref{thm:general-bounds}, we obtain 
\begin{align*}
\exp\left(\eta \cdotk \E_{\tpi_T}[\max_{G,(i,j)}|\tV_T^{G,(i,j)}|]\right) 
\leq& 4|\cG|n \cdot n' \left((1 + \frac{\eta}{rn} +\frac{\eta}{rn'}+  2\eta^2\right)^T + \epsilon\sum_{t=0}^{T-1} \left(1 + \frac{\eta}{rn} +\frac{\eta}{rn'}+  2\eta^2 \right)^t, \\
\leq& 4|\cG|n \cdot n' \left((1 + \frac{\eta}{rn} +\frac{\eta}{rn'}+  2\eta^2\right)^T +\epsilon T \left(1 + \frac{\eta}{rn} +\frac{\eta}{rn'}+  2\eta^2 \right)^T,\\
=& (4|\cG|n \cdot n' +\epsilon T) \exp\left(T\ln\left(1 + \frac{\eta}{rn} +\frac{\eta}{rn'} + 2\eta^2\right) \right) ,\\
\leq& (4|\cG|n \cdot n' +\epsilon T) \exp\left(\frac{T \eta}{rn} +\frac{T \eta}{rn'} + 2T\eta^2\right),
\end{align*}
Taking logs and dividing by $\eta T$, we get
\[\frac{1}{T}\E_{\tpi_T}[\max_{G,(i,j)}|\tV_T^{G,(i,j)}|] \leq \frac{\ln(4|\cG|n \cdot n' +\epsilon T)}{\eta T} +\frac{1}{rn} +\frac{1}{rn'}+ 2\eta.\] Setting the two terms involving $\eta$ equal, we have:
\[\eta = \sqrt{\frac{\ln(4|\cG|n \cdot n' +\epsilon T)}{2T}}.\] For this choice of $\eta$, we obtain the following \emph{in-expectation} multivalidity guarantee (and the same guarantee for the $M$'s):
\[\frac{1}{T}\E_{\tpi_T}[\max_{G,(i,j)}|\tV_T^{G,(i,j)}|] \leq \frac{1}{rn} +\frac{1}{rn'} + 2\sqrt{\frac{2\ln(4|\cG|n \cdot n' +\epsilon T)}{T}}.\]
Now, setting $\epsilon = \frac{\epsilon'}{T}$ for any desired $\epsilon' > 0$, we obtain the guarantee (and same for the $M$'s) that 
\[\frac{1}{T}\E_{\tpi_T}[\max_{G,(i,j)}|\tV_T^{G,(i,j)}|] \leq \frac{1}{rn} +\frac{1}{rn'} + 2\sqrt{\frac{2\ln(4|\cG|n \cdot n' +\epsilon')}{T}} \quad \text{ if we set } \eta = \sqrt{\frac{\ln(4|\cG|n \cdot n' +\epsilon')}{2T}},\] and the resulting runtime will be polynomial in $T$ and $\log \frac{1}{\epsilon}$ and thus polynomial in $T$ and $\log \frac{1}{\epsilon'}$.

Now, we show the \emph{high-probability} multivalidity guarantee. In the proof of Theorem~\ref{thm:general-bounds}, the statement of Lemma~\ref{lem:bounded-expected-increase-main} changes to:
\begin{restatable}{lemma}{highprobeps}
For any  $\pi_T$, we have
\[
        \sum_{t=1}^T \left(\E_{\tilde{\pi_t}}\left[ \ln( \tX_{t} )\middle| \pi_{t-1}\right]  - \ln(X_{t-1}(\pi_{t-1})) \right) \le T \left(\eta c + 2\eta^2 + \epsilon\right).
    \]    
\end{restatable}
\begin{proof}
Fixing $\pi_T$ and taking any $t \leq T$, we have 
\begin{align*}
    &\E_{\tilde{\pi}_t}\left[ \ln( \tX_{t} ) |\pi_{t-1}\right] \le \ln\left(\E_{\tilde{\pi}_t}[ \tX_{t}|\pi_{t-1} ]\right), &\text{(Jensen's inequality)}\\
    &\le \ln \left(X_{t-1}(\pi_{t-1}) \cdot \left(1 + c\eta + 2\eta^2\right) + \epsilon\right), & \text{(since we computed an $\epsilon$-approximation)}\\
    &\le \ln \left(X_{t-1}(\pi_{t-1}) \cdot \left(1 + c\eta + 2\eta^2\right)\right) + \frac{\epsilon}{X_{t-1}(\pi_{t-1}) \cdot \left(1 + c\eta + 2\eta^2\right)}, & \text{($\ln(x+y) \leq \ln(x) + \frac{y}{x}$ for $x, y \geq 0$)}\\
    &\le \ln(X_{t-1}(\pi_{t-1})) + \ln\left(1 + c\eta + 2\eta^2\right) + \epsilon, & \text{(since the loss satisfies $X_{t-1}(\pi_{t-1}) \geq 1$)}\\
    &\le \ln(X_{t-1}(\pi_{t-1})) + \left(c\eta + 2\eta^2 + \epsilon\right). &\text{($\ln(1+x) \le x$ for any $ x> -1$)}
\end{align*}
Summing over every round $t \in [T]$ gives us the result.
\end{proof}
Thus, the statement of the second part of Theorem~\ref{thm:general-bounds} becomes that with probability $1-\lambda$,
\[\ln(X_T(\pi_T)) \le \ln(X_0 ) + T \left(\eta c + 2\eta^2 + \epsilon\right) + \eta \sqrt{8 T \ln\left(\frac{1}{\lambda}\right) }.\]

Now, applying it to the setting at hand, we obtain:
\begin{align*}
& \ln(L_T(\pi_T)) \le \ln(4|\cG|n \cdot n') + T \left(\eta \left(\frac{1}{rn} +\frac{1}{rn'} \right) + 2\eta^2 + \epsilon\right) + \eta \sqrt{8 T \ln\left(\frac{1}{\lambda}\right) }.
\end{align*}

Thus, taking log on both sides and dividing both sides by $\eta T$, we get
\begin{align*}
    \frac{1}{T} \max_{G,i,j} |V^{G,(i,j)}_{T}| \le \frac{1}{\eta T} \ln(L_T(\pi_T))
    \le \frac{\ln(4|\cG|n \cdot n')}{\eta\cdotk T} + \frac{1}{rn} +\frac{1}{rn'} + 2\eta + \frac{\epsilon}{\eta} + \sqrt{\frac{8  \ln\left(\frac{1}{\lambda}\right)}{T}}.
\end{align*}
Choosing $\eta = \sqrt{\frac{\ln(4|\cG|n \cdot n') + \epsilon T}{2T}}$, we obtain (and the same holds for the $M$'s):
\begin{align*}
\frac{1}{T} \max_{G,i,j}|V_T^{G,i,j}| &\le \frac{1}{rn} +\frac{1}{rn'} + 2 \sqrt{\frac{2(\ln(4|\cG|n \cdot n') + \epsilon T)}{T}} + \sqrt{\frac{8  \ln\left(\frac{1}{\lambda}\right)}{T}}\\
&\le \frac{1}{rn} +\frac{1}{rn'} + 4 \cdotk \sqrt{\frac{2}{T} \ln\left(\frac{4|\cG|n \cdot n'}{\lambda} \right)  + 2\epsilon},
\end{align*}
as desired.
\end{proof}

\section{Proofs from Section \ref{sec:onlinemultivalid}}\label{app:intervalproofs}

\deltaboundinterval*

\begin{proof} 
We calculate:
\begin{eqnarray*}
 && \Delta_{s+1}(\pi_s,x_{s+1},(\bell_{s+1},\bu_{s+1}))\\
 &=& \E_{\ty_{s+1}}\left[\sum_{\cG(x_{s+1})}\exp(\eta V_s^{G,(i,j)})\left(\exp(\eta v_\delta((\bell_{s+1}, \bu_{s+1}), \ty_{s+1}))-1\right) + \exp(-\eta V_s^{G,(i,j)})\left(\exp(-\eta v_\delta((\bell_{s+1},\bu_{s+1}), \ty_{s+1})-1\right)\right] \\
 &\leq& \E_{\ty_{s+1}}\left[\sum_{\cG(x_{s+1})}\exp(\eta V_s^{G,(i,j)})\left(\eta v_\delta((\bell_{s+1},\bu_{s+1}), \ty_{s+1})+2\eta^2\right) + \exp(-\eta V_s^{G,(i,j)})\left(-\eta v_\delta((\bell_{s+1},\bu_{s+1}), \ty_{s+1})+2\eta^2\right)\right] \\
  &=& \eta (\E_{\ty_{s+1}}[v_\delta((\bell_{s+1},\bu_{s+1}), \ty_{s+1})])
   C^{i,j}_s + 2\eta^2 
  \sum_{\cG(x_{s+1})}\exp(\eta V_s^{G,(i,j)}) + \exp(-\eta V_s^{G,(i,j)}) \\
  &\leq& \eta (\E_{\ty_{s+1}}[v_\delta((\bell_{s+1},\bu_{s+1}), \ty_{s+1})])
   C^{i,j}_s + 2\eta^2 
  L_s,
\end{eqnarray*}
as desired. 
Here the first inequality follows from the fact that for $0 < |x| < \frac{1}{2}$, $\exp(x) \leq 1+x+2x^2$, the following equality from organizing terms and the final inequality by noting that $\sum_{\cG(x_{s+1})}\exp(\eta V_s^{G,(i,j)}) + \exp(-\eta V_s^{G,(i,j)}) \leq L_s$ by definition of $L$.  
\end{proof}

\intervalmulti*

\begin{proof} 
From Observation \ref{obs:multicoverage}, it suffices to show that $\frac{1}{T} \E_{\pi_T}[\max|V_T^{G,(i,j)}|] \leq \alpha$. 

We begin by computing a bound on the (exponential of) the expectation of this quantity:
\begin{eqnarray*}
\exp\left(\eta \cdotk \E_{\tpi_T}[\max_{G,(i,j)}|\tV_T^{G,(i,j)}|]\right)&\leq&\E_{\tpi_T}\left[\exp\left(\eta\cdotk \max_{G,(i,j)}|\tV_T^{G,(i,j)}|\right)\right] \\
&=&\E_{\tpi_T}\left[\max_{G,(i,j)}\exp\left(\eta\cdotk|\tV_T^{G,(i,j)}|\right)\right] \\
&\leq& \E_{\tpi_T}\left[\max_{G,(i,j)}\left(\exp\left(\eta\cdotk \tV_T^{G,(i,j)}\right)+\exp\left(-\eta\cdotk V_T^{G,(i,j)}\right)\right)\right] \\
&\leq& \E_{\tpi_T}\left[\sum_{G,(i,j)}\left(\exp\left(\eta\cdotk \tV_T^{G,(i,j)}\right)+\exp\left(-\eta\cdotk \tV_T^{G,(i,j)}\right)\right)\right] \\
&=& \E_{\tpi_T}[\tL_T(\tpi_T)] \\
&\leq& 2|\cG|n^2\cdotk \exp\left(T\eta\rho+2T\eta^2\right).
\end{eqnarray*}
Here the first inequality follows from Jensen's inequality and the last one follows from Lemma \ref{lem:surrogatelossinterval}. Taking the log of both sides and dividing by $\eta\cdotk T$ we obtain:
$$\frac{1}{T}\E_{\tpi_T}[\max_{G,(i,j)}|\tV_T^{G,(i,j)}|] \leq \frac{\ln(2|\cG|n^2)}{\eta\cdotk T} + \rho + 2\eta.$$
Choosing $\eta = \sqrt{\frac{\ln(2|\cG|n^2)}{2T}}$  we obtain:
$$ \frac{1}{T}\E_{\tpi_T}[\max_{G,(i,j)}|\tV_T^{G,(i,j)}|]\leq \rho + 2\cdotk \sqrt{\frac{2\ln(2|\cG|n^2)}{T}},$$
as desired.
\end{proof}

Now, given $\tilde{L}$, define $\tilde{Z}$ analogously to the second part of Theorem \ref{thm:general-bounds}. Next, we can show that the increments of $\tilde{Z}$ thusly defined, at any round $t$, can be bounded.
\begin{lemma}
\label{lem:martingale-bounded-interval}
At any round $t \in [T]$ and for any realized transcript $\pi_t$, $|Z_t - Z_{t-1}|\le 2\eta.$
\end{lemma}
\begin{proof}
Observe that
\begin{align*}
    |Z_t - Z_{t-1}|
    =& \left|\ln(L_{t}(\pi_t)) - \E\left[ \ln( L_{t}(\tpi_t) ) |\pi_{t-1}\right]  \right|\\
    =& \left|\E \left[\ln\left(\frac{L_{t}(\pi_t)}{L_{t}(\tpi_t)}\right) \middle| \pi_{t-1} \right]\right|
\end{align*}
Note that for any $\pi_t$,
    $$L_t(\pi_{t}) = L_{t-1}(\pi_{t-1}) + \Delta_{t}(\pi_{t-1}, x_t, y_t, (\ell_t, \mu_t))$$
where:    
\begin{align*}
&\Delta_{t}(\pi_{t-1}, x_t, y_t, (\ell_t, u_t)) \\
=& \sum_{\cG(x_t)}\exp(\eta V_{t-1}^{G,\Bminv(\ell_t, u_t)})\left(\exp(\eta v_\delta((\ell_t, u_t), y_t))-1\right) + \exp(-\eta V_{t-1}^{G,\Bminv(\ell_t, u_t)})\left(\exp(-\eta v_\delta((\ell_t, u_t), y_t)-1\right).
\end{align*}
Since $v_\delta((\ell_t, u_t), y_t)$ must lie in $[-1,1]$ (actually $[-(1-\delta), \delta]$), we have that:
$$(\exp(-\eta) - 1) \cdotk L_{t-1}(\pi_{t-1}) \le \Delta_{t}(\pi_{t-1}, x_t, y_t,(\ell_t,u_t)) \le (\exp(\eta) -1) \cdotk L_{t-1}(\pi_{t-1})$$ which implies:
$$\exp(-\eta) \cdotk L_{t-1}(\pi_{t-1}) \le L_t(\pi_t) \le \exp(\eta) \cdotk L_{t-1}(\pi_{t-1}).$$
Therefore, for any two $\pi_t, \pi'_t$ such that the corresponding transcripts for the first $t-1$ periods are the same, we have
\begin{align*}
    \left| \ln\left(\frac{L_t(\pi_{t})}{L_t(\pi'_{t})}\right) \right| \le \ln\left(\frac{\exp(\eta)}{\exp(-\eta)}\right) = 2\eta.
\end{align*}
Therefore we have $\left|\E \left[\ln\left(\frac{L_{t}(\pi_t)}{L_{t}(\tpi_t)}\right) \middle| \pi_{t-1} \right]\right| \leq 2\eta$ as desired.
\end{proof}

\hpintervalcalibration*

\begin{proof}
By Lemma \ref{lem:martingale-bounded-interval}, the second part of Theorem \ref{thm:general-bounds} applies, and plugging in  $L_0 = 2|\cG|n^2$ and $c = \rho$, we have that, with probability $(1-\lambda)$ over the randomness of the transcript:
\begin{align*}
& \ln(L_T(\pi_T)) \le \ln(2|\cG|n^2) + T \left(\eta \rho+ 2\eta^2\right) + \eta \sqrt{8 T \ln\left(\frac{1}{\lambda}\right) }.
\end{align*}
Now, note that 
\begin{align*}
\exp\left(\eta\cdotk \max_{G,i,j}|V_T^{G,(i,j)}|\right) &=\max_{G,i,j}\exp\left(\eta\cdotk|V_T^{G,(i,j)}|\right), \\
&\leq \max_{G,i,j}\left(\exp\left(\eta\cdotk V_T^{G,(i,j)}\right)+\exp\left(-\eta\cdotk V_T^{G,(i,j)}\right)\right), \\
&\leq \sum_{G,i,j}\left(\exp\left(\eta\cdotk V_T^{G,(i,j)}\right)+\exp\left(-\eta\cdotk V_T^{G,(i,j)}\right) \right), \\
&= L_T(\pi_T). 
\end{align*}
Taking log on both sides and dividing both sides by $\eta T$, we get
\begin{align*}
    \frac{1}{T} \max_{G,i,j} |V^{G,(i,j)}_{T}| \le \frac{1}{\eta T} \ln(L_T(\pi_T))
    \le \frac{\ln(2|\cG|n^2)}{\eta\cdotk T} + \rho+ 2\eta + \sqrt{\frac{8  \ln\left(\frac{1}{\lambda}\right)}{T}}.
\end{align*}
Choosing $\eta = \sqrt{\frac{\ln(2|\cG|n^2)}{2T}}$, we obtain
\begin{align*}
\frac{1}{T} \max_{G,i,j}|V_T^{G,i,j}| &\le \rho+ 2 \sqrt{\frac{2\ln(2|\cG|n^2)}{T}} + \sqrt{\frac{8  \ln\left(\frac{1}{\lambda}\right)}{T}}\\
&\le \rho + 4 \cdotk \sqrt{\frac{2}{T} \ln\left(\frac{2|\cG|n^2}{\lambda} \right)},
\end{align*}
as desired.
\end{proof}

\intervalepsmultivalidity*

\begin{proof}
We briefly argue that the additive $\epsilon$-approximation to the (shifted and rescaled) value of the game results in the claimed dependence of the multivalidity guarantees on $\epsilon$.
When the learner achieves an $\epsilon$ approximation to the value of the game at each round, the statement of Corollary~\ref{cor:intervalexists} becomes: \[\E_{ (\ell,u) \sim Q^L_{s+1}}[\tL_{s+1}|\pi_s] \leq L_s\cdotk\left(1 + \eta\rho + 2\eta^2\right) + \eta\epsilon \leq L_s\cdotk\left(1 + \eta\rho + 2\eta^2\right) + \epsilon.\]
Indeed, recall that the linear program that we solve at each round solves for the value of the game that has been shifted by $2\eta^2 L_s$ \emph{and divided by $\eta$}. For the second inequality, recall that $\eta < 1$.

Now, using the telescoping argument from the first part of the proof of Theorem~\ref{thm:general-bounds}, we obtain 
\begin{align*}
\exp\left(\eta \cdotk \E_{\tpi_T}[\max_{G,(i,j)}|\tV_T^{G,(i,j)}|]\right) 
\leq& 2|\cG| n^2 \left(1 + \eta \rho  + 2\eta^2\right)^T + \epsilon\sum_{t=0}^{T-1} (1 + \eta \rho + 2 \eta^2)^t, \\
\leq& 2|\cG| n^2 \left(1 + \eta \rho  + 2\eta^2\right)^T +\epsilon T (1 + \eta \rho + 2 \eta^2)^T,\\
=& (2|\cG| n^2 +\epsilon T) \exp\left(T\ln\left(1 + \eta \rho  + 2\eta^2\right) \right) ,\\
\leq& (2|\cG| n^2 +\epsilon T) \exp\left(T \eta \rho + 2T\eta^2\right),
\end{align*}
Taking logs and dividing by $\eta T$, we get
\[\frac{1}{T}\E_{\tpi_T}[\max_{G,(i,j)}|\tV_T^{G,(i,j)}|] \leq \frac{\ln(2|\cG| n^2 + \epsilon T)}{\eta T} + \rho + 2\eta.\] Setting the two terms involving $\eta$ equal, we have:
\[\eta = \sqrt{\frac{\ln(2|\cG| n^2 + \epsilon T)}{2T}}.\] For this choice of $\eta$, we obtain the following \emph{in-expectation} multivalidity guarantee:
\[\frac{1}{T}\E_{\tpi_T}[\max_{G,(i,j)}|\tV_T^{G,(i,j)}|] \leq \rho + 2\sqrt{\frac{2\ln(2|\cG| n^2 + \epsilon T)}{T}}.\]
Now, setting $\epsilon = \frac{\epsilon'}{T}$ for any desired $\epsilon' > 0$, we obtain the guarantee that 
\[\frac{1}{T}\E_{\tpi_T}[\max_{G,(i,j)}|\tV_T^{G,(i,j)}|] \leq \rho + 2\sqrt{\frac{2\ln(2|\cG| n^2 + \epsilon')}{T}} \quad \text{ if we set } \eta = \sqrt{\frac{\ln(2|\cG| n^2 + \epsilon')}{2T}},\] and the resulting runtime will be polynomial in $T$ and $\log \frac{1}{\epsilon}$ and thus polynomial in $T$ and $\log \frac{1}{\epsilon'}$.

Now, we show the \emph{high-probability} multivalidity guarantee. In the proof of Theorem~\ref{thm:general-bounds}, the statement of Lemma~\ref{lem:bounded-expected-increase-main} changes to:
\highprobeps*
We show this updated claim in the proof of Lemma~\ref{lem:momentepsmultivalidity} of Section~\ref{sec:momentalg}.

Thus, the statement of the second part of Theorem~\ref{thm:general-bounds} becomes that with probability $1-\lambda$,
\[\ln(X_T(\pi_T)) \le \ln(X_0 ) + T \left(\eta c + 2\eta^2 + \epsilon\right) + \eta \sqrt{8 T \ln\left(\frac{1}{\lambda}\right) }.\]

Now, applying it to the setting at hand, we obtain:
\begin{align*}
& \ln(L_T(\pi_T)) \le \ln(2|\cG|n^2) + T \left(\eta \rho+ 2\eta^2 + \epsilon\right) + \eta \sqrt{8 T \ln\left(\frac{1}{\lambda}\right) }.
\end{align*}

Thus, taking log on both sides and dividing both sides by $\eta T$, we get
\begin{align*}
    \frac{1}{T} \max_{G,i,j} |V^{G,(i,j)}_{T}| \le \frac{1}{\eta T} \ln(L_T(\pi_T))
    \le \frac{\ln(2|\cG|n^2)}{\eta\cdotk T} + \rho+ 2\eta + \frac{\epsilon}{\eta} + \sqrt{\frac{8  \ln\left(\frac{1}{\lambda}\right)}{T}}.
\end{align*}

Choosing $\eta = \sqrt{\frac{\ln(2|\cG|n^2) + \epsilon T}{2T}}$, we obtain:
\begin{align*}
\frac{1}{T} \max_{G,i,j}|V_T^{G,i,j}| &\le \rho+ 2 \sqrt{\frac{2(\ln(2|\cG|n^2) + \epsilon T)}{T}} + \sqrt{\frac{8  \ln\left(\frac{1}{\lambda}\right)}{T}}\\
&\le \rho + 4 \cdotk \sqrt{\frac{2}{T} \ln\left(\frac{2|\cG|n^2}{\lambda} \right)  + 2\epsilon},
\end{align*}
as desired.
\end{proof}

\end{document}